\documentclass{article} 

\usepackage{microtype}
\usepackage{graphicx}
\usepackage{booktabs} %

\usepackage[colorlinks=true,citecolor=blue,urlcolor=red]{hyperref}

\usepackage[accepted]{icml2023}

\usepackage{amsmath}
\usepackage{amssymb}
\usepackage{amsthm}
\usepackage{amsfonts}
\usepackage{mathtools}
\usepackage{multirow}
\usepackage{nicefrac}
\usepackage{lipsum}
\usepackage{ulem}
\usepackage[font=normal, labelfont=bf]{caption}
\usepackage{subcaption}
\usepackage{wrapfig}
\usepackage{xcolor}
\usepackage[capitalize,noabbrev]{cleveref}

\usepackage{multibib}

\usepackage[export]{adjustbox}

\graphicspath{ {images/} }

\def\maxop{\mathop{\rm max}\limits}

\newcommand{\sign}{\mathop{\rm sign}}
\def\R{\mathbb{R}}

\newcommand{\norm}[1]{\left\|#1\right\|}

\newcommand{\inner}[1]{\left\langle#1\right\rangle}

\DeclareMathOperator{\E}{\mathbb{E}}
\let\l\relax
\DeclareMathOperator{\l}{\mathbf{\ell}}

\newcommand{\cv}{\boldsymbol{c}}

\newcommand{\rv}{\boldsymbol{r}}
\newcommand{\sv}{\boldsymbol{s}}
\newcommand{\tv}{\boldsymbol{t}}
\newcommand{\uv}{\boldsymbol{u}}
\newcommand{\vv}{\boldsymbol{v}}
\newcommand{\wv}{\boldsymbol{w}}
\newcommand{\xv}{\boldsymbol{x}}
\newcommand{\yv}{\boldsymbol{y}}
\newcommand{\zv}{\boldsymbol{z}}

\newcommand{\vI}{\boldsymbol{I}}

\newcommand{\vU}{\boldsymbol{U}}
\newcommand{\vV}{\boldsymbol{V}}

\newcommand{\vX}{\boldsymbol{X}}

\newcommand{\tr}{\text{tr}}
\newcommand{\lm}{\lambda_{\text{max}}}

\newcommand{\relu}{\text{ReLU}}
\newcommand{\gelu}{\text{GELU}}

\newcommand{\diag}{\mbox{$\mbox{diag}$}}

\newcommand{\betav}{\boldsymbol{\beta}}
\newcommand{\deltav}{\boldsymbol{\delta}}
\newcommand{\gammav}{\boldsymbol{\gamma}}

\definecolor{darkgreen}{rgb}{0. .7 0.}

\newcommand{\Scal}{\mathcal{S}}

\usepackage{tcolorbox}

\newlength\newl
\newcommand{\iter}[2]{#1^{(#2)}}

\newcommand{\myparagraph}{\textbf}

\newtheorem{theorem}{Theorem}

\newtheorem{proposition}[theorem]{Proposition}

\usepackage{enumitem}
\setitemize{topsep=3.5pt,parsep=0pt,partopsep=0pt,leftmargin=25pt}
\usepackage{amssymb}

\begin{document}

\twocolumn[
\icmltitle{A Modern Look at the Relationship between Sharpness and Generalization}

\icmlsetsymbol{equal}{*}

\begin{icmlauthorlist}
    \icmlauthor{Maksym Andriushchenko}{epfl}
    \icmlauthor{Francesco Croce}{tueai,tue}
    \icmlauthor{Maximilian M{\"u}ller}{tueai,tue}
    \icmlauthor{Matthias Hein}{tueai,tue}
    \icmlauthor{Nicolas Flammarion}{epfl}
\end{icmlauthorlist}

\icmlaffiliation{epfl}{EPFL}
\icmlaffiliation{tueai}{T{\"u}bingen AI Center}
\icmlaffiliation{tue}{University of T{\"u}bingen}

\icmlcorrespondingauthor{Maksym Andriushchenko}{maksym.andriushchenko@epfl.ch}

\icmlkeywords{Deep learning, Sharpness, Generalization}

\vskip 0.3in
]

\printAffiliationsAndNotice{}  %

\begin{abstract}
    Sharpness of minima is a promising quantity that can positively correlate with test error for deep networks and, when optimized during training, can improve generalization. However, standard sharpness is not invariant under reparametrizations of neural networks, and, to fix this, reparametrization-invariant sharpness definitions have been proposed, most prominently \textit{adaptive sharpness} \citep{kwon2021asam}. But does it really capture generalization in modern practical settings? We comprehensively explore this question in a detailed study of various definitions of adaptive sharpness in settings ranging from training from scratch on ImageNet and CIFAR-10 to fine-tuning CLIP on ImageNet and BERT on MNLI. We focus mostly on \textit{transformers} for which little is known in terms of sharpness despite their widespread usage. Overall, we observe that sharpness \textit{does not correlate well} with generalization but rather with some training parameters like the learning rate that can be positively or negatively correlated with generalization depending on the setup. Interestingly, in multiple cases, we observe a consistent \textit{negative} correlation of sharpness with out-of-distribution error implying that \textit{sharper} minima can generalize \textit{better}. Finally, we illustrate on a simple model that the right sharpness measure is highly data-dependent, and that we do not understand well this aspect for realistic data distributions. Our code is available at \url{https://github.com/tml-epfl/sharpness-vs-generalization}.
\end{abstract}

\section{Introduction}
Considering the sharpness of the training objective at a minimum has intuitive appeal: if the loss surface is slightly perturbed due to a train vs. test or out-of-distribution (OOD) discrepancy, flat minima of deep networks should still have low loss \citep{hochreiter1995simplifying, keskar2016large}. 
On the theoretical side, sharpness appears in generalization bounds \citep{neyshabur2017exploring, dziugaite2018entropy, foret2021sharpnessaware} but this fact alone is not necessarily informative for practical settings. For example, quantities like the VC-dimension typically correlate \textit{negatively} with generalization contrary to what the generalization bound might suggest \citep{jiang2019fantastic}. 
Importantly, it has been shown empirically that sharpness can also correlate well with generalization in common deep learning setups \citep{keskar2016large, jiang2019fantastic} which makes it a promising generalization measure that can potentially distinguish well-generalizing solutions. 
Additionally, empirical success of training methods that minimize sharpness such as sharpness-aware minimization (SAM) \citep{zheng2020regularizing, wu2020adversarial, foret2021sharpnessaware} further suggests that sharpness can be an important quantity for generalization.

\myparagraph{Motivation: why revisiting sharpness?}
Many works imply or conjecture that flatter minima should generalize better \citep{xing2018walk, zhou2020towards, cha2021swad, park2022vision, lyu2022understanding} for standard or OOD data. However, standard sharpness definitions do not correlate well with generalization \citep{jiang2019fantastic, kaur2022maximum} which can be partially due to their lack of invariance under reparametrizations that leave the model unchanged \citep{dinh2017sharp, granziol2020flatness, zhang2021flatness}. 
Adaptive sharpness appears to be more promising since it fixes the reparametrization issue and is shown to empirically correlate better with generalization \citep{kwon2021asam}. %
However, the empirical evidence in \citet{kwon2021asam} and other works that discuss sharpness \citep{keskar2016large, jiang2019fantastic, dziugaite2020search, bisla2022low} is restricted to small datasets like CIFAR-10 or SVHN.
In addition, SAM appears to be particularly useful for new architectures like vision transformers \citep{chen2021vision} for which there has been no systematic studies of sharpness vs. generalization.
Moreover, transfer learning is becoming the default option for vision and language tasks but not much is known about sharpness there.
Finally, the relationship between sharpness and OOD generalization is also underexplored. 
These new developments motivate us to revisit the role of sharpness in these new settings.

\myparagraph{Contributions.}
We aim to provide a comprehensive study focusing specifically on adaptive sharpness in order to answer the following fundamental question:
\begin{center}
    \textit{Can reparametrization-invariant sharpness capture generalization in modern practical settings?}
    \\
\end{center}
Towards this goal, we make the following contributions:
\begin{itemize}[left=0mm, topsep=0pt]
    \item We provide extensive evaluations of multiple reparametrization-invariant sharpness measures for (1) training from scratch on ImageNet and CIFAR-10 using transformers and ConvNets, and (2) fine-tuning CLIP and BERT transformers on ImageNet and MNLI.
    \item 
    We observe that sharpness \textit{does not correlate well} with generalization but rather with some training parameters like the learning rate which can be positively or negatively correlated with generalization depending on the setup.
    \item Interestingly, in multiple cases, we observe a consistent \textit{negative} correlation of sharpness with OOD generalization implying that \textit{sharper} minima can generalize \textit{better}.
    \item Finally, we provide an analysis on a simple model where we know the measure responsible for generalization. Our analysis suggests that (1) different sharpness definitions can capture totally different trends, and (2) the right sharpness measure is highly \textit{data-dependent}. %
\end{itemize}

\section{Related work}
\label{sec:related_work}
Here we discuss the most related papers to our work. %

\myparagraph{Systematic studies on sharpness vs. generalization.}
The seminal work of \citet{keskar2016large} shows that the performance degradation of large-batch SGD \citep{lecun2012efficient} is correlated with sharpness of minima. 
\citet{neyshabur2017exploring} explore different generalization measure that may explain generalization for deep networks suggesting that sharpness can be a promising measure. %
\citet{jiang2019fantastic} perform a systematic study that shows a strong correlation between sharpness and generalization on a large set of CIFAR-10/SVHN models trained with many different hyperparameters.  
Their experimental protocol is, however, criticized in \citet{dziugaite2020search} since it can obscure failures of generalization measures and instead should be  evaluated within the framework of distributional robustness. 
\citet{vedantam2021an} discuss OOD generalization on small datasets and evaluate a definition of sharpness which, however, does not correlate well with OOD generalization. 
\citet{stutz2021relating} study the relationship between sharpness and generalization under $\ell_p$-bounded adversarial perturbations. %
\citet{andriushchenko2022towards} study reasons behind the success of SAM and highlight the importance of using sharpness computed on a small subset of training points. 
\citet{kaur2022maximum} discuss that the maximum eigenvalue of the Hessian is not always predictive to generalization even for models obtained via standard training methods.

\myparagraph{Reparametrization-invariant sharpness definitions.}
The magnitude-aware sharpness of \citet{keskar2016large} mitigates but does not completely resolve reparametrization invariance.
\citet{liang2019fisher} consider the Fisher-Rao metric related to sharpness and invariant to network reparametrization.  %
\citet{petzka2021relative} propose a sharpness measure based on the trace of the Hessian and show correlation for a small ConvNet on CIFAR-10. 
\citet{tsuzuku2020normalized} suggest to use a specifically rescaled sharpness inspired by the PAC-Bayes theory and report high correlation with generalization for ResNets on CIFAR-10. %
Most importantly for our work, \citet{kwon2021asam} introduce adaptive sharpness which is reparametrization invariant, correlates well with generalization, and generalizes multiple existing sharpness definitions.

\myparagraph{Explicit and implicit sharpness minimization.}
The idea that flat minima can be beneficial for generalization dates back to \citet{hochreiter1995simplifying} and inspires multiple methods that optimize for more robust minima. These methods optimize different criteria ranging from random perturbations such as dropout \citep{srivastava2014dropout} and Entropy-SGD \citep{chaudhari2016entropy} to worst-case perturbations such as SAM \citep{foret2021sharpnessaware} and its variations \citep{kwon2021asam, zhuang2022surrogate, du2022sharpnessaware}. 
Notably, \citet{chen2021vision} suggest that SAM is particularly helpful for vision transformers on ImageNet scale and that standard transformers by default converge to very sharp minima. %
Concurrently, 
works on the implicit bias of SGD suggest \textit{implicit} minimization of some hidden complexity measures related to flatness of minima \citep{keskar2016large, smith2017bayesian, xing2018walk}. 
\citet{izmailov2018averaging} propose to average weights during SGD to improve generalization and motivate it by sharpness reduction. 
\citet{smith2021origin} derive an implicit regularization term of SGD based on the gradient norm. 
Sharpness-related quantities based on the Hessian have been a focus of many recent works. 
E.g., \citet{cohen2021gradient, arora2022understanding, damian2022self} empirically and theoretically characterize the regime of full-batch gradient descent where the maximum eigenvalue of the Hessian becomes inversely proportional to the learning rate used for training. 
\citet{blanc2020implicit, li2021happens, damian2021label} discover implicit minimization of the trace of the Hessian for label-noise SGD used as a proxy of standard SGD. 
The common theme behind these works is a focus on sharpness-related metrics as a tool to better understand generalization for deep networks.

\section{Adaptive Sharpness, its Invariances, and Computation}
In this section, we first provide background on adaptive sharpness, then discuss its invariance properties for modern architectures, and propose a way to compute worst-case sharpness efficiently.

\subsection{Background on Sharpness}
\myparagraph{Sharpness definitions.}
We denote the loss on a set of \textit{training} points $\Scal$ as $L_\Scal(\wv) = \frac{1}{|S|} \sum_{(\xv, \yv) \in \Scal} \l_{\xv\yv}(\wv)$, where $\ell_{\xv \yv}(\wv) \in \R_+$ represents some loss function (e.g., cross-entropy) on the training pair $(\xv, \yv) \in \Scal$ computed with the network weights $\wv$.
For arbitrary $\wv \in \R^p$ (i.e., not necessarily a minimum), we define the \textit{adaptive average-case} and \textit{adaptive worst-case $m$-sharpness} with radius $\rho$ and with respect to a vector $\cv \in \R^p$ as:
\begin{align} \label{eq:sharpness}
S_{avg}^\rho(\wv, \cv) &\triangleq \E_{\substack{\Scal \sim P_m \ \ \ \ \ \ \\ \deltav \sim \mathcal{N}(0, \rho^2 diag(\cv^2))}} \hspace{-8mm} L_\Scal(\wv + \deltav) - L_\Scal(\wv), \\
S_{max}^\rho(\wv, \cv) &\triangleq \E_{\Scal \sim P_m} \maxop_{\norm{\deltav \odot \cv^{-1}}_p \leq \rho} L_\Scal(\wv + \deltav) - L_\Scal(\wv), \nonumber
\end{align}
where $\odot$/$^{-1}$ denotes elementwise multiplication/inversion and $P_m$ is the data distribution that returns $m$ training pairs $(\xv, \yv)$. 
Both average-case and worst-case sharpness have often been considered in the literature, and worst-case sharpness is mostly determined to correlate better with generalization \citep{jiang2019fantastic, dziugaite2020search, kwon2021asam}, especially with a small $m$ (i.e., $|\Scal|$) in worst-case sharpness \citep{foret2021sharpnessaware}. 
Using $\cv = |\wv|$ leads to \textit{elementwise} adaptive sharpness \citep{kwon2021asam} and makes the sharpness invariant under multiplicative reparametrizations that preserve the network, i.e., for any $\cv \in \R^p$ such that $f(\wv \odot \cv) = f(\wv)$ we have:
\begin{align*}
&S_{max}^\rho(\wv \odot \cv, |\wv \odot \cv|) =\\
&\E_\Scal \maxop_{\|\deltav \odot (|\wv| \odot \cv)^{-1}\|_p \leq \rho} L_\Scal(\wv \odot \cv + \deltav) - L_\Scal(\wv \odot \cv) =\\
&\E_\Scal \maxop_{\|\deltav' \odot |\wv|^{-1}\|_p \leq \rho} L_\Scal((\wv + \deltav') \odot \cv) - L_\Scal(\wv \odot \cv) =\\
&\E_\Scal \maxop_{\|\deltav' \odot |\wv|^{-1}\|_p \leq \rho} L_\Scal(\wv + \deltav') - L_\Scal(\wv) = S_{max}^\rho(\wv, |\wv|),
\end{align*}
where we used the substitution $\deltav' := \deltav \odot \cv^{-1}$. Similarly, one can show that $S_{avg}^\rho(\wv \odot \cv, |\wv \odot \cv|) = S_{avg}^\rho(\wv, |\wv|)$. 
Thus, this illustrates that \textit{the criticism of sharpness stated in \citet{dinh2017sharp} does not apply to adaptive sharpness}, and there is no need to ``balance'' the network in a pre-processing step like, e.g., done in \citet{bisla2022low}.

\myparagraph{Connections between different sharpness definitions.}
Here we generalize the analytical expressions of standard sharpness for radius $\rho \to 0$ that depend on the first- or second-order terms which are frequently used in the literature \citep{blanc2020implicit, tsuzuku2020normalized, li2021happens, damian2021label}. For a thrice differentiable loss $L(\wv)$, the average-case elementwise adaptive sharpness can be computed as (see App.~\ref{sec:app_asymptotic} for proofs):
\begin{align} \nonumber
S_{avg}^\rho(\wv, |\wv|) = &\E_{\Scal \sim P_m} \frac{\rho^2}{2} \tr(\nabla^2 L_\Scal(\wv) \odot |\wv| |\wv|^\top)\\
                           &+ O(\rho^3).
\label{eq:avg_sharpness_small_rho}
\end{align}
We note that the first-order term cancels out completely and plays no role. This is not the case for worst-case adaptive sharpness where we get for $p=2$ the following expression for every critical point that is not a local maximum:
\begin{align} \nonumber
S_{max}^\rho(\wv, |\wv|) = &\E_{\Scal \sim P_m} \frac{\rho^2}{2} \lm(\nabla^2 L_\Scal(\wv) \odot |\wv| |\wv|^\top)\\
                           &+ O(\rho^3),
\label{eq:max_sharpness_small_rho}
\end{align}
otherwise the first-order term dominates and we get $\rho \E_{\Scal \sim P_m}\| \nabla L(\wv) \odot |\wv| \|_2$, which resembles the implicit gradient regularization of \citet{smith2021origin}. Thus, worst-case sharpness with a small radius captures different properties of the loss surface depending on whether $\wv$ is close to a minimum or not. We make use of these quantities in the last section to discuss insights from simple models.
For the experiments, however, we evaluate a range of $\rho$ where the smallest $\rho$ well-approximates the above quantities.

\myparagraph{What do we expect sharpness to capture?}
We are looking for a sharpness measure that can be \textit{predictive for generalization} meaning that it satisfies either of these two hypotheses: 
\begin{itemize}[left=0mm]
    \item \textbf{Strong hypothesis}: sharpness is highly correlated with generalization suggesting a \textit{possibility} of a causal relation.
    \item \textbf{Weak hypothesis}: models with the lowest sharpness generalize well suggesting that sharpness might be \textit{sufficient but not necessary} for generalization.
\end{itemize}
To detect correlation, we follow the previous works by \citet{jiang2019fantastic, dziugaite2020search, kwon2021asam} and use the Kendall rank correlation coefficient:
\begin{align} \label{eq:tau}
\hspace{-3mm}
\tau(\tv, \sv) = \frac{2}{M(M-1)} \sum_{i<j} \sign(t_i - t_j) \sign(s_i - s_j)
\end{align}
where $\tv, \sv \in \R^M$ are vectors of test error and sharpness values for $M$ different models.
We adopt a less demanding setting than in the previous works of \citet{neyshabur2017exploring, jiang2019fantastic, dziugaite2020search}, and only compare models \textit{within the same loss surface} motivated by the geometric motivation behind sharpness. This restriction rules out comparing models with different architectures (including different width and depth) or measuring sharpness on a different set of points since both changes would change the loss surface. 
According to the same reason, we also do not consider the ability of sharpness to capture robustness to different amounts of noisy labels (unlike, e.g., \citet{neyshabur2017exploring}). 
We always evaluate sharpness on the \textit{same} training points taken without any data augmentations. Moreover, we always compare models trained with exactly the same training sets but, at the same time, we allow the usage of algorithmic techniques such as data augmentation or mixup for training.

\subsection{Which Invariances Do We Need Sharpness to Capture for Modern Architectures?}
Throughout the paper, we focus on \textit{elementwise} adaptive sharpness which, as we show, satisfies the main reparametrization invariances for ResNets and ViTs. 
Let us denote $f_{\wv}:\R^d \to \R^K$ a network with parameters $\wv$, which returns the logits $f_{\wv}(\xv) \in \R^K$ for an input $\xv \in \R^d$.
By a reparametrization invariance we mean a function $T: \R^p \to \R^p$ such that for every $\wv \in \R^p$ and $\xv \in \R^d$ it holds $f_{\wv}(\xv) = f_{T(\wv)}(\xv)$.
We briefly discuss here that adaptive sharpness also stays invariant for \textit{modern} architectures like ResNets and ViTs %
involving normalization layers and self-attention. 
Finally, we discuss how to treat the scale-sensitivity of classification losses.

\myparagraph{Adaptive sharpness for ResNets.}
A typical block of a pre-activation ResNet between skip connections includes the following sequence of operations: \texttt{BN}$\to$\texttt{ReLU}$\to$\texttt{conv}$\to$\texttt{BN}$\to$\texttt{ReLU}$\to$\texttt{conv}  where \texttt{BN} denotes BatchNorm. So we need to make sure that the sharpness definition we use is invariant to transformations that leave the network unchanged: (1) multiplication of the affine BatchNorm parameters by $\alpha \in \R_+$ and division of the subsequent convolutional parameters by the same $\alpha$ (since ReLU is positive one-homogeneous and $\relu(\alpha z) / \alpha = \relu(z)$), and (2) multiplying the convolutional layer by any $\alpha \in \R_+$ due to scale-invariance of the subsequent BatchNorm layer. Both multiplicative invariances are satisfied by elementwise adaptive sharpness since $S_{max}^\rho(\wv \odot \cv, |\wv \odot \cv|) = S_{max}^\rho(\wv, |\wv|)$ as shown above.

\myparagraph{Adaptive sharpness for ViTs.}
A typical MLP block of ViTs contains the following operations: \texttt{LN}$\to$\texttt{Linear} $\to$\texttt{GELU}$\to$\texttt{Linear} where \texttt{LN} denotes LayerNorm, and pre-softmax self-attention weights are computed as $Z W_Q W_K^\top Z^\top$ where $Z \in \R^{P \times D}$ is the matrix of $P$ $D$-dimensional tokens.
The network thus has the following invariances to multiplication/division by $\alpha$: (1) between \texttt{LN} and \texttt{Linear} in MLP, (2) between $W_Q$ in $W_K$ in self-attention, (3) between two \texttt{Linear} layers that have GELU in-between for which $\gelu(\alpha z) / \alpha \approx \gelu(z)$. Moreover, at the beginning of the network there is a part of the network which is invariant to the scale of the \texttt{Linear} layer (\texttt{Linear}$\to$\texttt{LN}).
Similarly to ResNets, all these invariances are multiplicative, so the argument about the invariance of elementwise adaptive sharpness is the same.

\myparagraph{Scale-sensitivity for classification losses.}
However, adaptive sharpness remains sensitive to the \textit{scale} of the classifier, meaning that 
the sharpness together with the cross-entropy loss keep decreasing to zero after reaching zero training error.
This can be seen even for linear models for which scaling the weight vector by a constant %
changes the adaptive sharpness as shown in Fig.~\ref{fig:sharpness_linear_model}. 
To fix this issue, \citet{tsuzuku2020normalized} propose to use normalization of the logits $f_{\wv}$, i.e.: 
\begin{align}
\tilde{f}_{\wv}(\xv) \triangleq \frac{f_{\wv}(\xv)}{\sqrt{\frac{1}{K}\sum_{i=1}^K (f_{\wv}(\xv)_i - f_{avg}(\xv))^2}}, %
\end{align}
where $f_{avg}(\xv) = \frac{1}{K}\sum_{j=1}^K f_{\wv}(\xv)_j$. 
This provably fixes the scaling issue meaning that scaling the output layer by $\alpha \in \R_+$ does not affect the logits. %
Moreover, this change can make models having different training loss more comparable to each other. %
\begin{figure}
    \centering
    \includegraphics[width=0.36\textwidth]{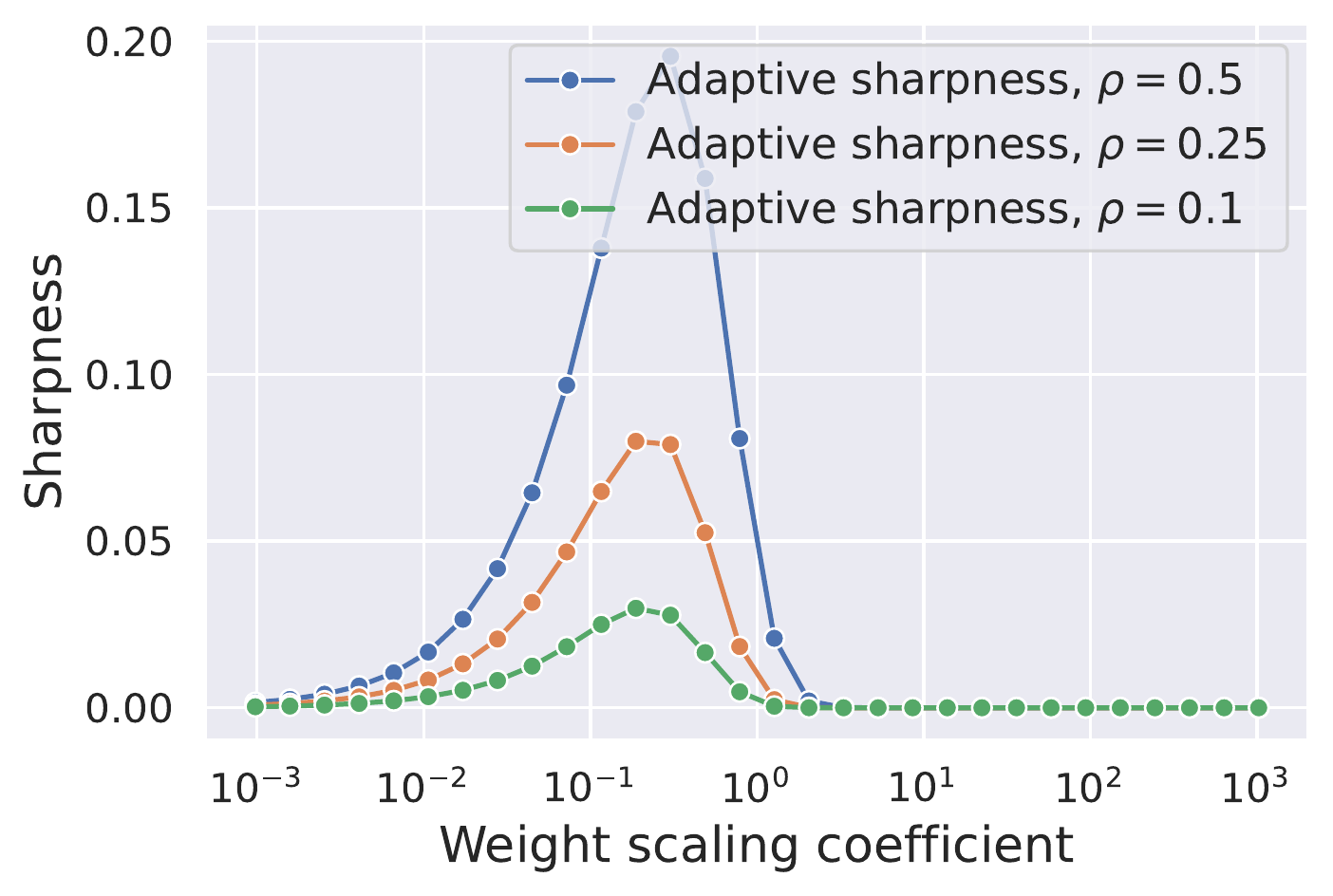}
    \vspace{-3mm}
    \caption{Sensitivity of adaptive sharpness to weight scaling for a linear model that achieves zero training error.} %
    \label{fig:sharpness_linear_model}
\end{figure}

\subsection{How to Compute Worst-Case Sharpness Efficiently?}
Estimation of worst-case sharpness involves solving a constrained maximization problem typically using projected gradient ascent which can be sensitive to its hyperparameters, primarily the step size. To avoid doing extensive grid searches over the hyperparameters of gradient ascent for each model, we choose to use \textit{Auto-PGD} \citep{croce2020reliable} (see Algorithm~\ref{alg:apgd} in Appendix for the precise formulation). Auto-PGD is a \textit{hyperparameter-free} method designed to accurately estimate adversarial robustness by solving a similar optimization problem to worst-case sharpness but over the input space instead of the parameter space. As in $\ell_\infty$ and $\ell_2$ versions of Auto-PGD, for each gradient step, we use gradient-sign and plain-gradient updates, respectively, but we make them proportional to $|\wv|$, to better take into account the geometry induced by elementwise adaptive sharpness. We show in Sec.~\ref{sec:app_cifar10_role_n_iter} in Appendix that as few as $20$ steps are typically sufficient to converge with Auto-PGD.

\section{Sharpness vs. Generalization: Modern Setup}
\begin{figure*}[t] \centering \small
    \tabcolsep=1.1pt
    \begin{tabular}{@{}c@{}}
        \textbf{With logit normalization} \\
        \includegraphics[width=0.99\textwidth]{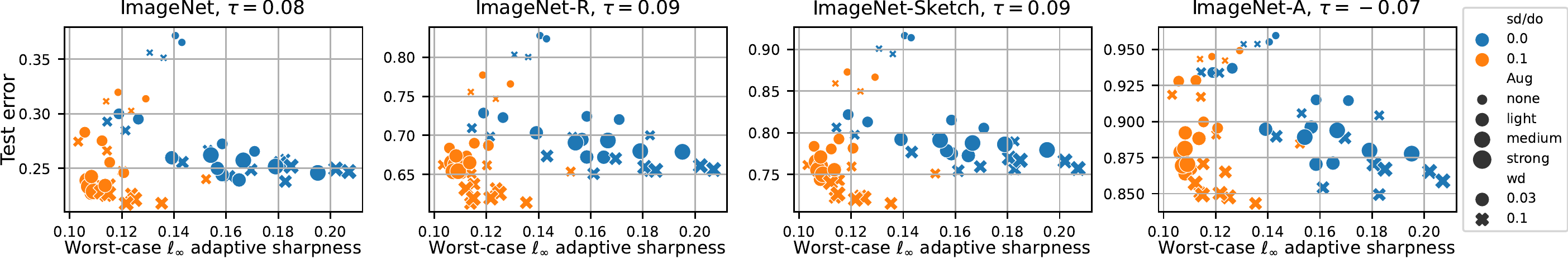} \vspace{1.5mm}\\
        \textbf{Without logit normalization} \\
        \includegraphics[width=0.99\textwidth]{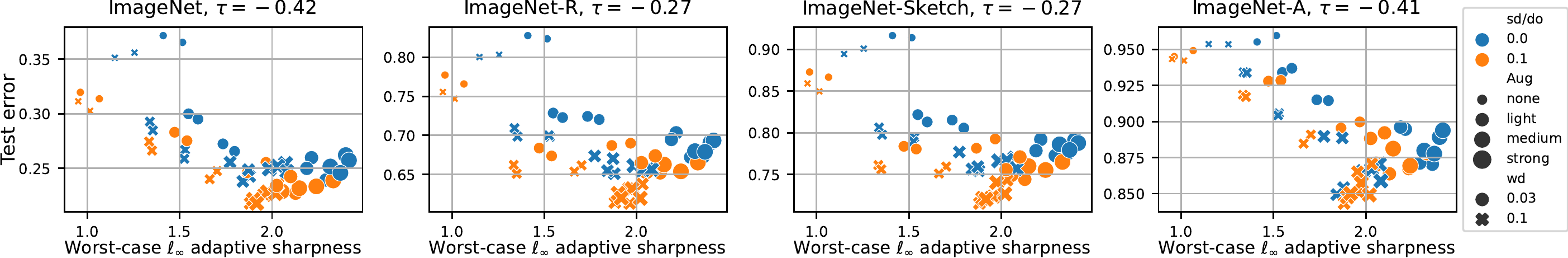}
    \end{tabular}
    \vspace{-2.5mm}
    \caption{\textbf{ViT-B/16 trained from scratch on ImageNet-1k}. 
        We show for 56 models from \citet{steiner2021train} the test error on ImageNet and its OOD variants vs. worst-case $\ell_\infty$ sharpness with (top) or without (bottom) normalization at $\rho=0.002$. The color indicates models trained with stochastic depth (sd) and dropout (do), markers and their size indicate the strength of weight decay (wd) and augmentations (aug), and $\tau$ indicates the rank correlation coefficient from Eq.~\eqref{eq:tau}. Overall, the correlation of sharpness with test error is either close to zero or even negative.}
    \label{fig:IN1k-scratch-main}
\end{figure*}
The current understanding of the relationship between sharpness and generalization is based on experiments on non-residual convolution networks and small datasets like CIFAR-10 and SVHN~\citep{jiang2019fantastic}. 
We revisit here this relationship for state-of-the-art transformers trained from scratch on ImageNet-1k and CLIP / BERT fine-tuned on ImageNet-1k / MNLI.
We explore both in-distribution (ID) and out-of-distribution (OOD) generalization due to the common intuition that flatter models are expected to be more robust~\citep{cha2021swad}. 
We focus on worst-case $\ell_\infty$ adaptive sharpness with low $m$ ($256$) since it appears to be one of the most promising sharpness definitions \citep{kwon2021asam}. 
We compute sharpness with and without logit normalization, and provide \textit{average-case} sharpness for different radii $\rho$ in Appendix.
We focus primarily on the relationship between sharpness and \textit{test error} but we also discuss sharpness vs. \textit{generalization gap} in Sec.~\ref{sec:app_gen_gap} in Appendix.

\myparagraph{Training on ImageNet-1k from scratch.}
To investigate the relationship between sharpness and generalization for large-scale settings, we evaluate ViT models from \citet{steiner2021train}, using ViT-B/16-224 weights. Those were trained from scratch on ImageNet-1k for 300 epochs with different hyperparameter settings, and subsequently fine-tuned on the same dataset for 20.000 steps with 2 different learning rates. The different hyperparameters include augmentations, weight decay, and stochastic depth / dropout, leading to a rich pool of 56 models with test errors ranging from 21.8\% to 37.2\%. 
As shown in Figure~\ref{fig:IN1k-scratch-main} (first column), neither the sharpness measure computed with nor without logit normalization can effectively distinguish model performance. 
Logit-normalized sharpness effectively separates  models with stochastic depth / dropout (sd/do from now on) from those without by grouping them into two distinct clusters (blue and orange). However, these clusters do not correspond to a separation by test error.
For the OOD tasks (ImageNet-R, ImageNet-Sketch, ImageNet-A), %
within each cluster, the models trained with higher weight decay yield lower test error fairly consistently. However, this ranking is not captured by sharpness, which only disentangles the sd/do clusters.
For sharpness without logit normalization, the sd/do clusters are not well-separated. 
Surprisingly, there is a consistent \textit{negative} correlation between sharpness and test error, both on ID and OOD data, 
i.e. the flattest models tend to have the largest test error.
Evaluation for other radii, average-case sharpness measures (App.~\ref{sec:app_IN-1k-extrafigures}) and for ViTs pretrained on IN-21k and fine-tuned on IN-1k (App.~\ref{sec:app_IN21k-IN1k-extrafigures}) similarly suggest that sharpness does not consistently capture generalization properties.
When considering IN-1k and IN-21k pre-trained models together (App.~\ref{sec:app_bothIN21k-IN1k-extrafigures}) we even find similar or \textit{higher} sharpness for significantly better-generalizing models. 
Then, for none of the settings studied, we can confirm either the strong or weak hypotheses.

\myparagraph{Fine-tuning on ImageNet-1k from CLIP.}\label{sec:clip_fine-tuning}
\begin{figure*}[t!] \centering \small
    \begin{tabular}{@{}c@{}}
        \textbf{With logit normalization} \\
        \includegraphics[width=1.0\linewidth]{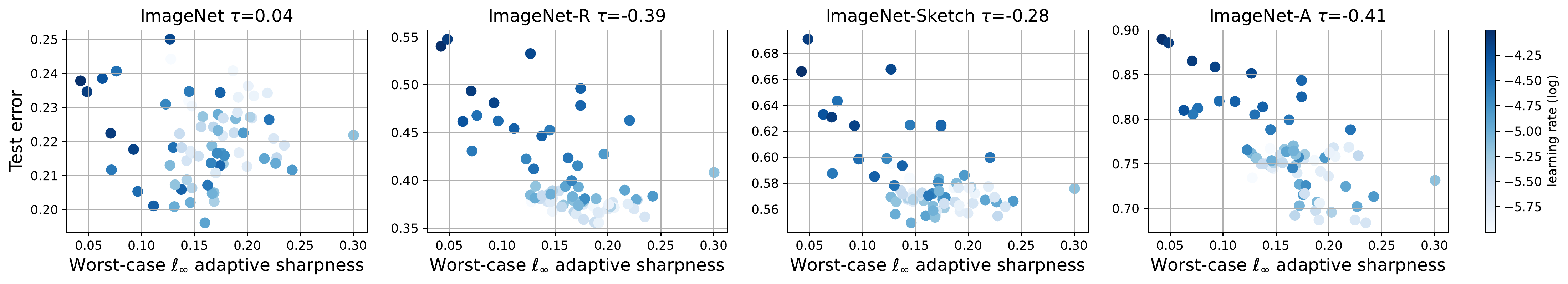}\\
        \textbf{Without logit normalization} \\
        \includegraphics[width=1.0\linewidth]{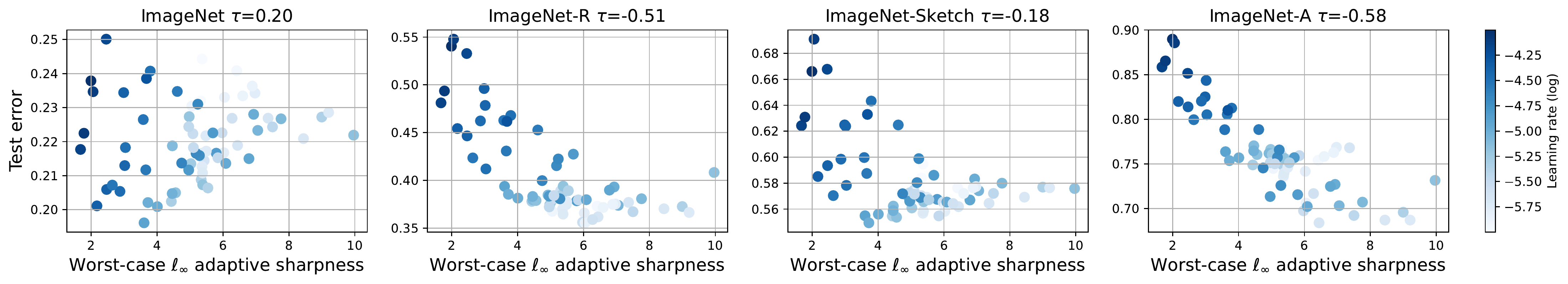}
    \end{tabular}
    \vspace{-3mm}
    \caption{\textbf{Fine-tuning CLIP ViT-B/32 on ImageNet-1k.} We show for 72 models from \citet{wortsman2022model} the test error on ImageNet or its variants (distribution shifts) vs worst-case $\ell_\infty$ sharpness with (top) or without (bottom) normalization at $\rho=0.002$. Darker color indicates larger learning rate used for fine-tuning.}\label{fig:clip_singlerho}
\end{figure*}
We investigate fine-tuning from CLIP~\citep{radford2021learning}, which is a crucial approach due to the popularity of CLIP features~\citep{ramesh2022hierarchical}, its fast training time, and its ability to achieve higher accuracy.
We study the pool of classifiers obtained by \citet{wortsman2022model} who fine-tuned a CLIP ViT-B/32 model on ImageNet multiple times by randomly selecting training hyperparameters such as learning rate, number of epochs, weight decay, label smoothing and augmentations.  This set of $71$ fine-tuned models, along with the base model, allows us to study how well generalization and training hyperparameters are captured by sharpness. 
The leftmost column of Fig.~\ref{fig:clip_singlerho} illustrates that worst-case $\ell_\infty$ adaptive sharpness does not effectively predict which classifiers have the lowest test error on ImageNet. 
Furthermore, there is a consistent negative correlation between sharpness and test error when evaluating classifiers on the distribution shifts ImageNet-R \citep{hendrycks2021many}, ImageNet-Sketch \citep{wang2019learning} and ImageNet-A \citep{hendrycks2021nae} (second to fourth columns). 
We further notice that, in contrast with ImageNet, higher test errors on these datasets go in parallel with higher learning rates used for fine-tuning (darker color in the plots). 
Indeed, smaller learning rates lead to smaller changes in the features of the base CLIP model which are more robust to distribution shifts since they were obtained from a much larger dataset than ImageNet. 
Finally, similar observations hold for the other sharpness definition and radii (App.~\ref{sec:app_finetuning_clip_on_imagenet}).

\begin{figure*}[t!] \centering \small
    \begin{tabular}{@{}c@{}}
        \textbf{With logit normalization} \\\includegraphics[width=1.0\linewidth]{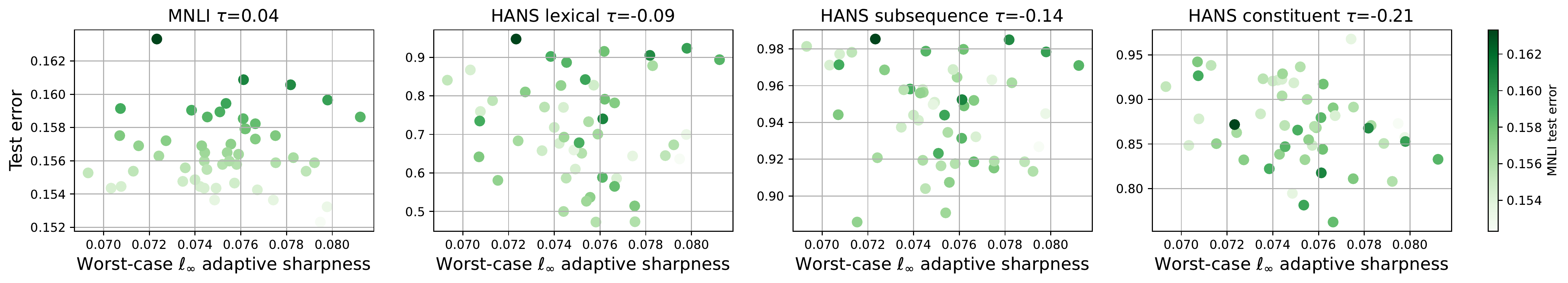}\\
        \textbf{Without logit normalization} \\\includegraphics[width=1.0\linewidth]{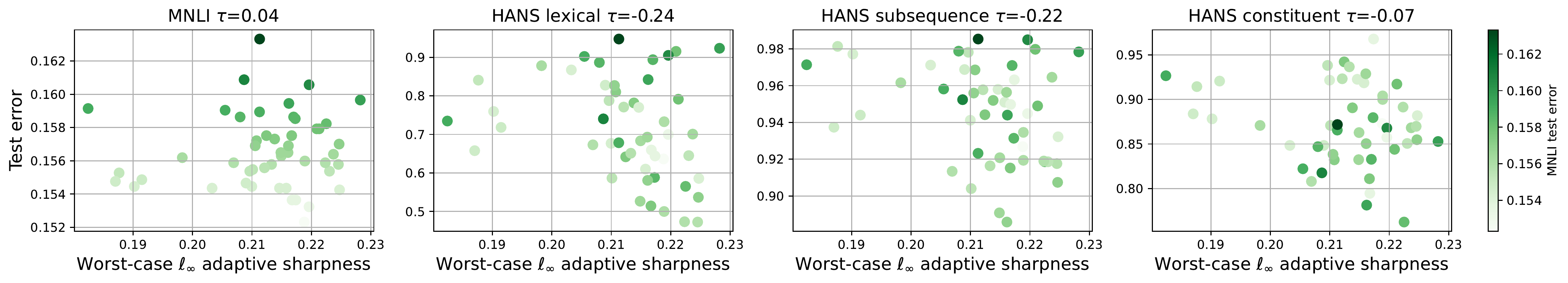}
    \end{tabular}
    \vspace{-3mm}
    \caption{\textbf{Fine-tuning BERT on MNLI.} We show for 50 models  the error on MNLI or out-of-distribution domains (HANS subsets) vs worst-case $\ell_\infty$ sharpness with (top) or without (bottom) normalization at $\rho=0.0005$. Darker color indicates higher test error on MNLI.} \label{fig:mnli_singlerho}
\end{figure*}
\myparagraph{Fine-tuning on MNLI from BERT.}
We explore fine-tuning from BERT~\citep{devlin-etal-2019-bert}, to expand our analysis beyond vision tasks. 
To study the linguistic generalization of multiple classifiers trained on the same dataset, \citet{mccoy-etal-2020-berts} have fine-tuned BERT $100$ times on the Multi-genre Natural Language Inference (MNLI)  dataset \citep{williams2018broadcoverage} varying exclusively the random seed across runs. These random seeds affect the initialization of the classifier and the scanning order of the training data for SGD.
All these classifiers achieve very similar in-distribution generalization, i.e. on MNLI test points, but behave differently on the out-of-distribution tasks represented by the HANS dataset~\citep{mccoy-etal-2019-right}. For example, in one of HANS sub-domains the accuracy of the models ranges from 5\% to 55\%.
We randomly choose $50$ of the $100$ available classifiers, and compute the different measures of sharpness for various radii. Fig.~\ref{fig:mnli_singlerho} shows how the worst-case $\ell_\infty$ adaptive sharpness, with and without logit normalization,  
correlates with test error on MNLI and three HANS tasks.
We observe that the correlation is weak and does not exceed $0.04$, even for datasets like HANS lexical (second column) where test errors vary significantly (between 45\% and 95\%). Moreover, in some cases the correlation is weakly negative suggesting that on average sharper models tend to generalize slightly better. 
Results for other radii can be found in App.~\ref{sec:app_mnli}.

\myparagraph{Summary of the findings.}
To conclude, \textit{none} of the settings studied above support either the strong or weak hypotheses about the role of sharpness.    
Contrary to our expectations, CLIP models fine-tuned on ImageNet suggest that flatter solutions consistently generalize \textit{worse} on OOD data.
Finally, sharpness is not useful to distinguish different solutions found by fine-tuning BERT on MNLI. 
All this evidence suggests that the intuitive ideas about the generalization benefits of flat minima are \textit{not supported in the modern settings}.

\section{Why Doesn't Sharpness Correlate Well with Generalization?}
The goal of this section is to clarify the disconnect between sharpness and generalization in the modern setup.
We first revisit sharpness in a controlled environment on CIFAR-10, then explore the different sharpness definitions for a simple model where generalization is well understood. 

\subsection{The Role of Sharpness in a Controlled Setup} \label{subsec:controlled_setup_cifar10}
\myparagraph{Motivation.}
We consider three potential explanations for why sharpness does not correlate well with generalization in the previous section: 
(1) the use of transformers instead of typical convolutional networks, 
(2) the use of much larger datasets (ImageNet vs. CIFAR-10), 
(3) the need to measure sharpness closer to a global minimum. 
We thus train $200$ ResNets-18 and $200$ ViTs on CIFAR-10 in a setting similar to \citet{jiang2019fantastic} and \citet{kwon2021asam}, and evaluate sharpness only for models that reach \textit{at most $1\%$ training error}. 
This is in contrast to the ImageNet models from the previous section that are not necessarily trained to $\approx0\%$ training error as it is usually not necessary in practice. 
Being closer to a global minimum ensures that the worst-case sharpness captures more the curvature by preventing first-order terms from dominating in Eq.~\ref{eq:max_sharpness_small_rho}.

\begin{figure*}[t!]
    \centering \small
    \tabcolsep=1.1pt
    \begin{tabular}{cccc}
        \multicolumn{2}{c}{\textbf{ResNets-18 with logit normalization}} & \multicolumn{2}{c}{\textbf{ViTs with logit normalization}} \\
        \includegraphics[width=0.245\textwidth]{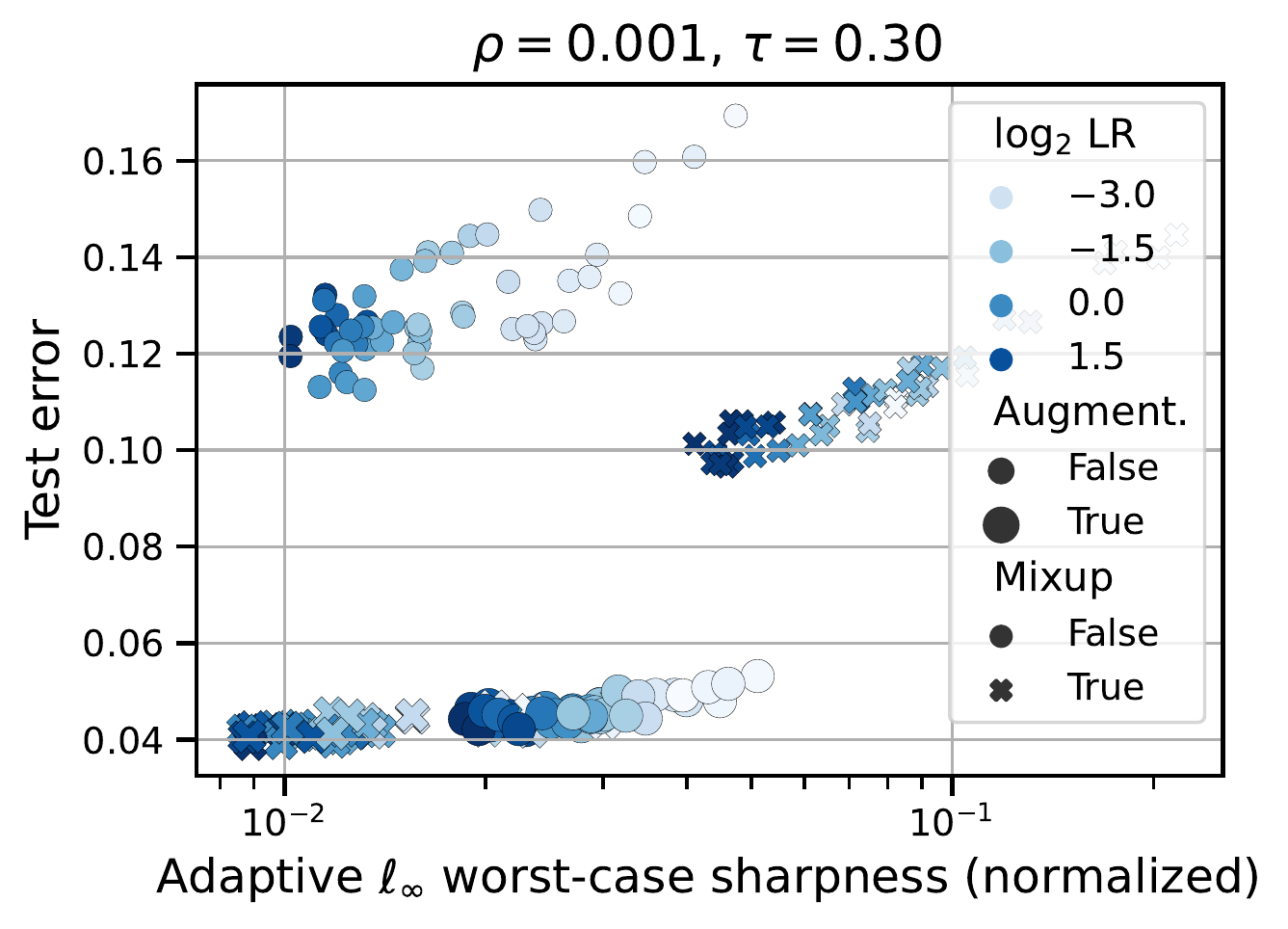} &
        \includegraphics[width=0.25\textwidth]{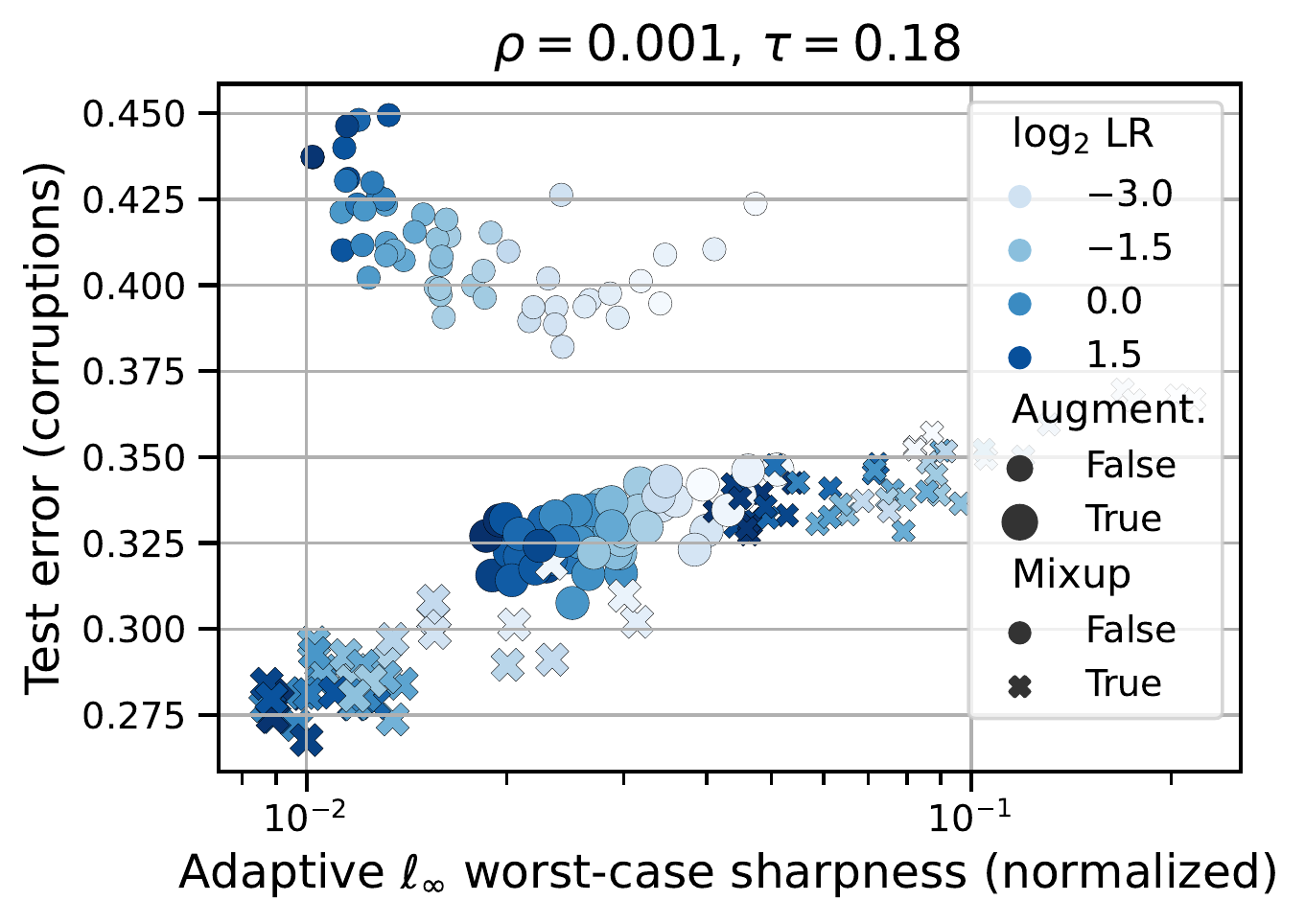} &
        \includegraphics[width=0.245\textwidth]{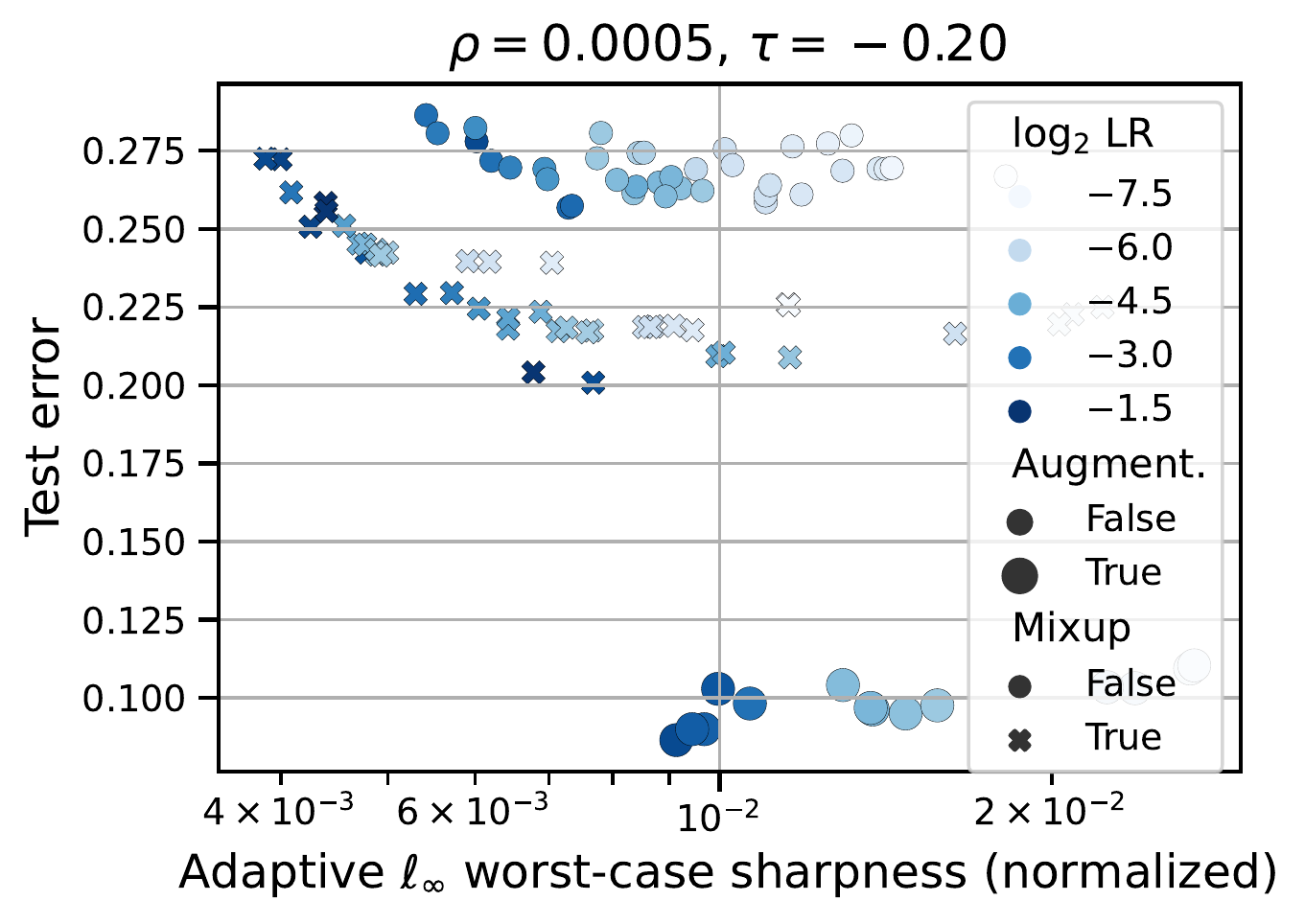} &
        \includegraphics[width=0.24\textwidth]{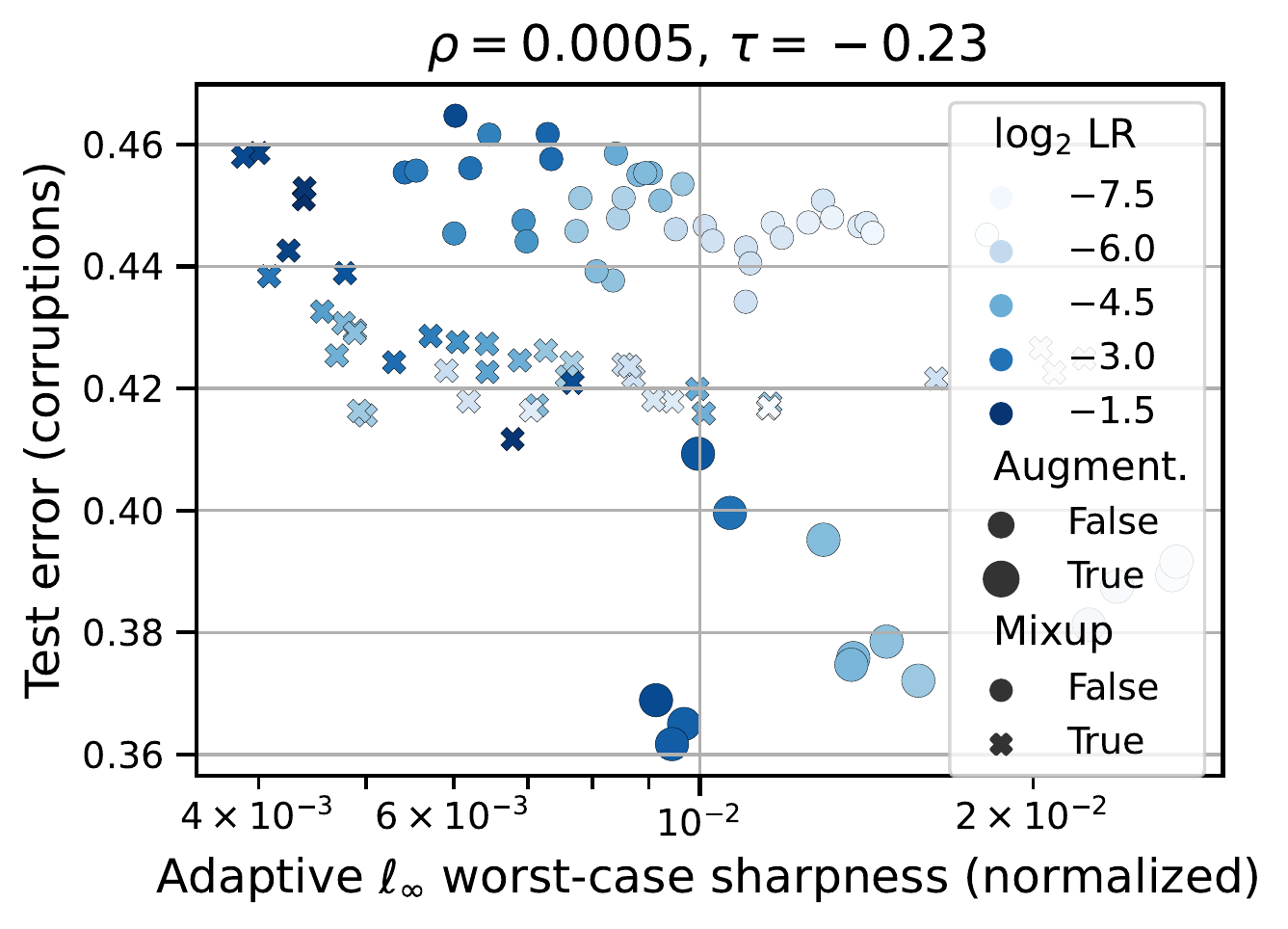} \\
    \end{tabular} 
    \begin{tabular}{cccc}
        \multicolumn{2}{c}{\textbf{ResNets-18 without logit normalization}} & \multicolumn{2}{c}{\textbf{ViTs without logit normalization}} \\
        \includegraphics[width=0.245\textwidth]{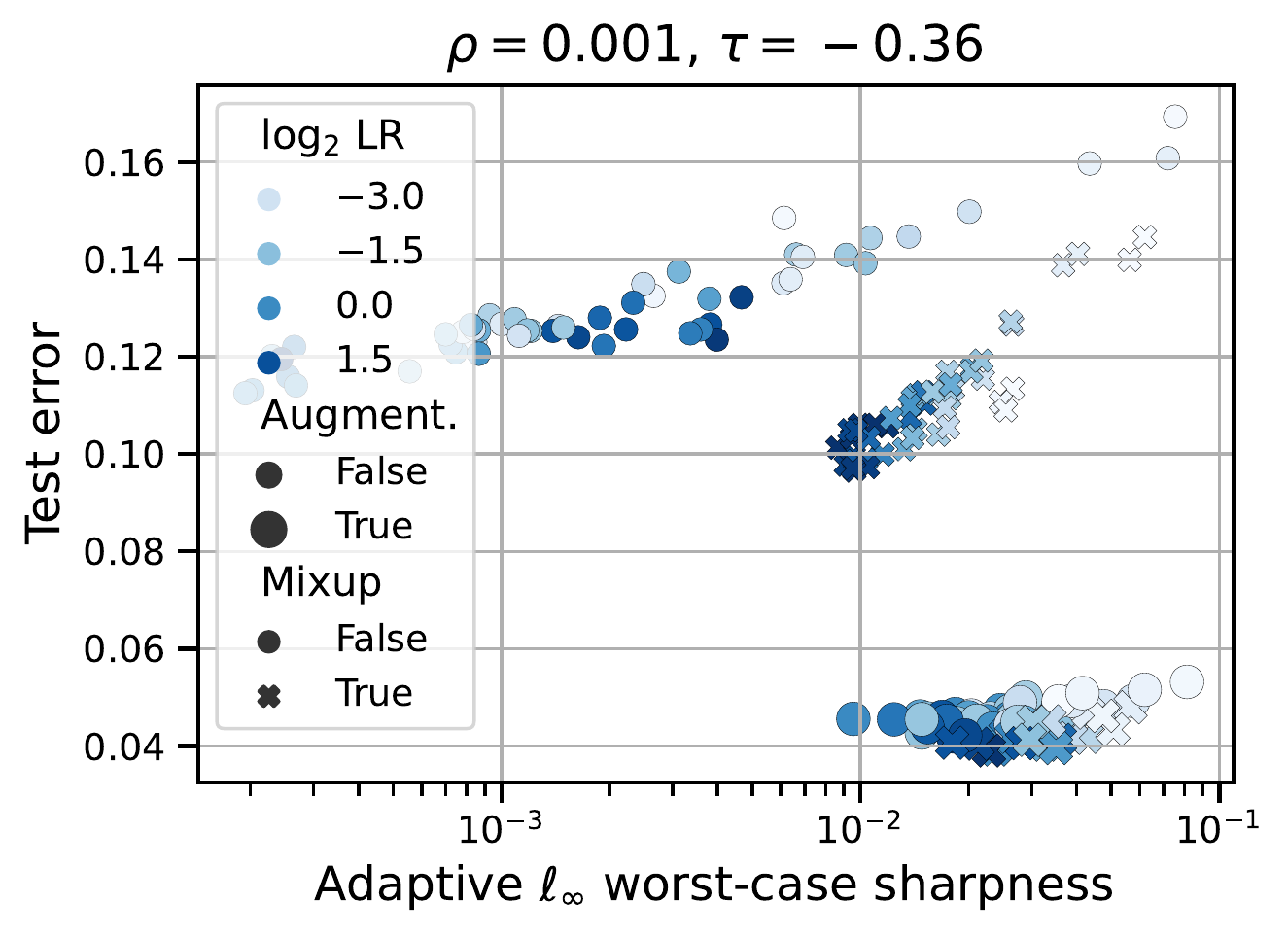} &
        \includegraphics[width=0.25\textwidth]{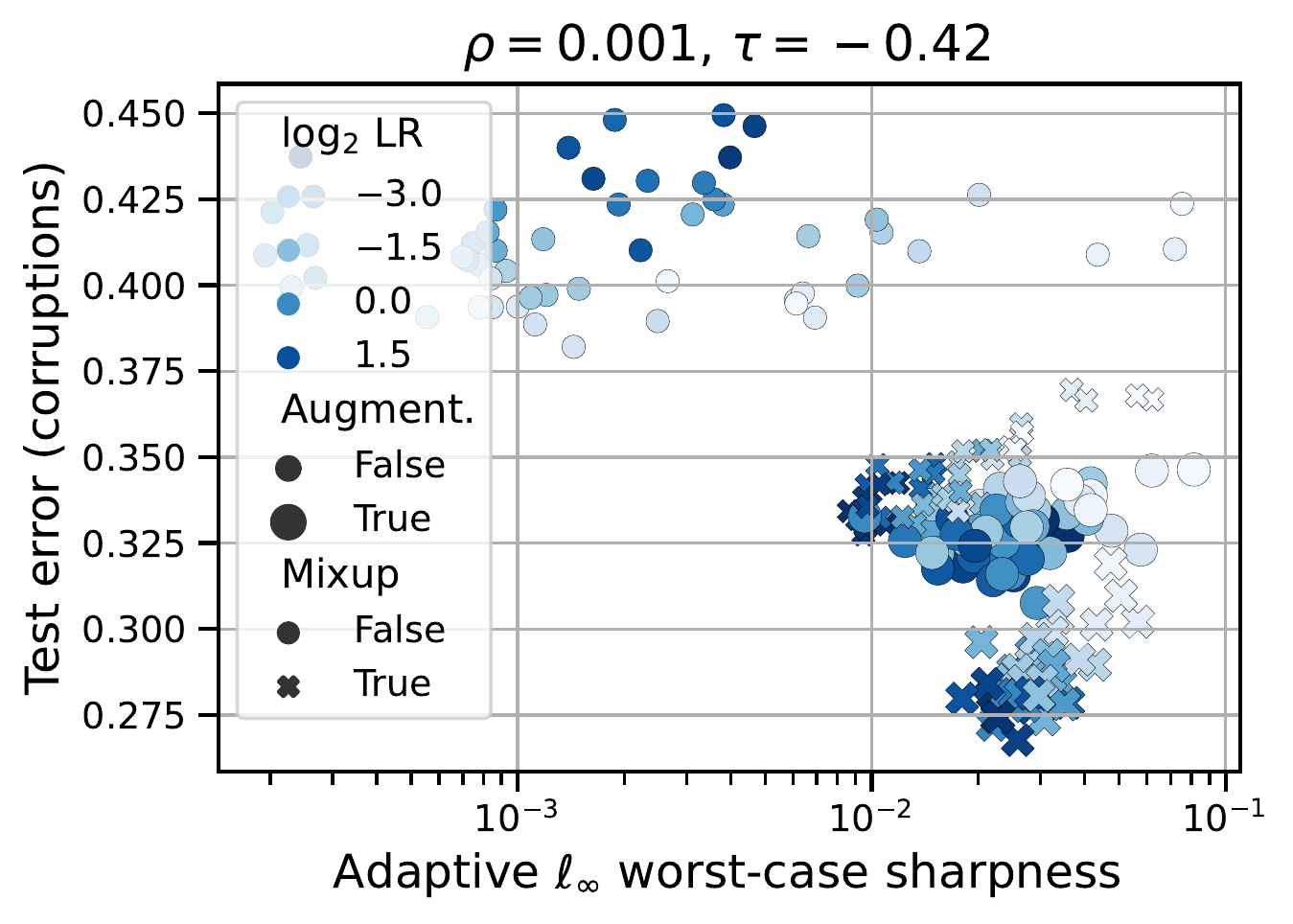} &
        \includegraphics[width=0.245\textwidth]{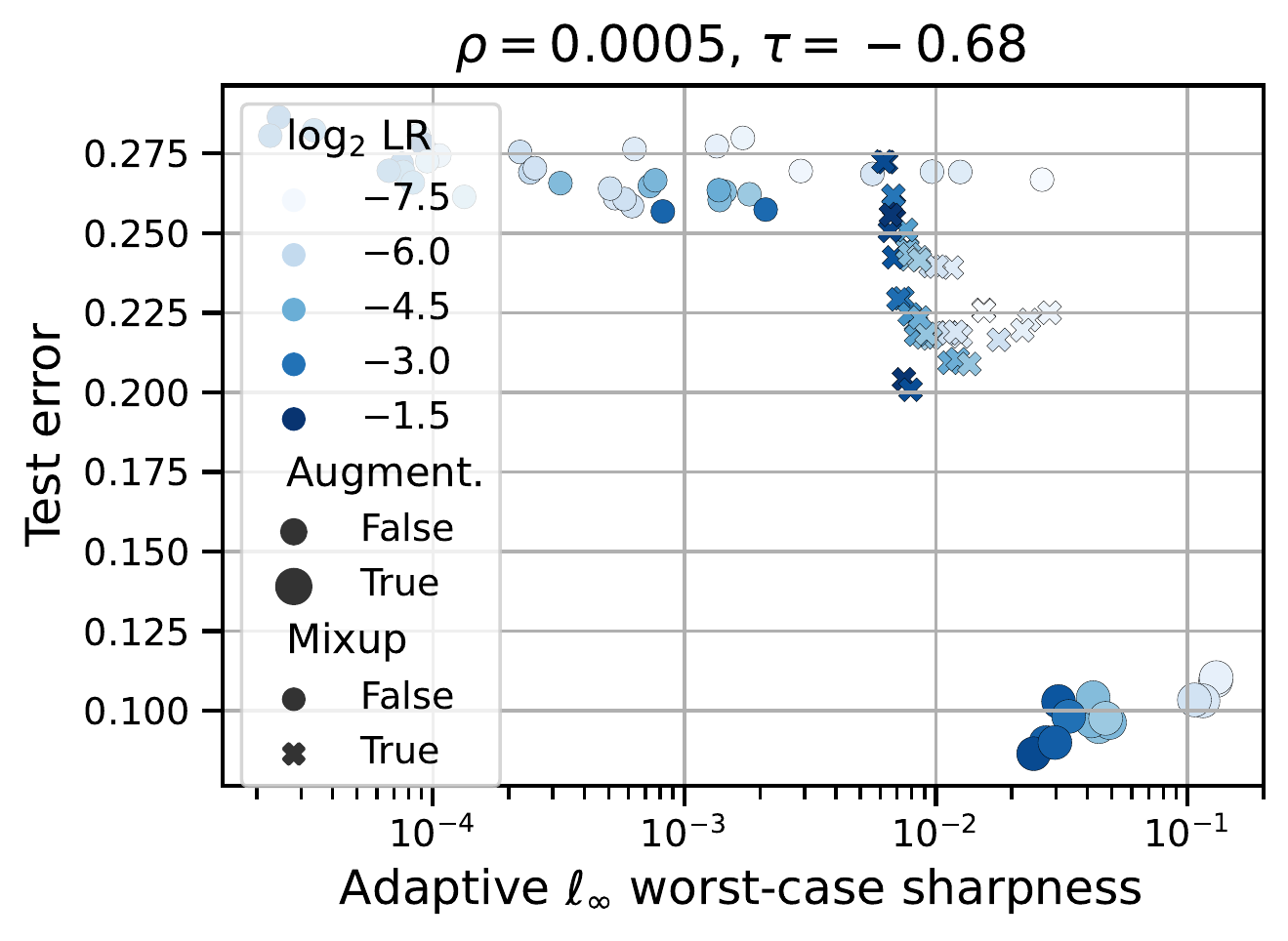} &
        \includegraphics[width=0.24\textwidth]{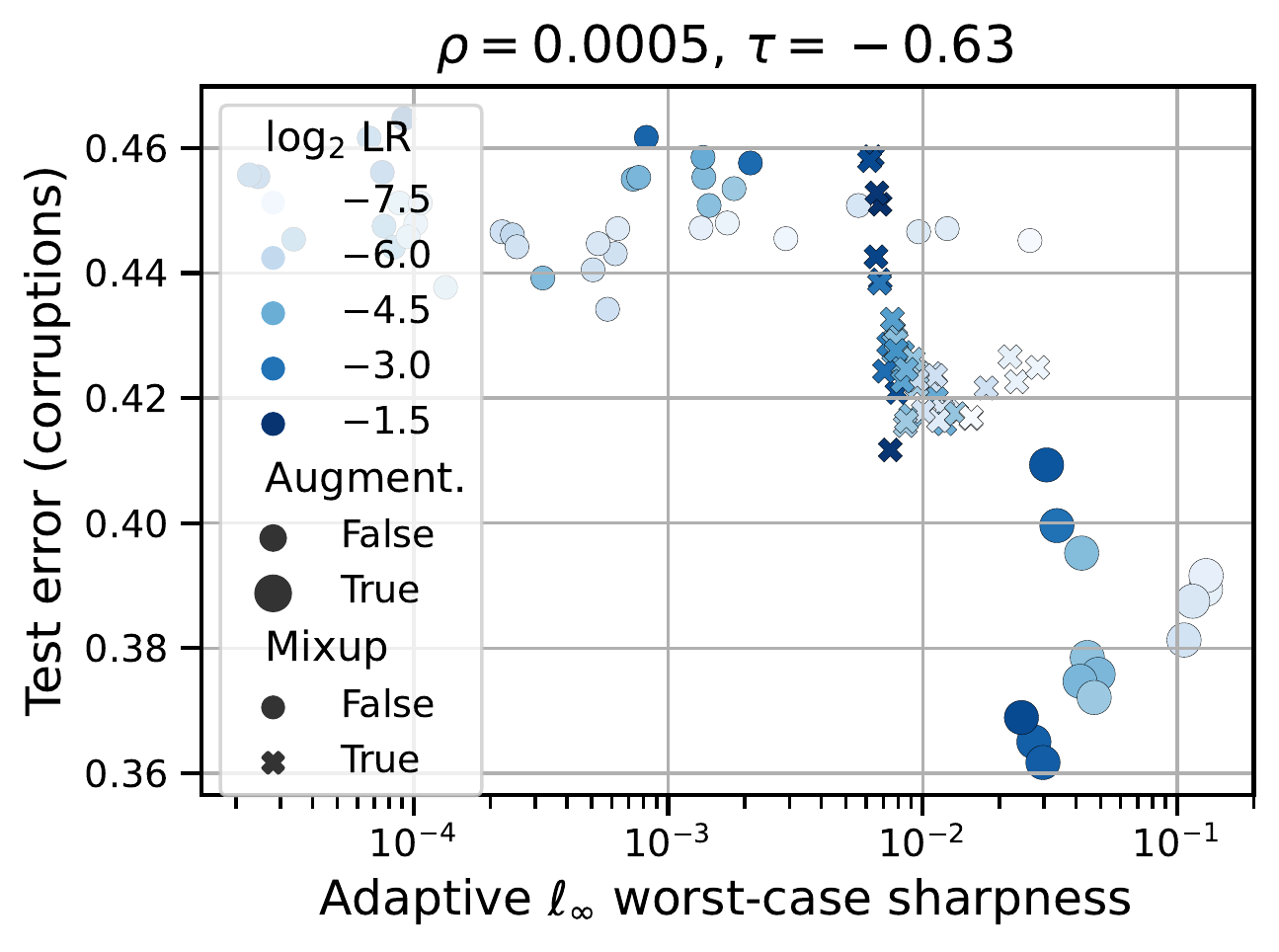} \\
    \end{tabular}
    \vspace{-3mm}
    \caption{\textbf{Training from scratch on CIFAR-10.} Normalized and unnormalized $\ell_\infty$ adaptive sharpness vs. standard and OOD test error on common corruptions for ResNets-18 and ViTs. For other sharpness definitions ($\ell_2$/$\ell_\infty$, average-/worst-case, etc) and multiple sharpness radii $\rho$, see App.~\ref{sec:app_cifar10_role_defns_radii}.}
    \label{fig:cifar10_adaptive_sharpness}
\end{figure*}

\myparagraph{Setup.}
We train models for $200$ epochs using SGD with momentum and linearly decreasing learning rates after a linear warm-up for the first $40\%$ iterations. We use the SimpleViT architecture from the \texttt{vit-pytorch} library which is a modification of the standard ViT \citep{dosovitskiy2020image} with a fixed positional embedding and global average pooling instead of the CLS embedding. We vary the learning rate, $\rho \in \{0, 0.05, 0.1\}$ of SAM \citep{foret2021sharpnessaware}, mixup ($\alpha=0.5$) \citep{zhang2017mixup}, and standard augmentations combined with RandAugment \citep{cubuk2019randaugment}. We only show models that have $\leq 1\%$ training error.

\myparagraph{Observations.}
We benchmark $12$ different sharpness definitions: $\ell_2$ vs. $\ell_\infty$, average- vs. worst-case, standard vs. adaptive, with vs. without logit normalization, and consider different perturbation radii $\rho$. We report most of these results in App.~\ref{sec:app_cifar10_extra_figures} and here highlight only $\ell_\infty$ adaptive sharpness in Fig.~\ref{fig:cifar10_adaptive_sharpness}.
We observe that for ResNets, there is a strong correlation between sharpness and test error but \textit{only within each subgroup of training parameters} such as augmentations and mixup.
Importantly, sharpness does not correctly capture generalization between different subgroups leading to low positive or negative correlation ($0.30$ and $-0.36$).
For ViTs, we do not observe strong positive correlation even within each subgroup (in fact, without logit normalization the correlation is noticeably negative $-0.68$), and many models with an order of magnitude difference in sharpness can have the same test error.
Moreover, we do not consistently observe that models with the lowest sharpness generalize best.
For OOD generalization on common image corruptions \citep{hendrycks2019robustness}, the trend is even less clear and the subgroups are mixed.
We note that similar conclusions hold for other sharpness radii $\rho$ and definitions which we show in App.~\ref{sec:app_cifar10_role_defns_radii}. Moreover, in App.~\ref{sec:app_cifar10_extra_figures} we also analyze the role of data points used to evaluate sharpness (with and without augmentations), number of iterations of Auto-PGD for worst-case sharpness, and different $m$ in worst-case $m$-sharpness \citep{foret2021sharpnessaware}.
In conclusion, even in this controlled small-scale setup that includes more established architectures like ResNets, we find no empirical support to either the strong or weak hypothesis.

\begin{figure}[t!] \centering\small
    \tabcolsep=1.1pt
    \newl=.495\columnwidth
    \begin{tabular}{c c} 
        \hspace{3mm} \textbf{ResNets-18} & \hspace{3mm} \textbf{Vision transformers} \\
        \includegraphics[width=\newl]{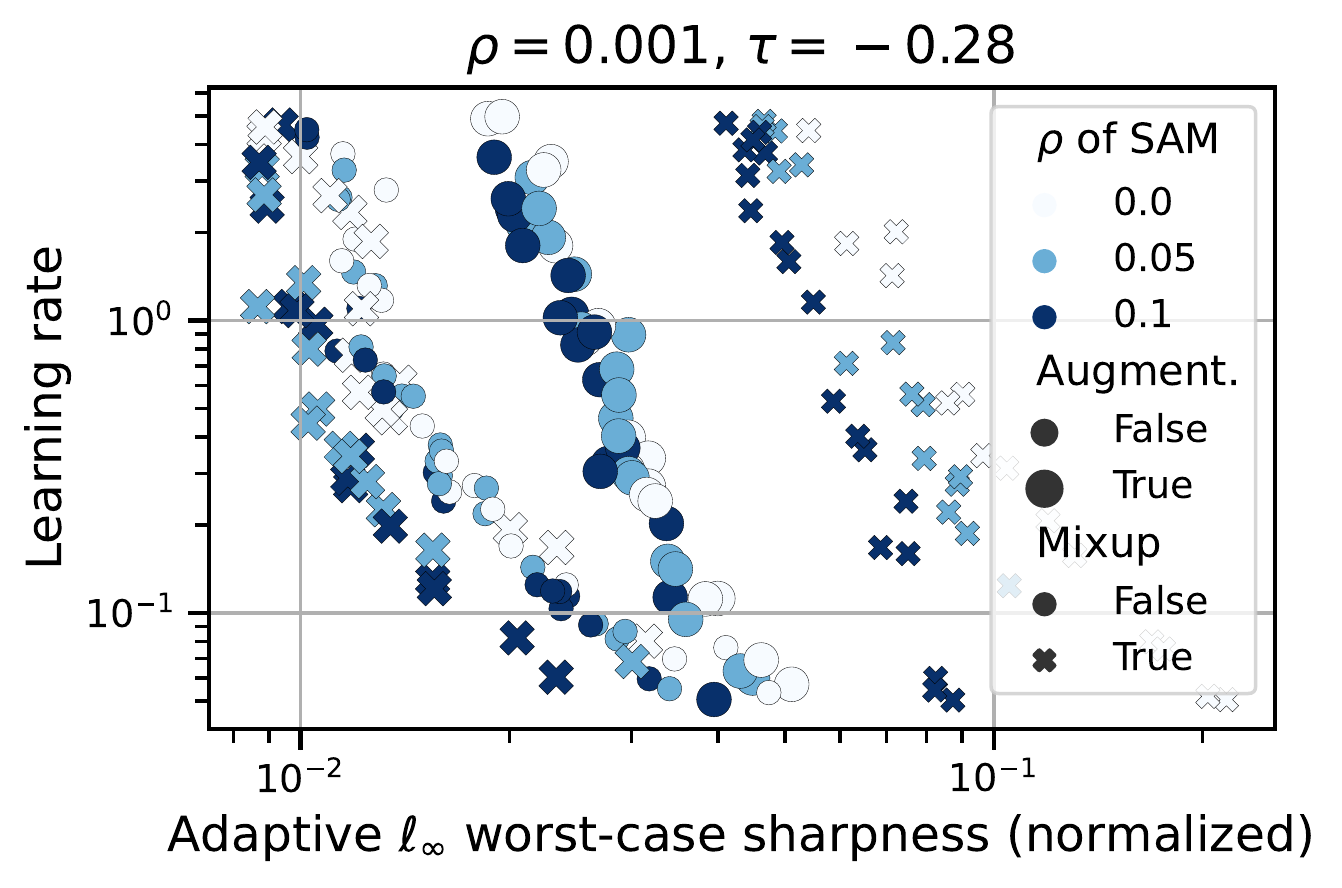} &
        \includegraphics[width=\newl]{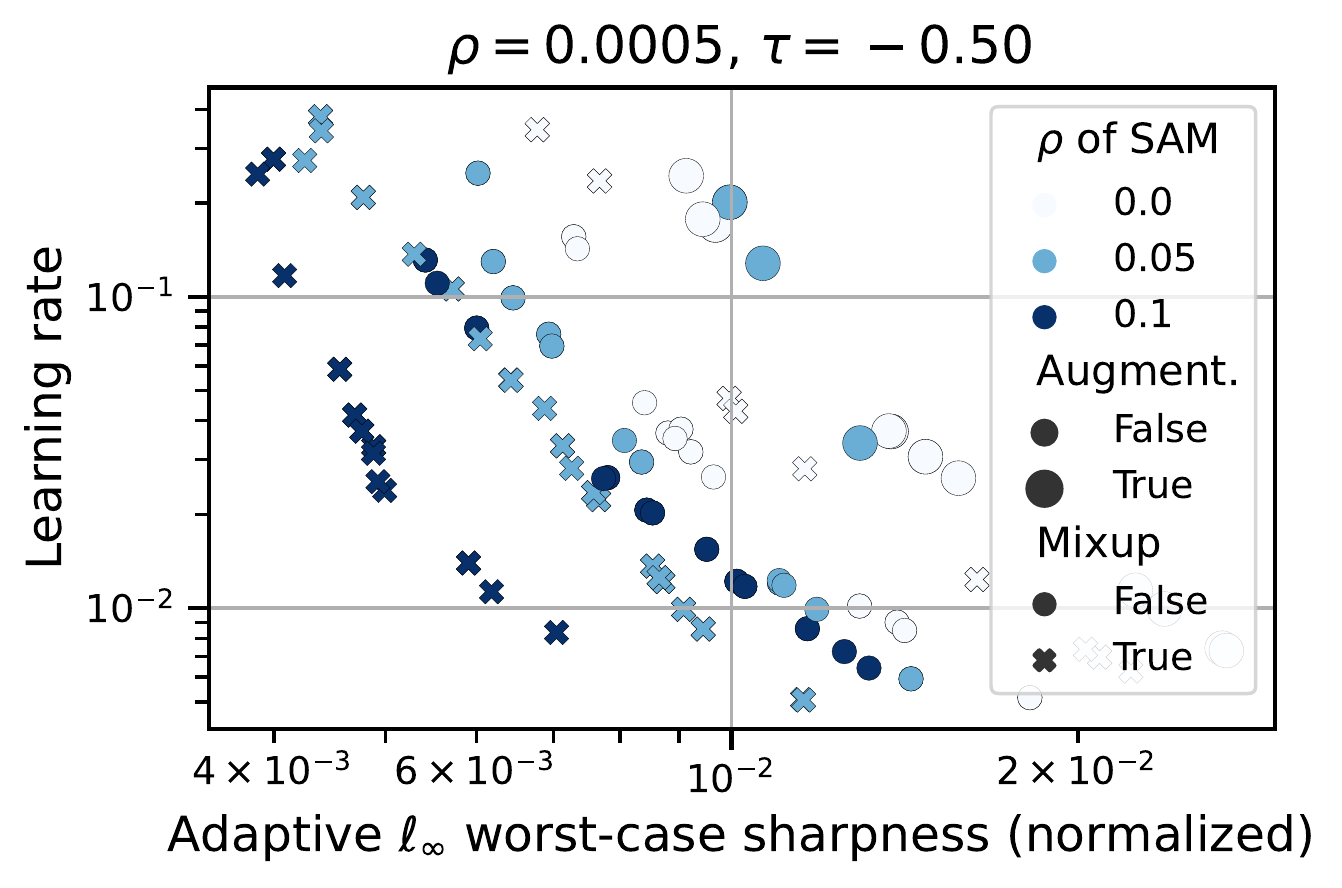}\\
        \includegraphics[width=\newl]{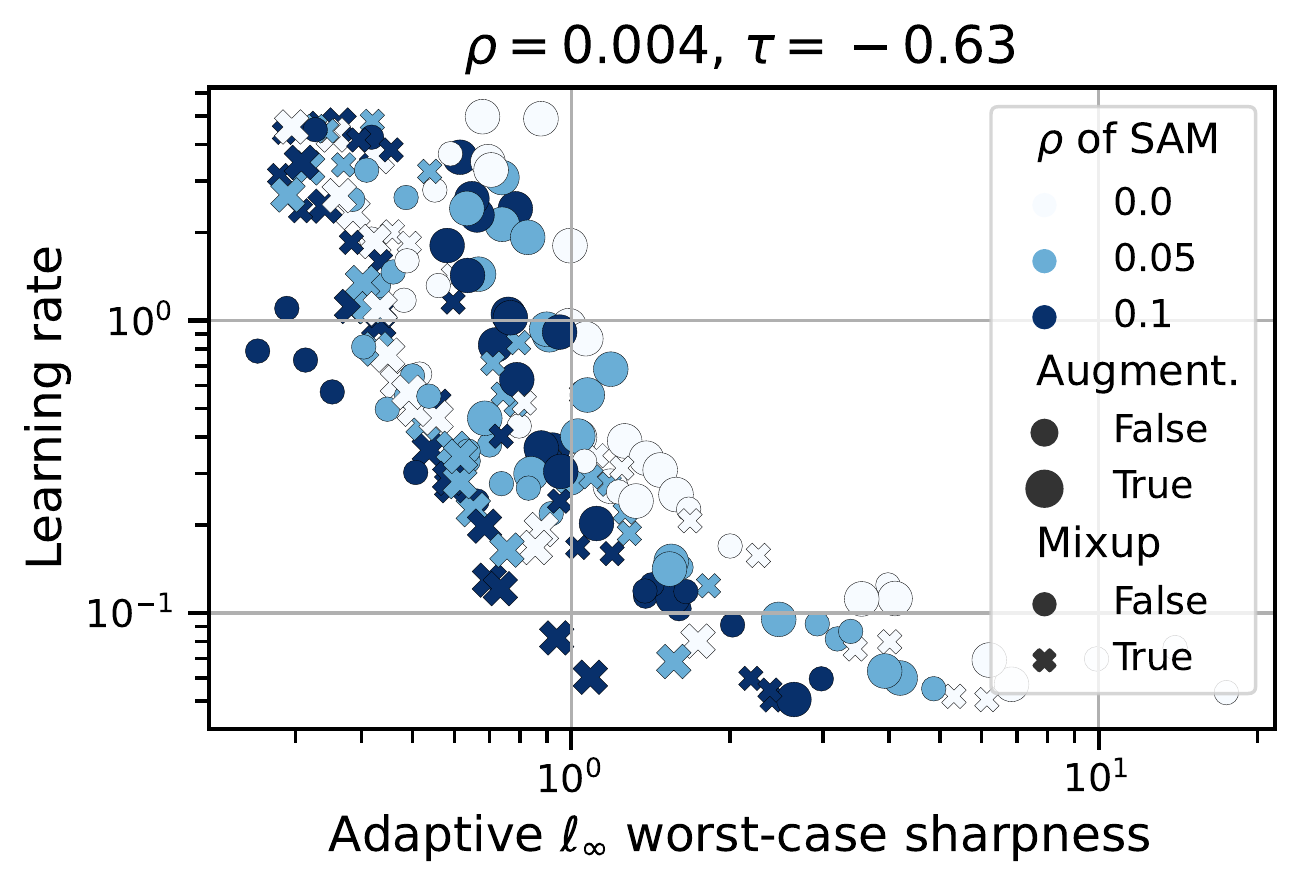} &
        \includegraphics[width=\newl]{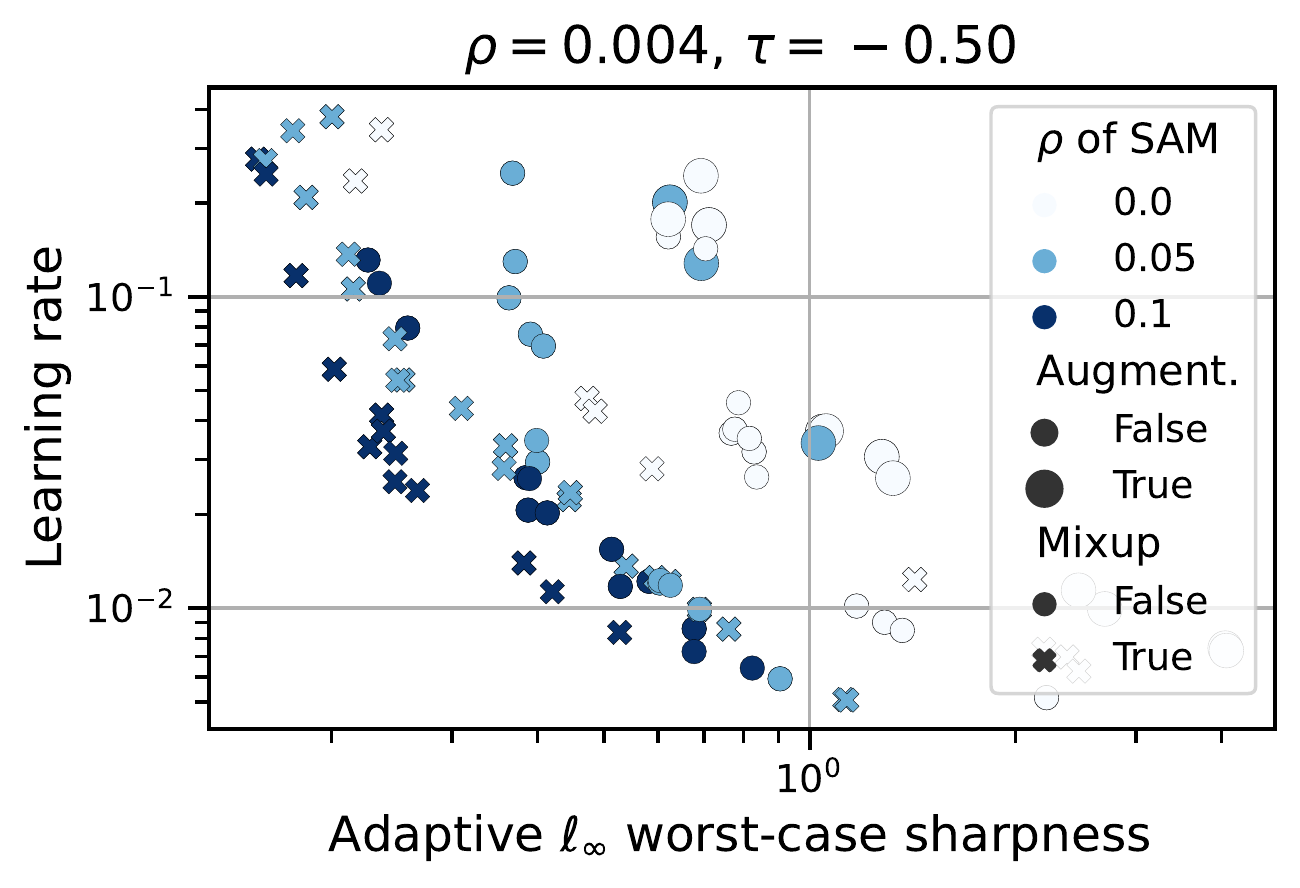}\\
    \end{tabular}
    \vspace{-2mm}
    \caption{\textbf{Training from scratch on CIFAR-10.} Sharpness negatively correlates with the \textit{learning rate}, especially within each subgroup defined by the same values of \texttt{augment} $\times$ \texttt{mixup}.} 
    \label{fig:cifar10_lr_sharpness}
\end{figure}
\myparagraph{Sharpness captures the learning rate even when it is not helpful to predict generalization.} 
Prior works have shown a robust link between the learning rate of first-order methods and standard sharpness definitions such as $\lm(\nabla^2 L(\wv))$ and $\tr(\nabla^2 L(\wv))$ \citep{cohen2021gradient, wu2022does}. However, the connection between the learning rate and \textit{adaptive} sharpness remains elusive, so we investigate it empirically in Fig.~\ref{fig:cifar10_lr_sharpness}.
For both ResNets and ViTs, we observe a significant negative correlation, especially within each subgroup defined by the same values of \texttt{augment} $\times$ \texttt{mixup}. 
This is however \textit{not} always a desirable property for predicting generalization. On the one hand, monotonically capturing the learning rates can be useful in setting like training ResNets from scratch \citep{li2019towards}. On the other hand, large learning rates do not preserve the original features and can significantly harm OOD generalization for fine-tuning~\citep{wortsman2022robust}. 
We also see a negative correlation between sharpness and learning rate for CLIP models fine-tuned on ImageNet in Fig.~\ref{fig:imagenet_clip_lr_sharpness}, shown in App.~\ref{sec:app_finetuning_clip_on_imagenet}.
However, for these models, we do not have subgroups as clearly defined as for the CIFAR-10 models so we cannot see a more fine-grained trend. 
Finally, we note that whenever learning rates have a beneficial regularization effect, it is closely tied to the amount of stochastic noise in SGD \citep{jastrzkebski2017three, andriushchenko2022sgd}. This amount is equally determined by other hyperparameters like batch size, momentum coefficient, or weight decay for normalized networks (see \citet{li2020reconciling} for a discussion on the intrinsic learning rate). These parameters are commonly varied in studies on sharpness vs. generalization \citep{jiang2019fantastic, kwon2021asam, bisla2022low} but all reflect essentially the same underlying trend.

\subsection{Is Sharpness the Right Quantity in the First Place? Insights from Simple Models} \label{sec:simple_models}
Here, we study the link between sharpness and generalization for sparse regression with \textit{diagonal linear networks} for which the $\ell_1$ norm of the solution is predictive of generalization. 
This simple model suggests that sharpness measures which are universally correlated with better generalization across all possible data distributions simply do not exist.

Diagonal linear networks are defined as predictors  $\langle \xv, \betav\rangle$ with parameterization $\betav = \uv \odot \vv$ for weights $\wv = \left[\begin{smallmatrix}\uv\\\vv\end{smallmatrix}\right] \in \R^{2 d}$. They have been widely studied as the simplest non-trivial neural network \citep{woodworth2020kernel, pesme2021implicit}.
We consider an overparametrized sparse regression problem for a data matrix $\vX \in \R^{n \times d}$ and label vector $\yv$:
\begin{align}\label{eq:dln_objective}
L(\wv) :=  \| \vX ( \uv \odot \vv) - \yv\|_2^2, 
\end{align}
for which the ground truth $\betav^*$ is a sparse vector (i.e., most coordinates are zeros) and there exist many solutions $\wv$ such that $L(\wv)=0$. Assuming whitened data $\vX^\top \vX = \vI$ and that $\wv$ is a global minimum, 
the Hessian of the loss $L$ simplifies to 
\begin{align*}
\nabla^2 L(\wv) = 
\begin{bmatrix}
\diag(\vv \odot \vv) & \diag(\uv \odot \vv)  \\
\diag(\uv \odot \vv) & \diag(\uv \odot \uv)  
\end{bmatrix}.
\end{align*}
We first consider standard definitions of \textit{local} (i.e., $\rho \to 0$) sharpness for which we have a closed-form expression. 
The average-case local sharpness is equal to  $\tr(\nabla^2 L(\wv)) = \sum_{i=1}^d u_i^2 + v_i^2$ while the worst-case local sharpness at a minimum is $\lm(\nabla^2 L(\wv)) = \maxop_{1 \leq i \leq d} v_i^2 + u_i^2$ (see Sec.~\ref{sec:app_proofs_dln} for details). 
Importantly, both average- and worst-case local sharpness are not invariant under $\alpha$-reparametrization $(\alpha \uv, \vv / \alpha)$ while the predictor $\betav=\uv\odot \vv$ is.  This fact emphasizes the need for a measure of the sharpness that adjusts to the changing scale of the parameters as the adaptive sharpness.
Indeed, with the carefully selected elementwise scaling $c_i = \sqrt{|v_i| / |u_i|} $ for $1 \leq i \leq d$ and $c_i = \sqrt{|u_i| / |v_i|}$ for $d < i \leq 2d$, we obtain for the average-case and worst-case adaptive local sharpness 
\begin{align*}
&S_{avg}^\rho(\wv, \cv) = \frac{1}{2} \!\sum_{i=1}^d\!u_i^2 |v_i| / |u_i| + \frac{1}{2}\!\sum_{i=1}^d\!v_i^2 |u_i| / |v_i| =\| \betav \|_1,\\ &S_{max}^\rho(\wv, \cv) = \maxop_{1 \leq i \leq d} |u_i| |v_i| = \| \betav \|_\infty.
\end{align*} 
We first note that both definitions of adaptive sharpness are invariant under $\alpha$-reparametrization as they only depend on the predictor $\betav$. However, average and worst-case sharpness do not capture the same properties of $\betav$. In particular, $\|\betav\|_1$ is a generalization measure that correctly captures the sparsity of the linear predictor which is a good indicator of generalization for a \textit{sparse} $\betav^*$. In contrast, $\|\betav\|_\infty$ is a generalization measure that is more suitable to capture how uniform the weights of $\betav$ are which is a good predictor of generalization for a \textit{dense} $\betav^*$. 
Finally, we note that using $\cv = \wv$ in adaptive sharpness would instead lead to $\|\betav\|_2^2$ and $\|\betav\|_\infty^2$ that would have a different interpretation. This simple model highlights that the sharpness definition that correlates well with generalization is data-dependent and in general $S_{avg}$ and $S_{max}$ capture very different trends. 
\begin{figure}[t]
    \centering
    \begin{subfigure}[t]{.235\textwidth}
        \includegraphics[width=1.0\columnwidth]{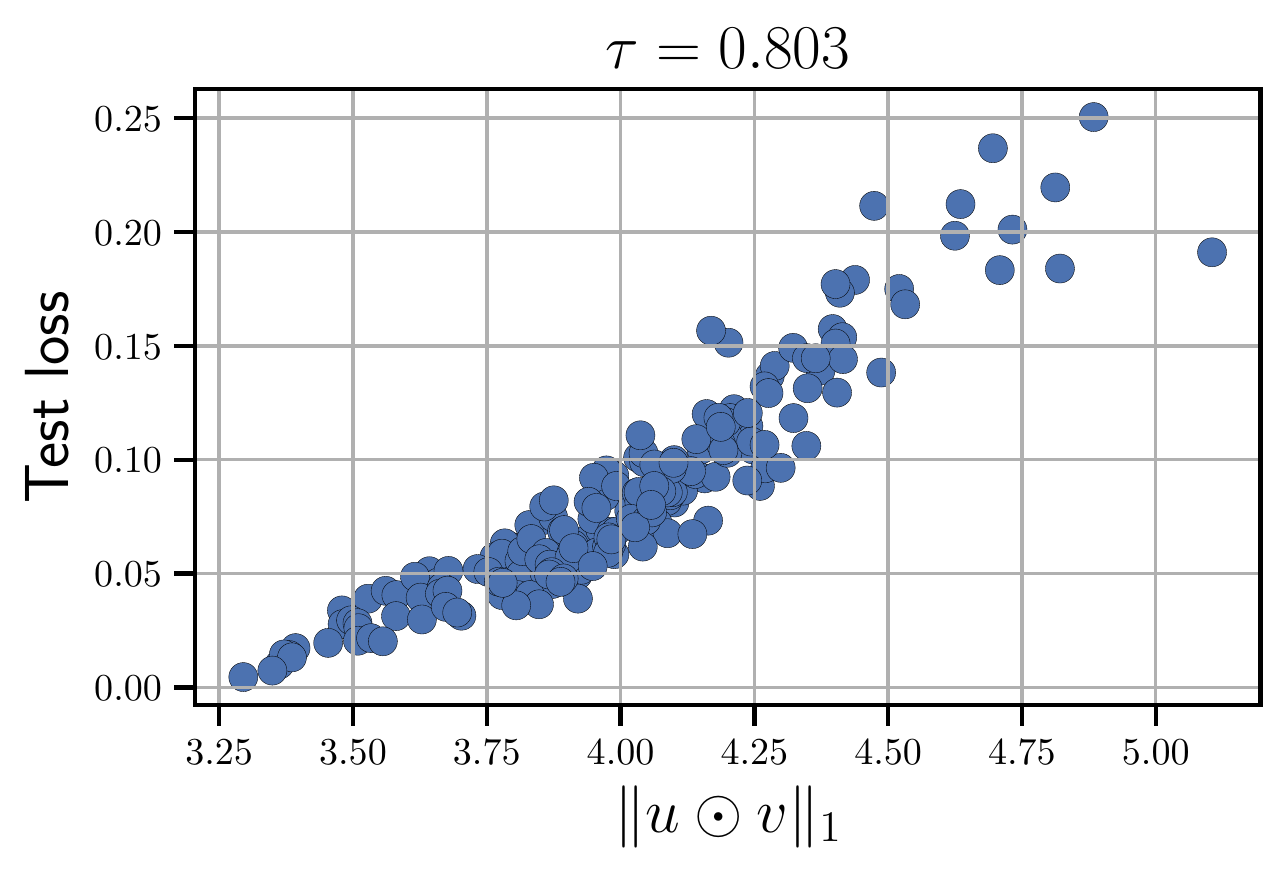}
    \end{subfigure}
    \begin{subfigure}[t]{.235\textwidth}
        \includegraphics[width=1.0\columnwidth]{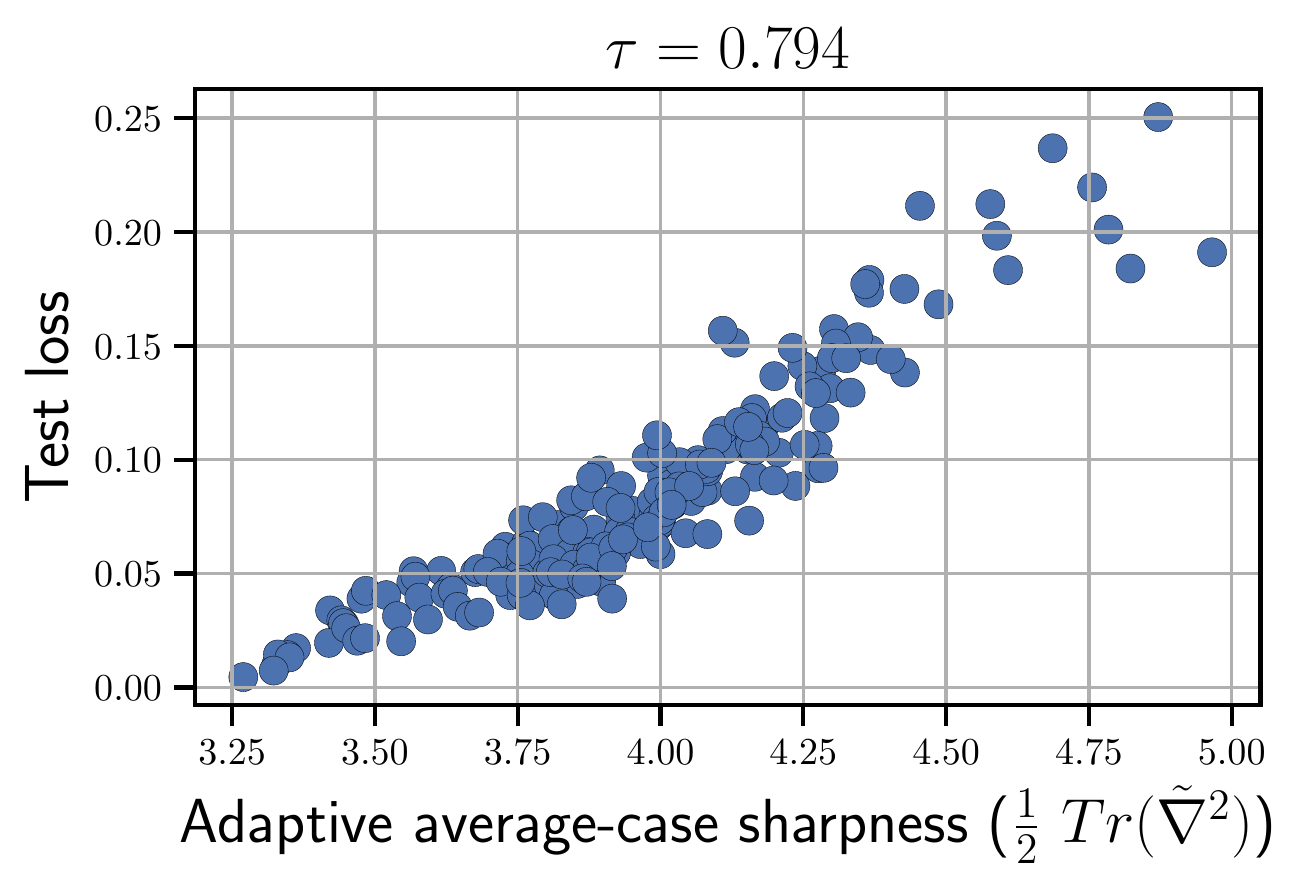}
    \end{subfigure}
    
    \begin{subfigure}[t]{.235\textwidth}
        \includegraphics[width=1.0\columnwidth]{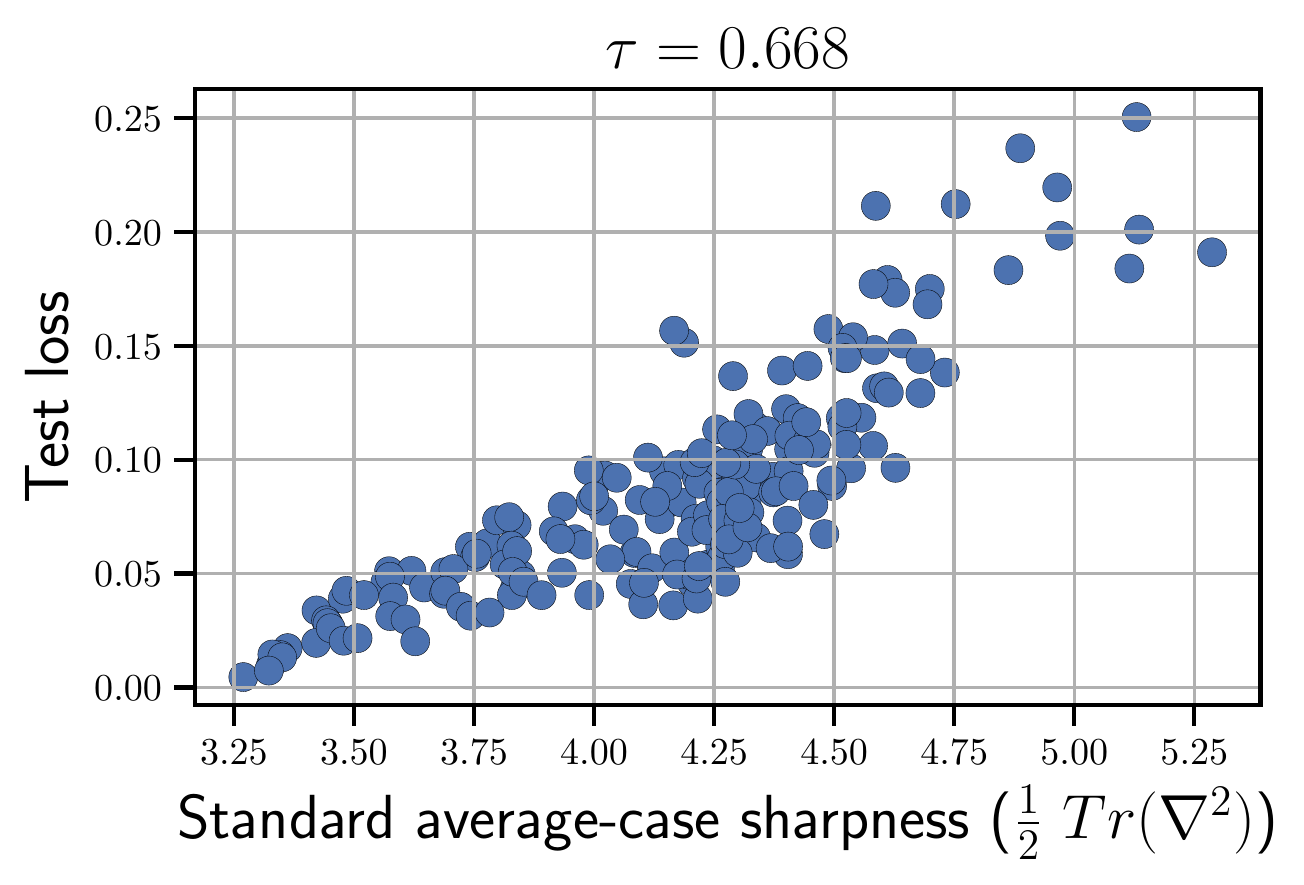}
    \end{subfigure}
    \begin{subfigure}[t]{.235\textwidth}
        \includegraphics[width=1.0\columnwidth]{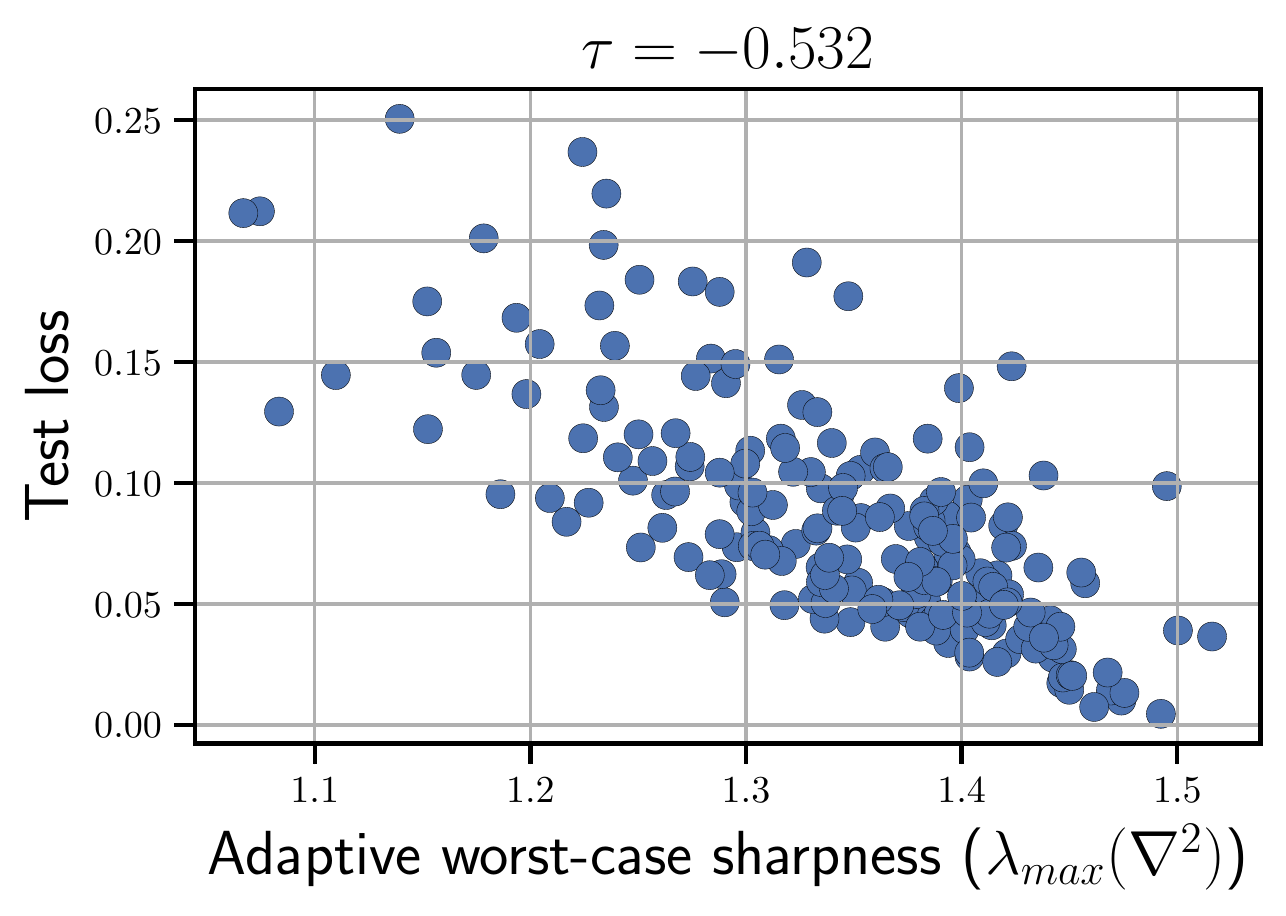}
    \end{subfigure}
    \vspace{-3mm}
    \caption{\textbf{Different generalization measures for diagonal linear networks}. $\tilde{\nabla}^2$ denotes the rescaled Hessian corresponding to adaptive sharpness.}
    \label{fig:diagonal_nets}
\end{figure}

To further illustrate this point, we train $200$ diagonal linear networks to $10^{-5}$ training loss on a sparse regression task ($d=200$ with $90\%$ sparsity) with different learning rates and random initializations. 
We show the results in Fig.~\ref{fig:diagonal_nets} which illustrate that 
(1) $\| \uv \odot \vv \|_1$ is approximated well by $\frac{1}{2} \tr(\tilde{\nabla}^2 L(\wv))$, 
(2) $\tr(\tilde{\nabla}^2 L(\wv))$ correlates better than $\tr(\nabla^2 L(\wv))$ so the adaptive part is important,
(3) the relationship between $\tr(\tilde{\nabla}^2 L(\wv))$ and $\lm(\tilde{\nabla}^2 L(\wv))$ can be even reverse showing that different sharpness definitions capture totally different trends. We also note that even with the right definition of sharpness, the correlation is not perfect (around $\tau=0.8$) and there is always some non-negligible gap in predicting the test loss. 
Overall, we conclude that finding a sharpness definition that correlates well with generalization requires understanding both the role of the data distribution and its interaction with the architecture. It is possible in very simple cases but appears extremely challenging for complex architectures like vision transformers on complex real-world datasets like ImageNet.

\section{Conclusions}
Our results suggest that even reparametrization-invariant sharpness is \textit{not} a good indicator of generalization in the modern setting. 
While there definitely exist restricted settings where correlation between sharpness and generalization is significantly positive (e.g., for ResNets on CIFAR-10 with a specific combination of augmentations and mixup), it is not true anymore when we compare all models \textit{jointly}. Moreover, the correlation, even within subgroups of models defined by augmentations, is much lower for vision transformers. 
Thus, we believe it is important to rethink the intuitive understanding of sharpness based on the geometric intuition about the shift of the loss surface. %
Moreover, our findings suggest that one should avoid blanket statements like \textit{``flatter minima generalize better''} since even when they are only intended to imply \textit{correlation}, their correctness still depends on a number of factors such as data distribution, model family, or initialization schemes (i.e., random vs. from pretrained weights).

\section*{Acknowledgements}
M.A. was supported by the Google Fellowship and Open Phil AI Fellowship. M.M. and M.H. were supported by the Carl Zeiss Foundation in the project "Certification and Foundations of Safe Machine Learning Systems in Healthcare". 
We thank David Stutz for very fruitful discussions at the initial stage of the project, Jana Vuckovic for experiments on sharpness that helped us to shape the project and Aditya Varre for discussions on sharpness for diagonal networks.

\bibliography{literature}
\bibliographystyle{icml2023}

\newpage
\appendix
\onecolumn

\begin{center}
    \ \\
    \Large\textbf{Appendix}
\end{center}

The appendix is organized as follows:
\begin{itemize}
    \item Sec.~\ref{sec:app_omitted_proofs}: omitted derivations for sharpness when $\rho \to 0$, first for the general case and then specifically for diagonal linear networks.
    \item Sec.~\ref{sec:app_gen_gap}: figures with correlation between sharpness and \textit{generalization gap}. We observe a similar trend between sharpness and \textit{generalization gap} as between  sharpness and \textit{test error} which is reported in the main part. 
    \item Sec.~\ref{sec:app_IN-1k-extrafigures}: additional figures about ViTs from \citet{steiner2021train} trained with different hyperparameter settings on ImageNet-1k. We observe that different sharpness variants are not predictive of the performance on ImageNet and the OOD datasets, typically only separating models by stochastic depth / dropout, but not ranking them according to generalization, and often even yielding a negative correlation with OOD test error.
    \item Sec.~\ref{sec:app_IN21k-IN1k-extrafigures}: figures about ViTs from \citet{steiner2021train} pre-trained on ImageNet-21k and then fine-tuned on ImageNet-1k. The observations are very similar to those for training on ImageNet-1k from scratch: sharpness variants are not predictive of the performance on ImageNet, and they often lead to a negative correlation with OOD test error.
    \item Sec.~\ref{sec:app_bothIN21k-IN1k-extrafigures}: figures for combined analysis of ViTs from \citet{steiner2021train} both with and without ImageNet-21k pre-training. We find the better-generalizing models pretrained on ImageNet-21k to have significantly higher worst-case sharpness and roughly equal or higher logit-normalized average-case adaptive sharpness, underlining that the models' generalization properties resulting from different pretraining datasets are not captured.
    \item Sec.~\ref{sec:app_finetuning_clip_on_imagenet}: additional details and figures for CLIP models fine-tuned on ImageNet. We observe that sharpness variants are not predictive of the performance on ImageNet and ImageNet-V2. Moreover, there is in most cases a negative correlation with test error in presence of distribution shifts which is likely to be related to the influence that the learning rate has on sharpness. 
    \item Sec.~\ref{sec:app_mnli}: additional details and figures for BERT models fine-tuned on MNLI. We find that all sharpness variants we consider are not predictive of the generalization performance of the model, and in some cases there is rather a weak negative correlation between sharpness and test error on out-of-distribution tasks from HANS.
    \item Sec.~\ref{sec:app_cifar10_extra_figures}: additional details and ablation studies for CIFAR-10 models. We analyze the role of data used to evaluate sharpness, the role of the number of iterations in Auto-PGD, the role of $m$ in $m$-sharpness, and the influence of different sharpness definitions and radii on correlation with generalization. Overall, we conclude that none of the considered sharpness definitions or radii correlates positively with generalization nor that low sharpness implies good performance of the model.
\end{itemize}

Also, for the sake of convenience, we provide in Table~\ref{tab:summary_imagenet}, Table~\ref{tab:summary_mnli}, Table~\ref{tab:summary_cifar10_resnet18}, and Table~\ref{tab:summary_cifar10_vit} a summary of correlation coefficients $\tau$ between sharpness and generalization for all our experiments (except ablation studies).
\begin{table}[h!]
    \centering

    \begin{tabular}{lll|rrrrrr}
        \multicolumn{9}{c}{\textbf{ImageNet-1k models trained from scratch}}\\
        & & & \multicolumn{6}{|c}{\textbf{Rank correlation coefficient} $\tau$}\\
        \textbf{Sharpness} & \textbf{LogitNorm} & $\rho$ & \textbf{IN} & \textbf{IN-v2} & \textbf{IN-R} & \textbf{IN-Sketch} & \textbf{IN-A} & \textbf{ObjectNet}\\
        \toprule
        Worst-case $\ell_\infty$ & Yes & 0.001 &  0.09 &  0.08 &  0.10 &  0.10 & -0.06 &  0.04\\
        Worst-case $\ell_\infty$ & Yes & 0.002 &  0.08 &  0.08 &  0.09 &  0.09 & -0.07 &  0.03\\
        Worst-case $\ell_\infty$ & Yes & 0.004 & -0.11 & -0.11 & -0.06 & -0.06 & -0.23 & -0.16\\
                                                      
        Worst-case $\ell_\infty$ & No & 0.001 &  -0.42 & -0.43 & -0.27 & -0.28 & -0.45 & -0.45\\
        Worst-case $\ell_\infty$ & No & 0.002 &  -0.42 & -0.42 & -0.27 & -0.27 & -0.41 & -0.45\\
        Worst-case $\ell_\infty$ & No & 0.004 &  -0.34 & -0.34 & -0.20 & -0.20 & -0.36 & -0.36\\
                                                      
        Avg-case $\ell_\infty$ & Yes & 0.05 &  0.46 &  0.44 &  0.38 &  0.42 & 0.31 &  0.39\\
        Avg-case $\ell_\infty$ & Yes & 0.1  &  0.44 &  0.43 &  0.39 &  0.43 & 0.29 &  0.39\\
        Avg-case $\ell_\infty$ & Yes & 0.2  &  0.42 &  0.42 &  0.39 &  0.42 & 0.29 &  0.38\\
                                                      
        Avg-case $\ell_\infty$ & No & 0.05 &  \textbf{-0.55} & \textbf{-0.56} & -0.40 & -0.42 & \textbf{-0.57} & \textbf{-0.60}\\
        Avg-case $\ell_\infty$ & No & 0.1  &  -0.44 & -0.43 & -0.28 & -0.32 & -0.47 & -0.47\\
        Avg-case $\ell_\infty$ & No & 0.2  &   0.13 &  0.15 &  0.26 &  0.23 &  0.05 &  0.11\\
        
        \\
        
        \multicolumn{9}{c}{\textbf{ImageNet-1k models fine-tuned from IN-21k}}\\
        & & & \multicolumn{6}{|c}{\textbf{Rank correlation coefficient} $\tau$}\\
        \textbf{Sharpness} & \textbf{LogitNorm} & $\rho$ & \textbf{IN} & \textbf{IN-v2} & \textbf{IN-R} & \textbf{IN-Sketch} & \textbf{IN-A} & \textbf{ObjectNet}\\
        \toprule
        Worst-case $\ell_\infty$ & Yes & 0.001 & -0.49 & -0.49 & -0.44 & -0.33 & \textbf{-0.53} & -0.46\\
        Worst-case $\ell_\infty$ & Yes & 0.002 & -0.48 & -0.48 & -0.46 & -0.33 & \textbf{-0.51} & -0.44\\
        Worst-case $\ell_\infty$ & Yes & 0.004 & -0.45 & -0.43 & -0.41 & -0.33 & -0.45 & -0.42\\
                                                      
        Worst-case $\ell_\infty$ & No & 0.001 &  -0.13 & -0.09 & -0.05 & 0.05 & -0.13 & -0.09\\
        Worst-case $\ell_\infty$ & No & 0.002 &  -0.10 & -0.03 & -0.01 & 0.11 & -0.07 & -0.02\\
        Worst-case $\ell_\infty$ & No & 0.004 &  -0.10 & -0.01 & -0.01 & 0.11 & -0.06 &  0.00\\
                                                      
        Avg-case $\ell_\infty$ & Yes & 0.05 &  -0.11 &  -0.08 &  -0.11 &  -0.07 & -0.06 & -0.06\\
        Avg-case $\ell_\infty$ & Yes & 0.1  &  -0.12 &  -0.11 &  -0.14 &  -0.10 & -0.09 & -0.08\\
        Avg-case $\ell_\infty$ & Yes & 0.2  &  -0.25 &  -0.24 &  -0.25 &  -0.23 & -0.25 & -0.24\\
                                                      
        Avg-case $\ell_\infty$ & No & 0.05 &  -0.02 & -0.04 & -0.03 & -0.02 & -0.05 & -0.06\\
        Avg-case $\ell_\infty$ & No & 0.1  &  -0.07 & -0.10 & -0.08 & -0.08 & -0.11 & -0.10\\
        Avg-case $\ell_\infty$ & No & 0.2  &  -0.11 & -0.11 & -0.10 & -0.11 & -0.12 & -0.13\\
        
        \\
        
        \multicolumn{9}{c}{\textbf{ImageNet-1k models fine-tuned from CLIP}}\\
        & & & \multicolumn{6}{|c}{\textbf{Rank correlation coefficient} $\tau$}\\
        \textbf{Sharpness} & \textbf{LogitNorm} & $\rho$ & \textbf{IN} & \textbf{IN-v2} & \textbf{IN-R} & \textbf{IN-Sketch} & \textbf{IN-A} & \textbf{ObjectNet}\\
        \toprule
        Worst-case $\ell_\infty$ & Yes & 0.001 &  -0.04 & -0.16 & -0.23 & -0.26 & -0.25 & -0.36\\
        Worst-case $\ell_\infty$ & Yes & 0.002 &   0.04 & -0.10 & -0.39 & -0.28 & -0.41 & -0.47\\
        Worst-case $\ell_\infty$ & Yes & 0.004 &  -0.08 & -0.19 & -0.12 & -0.16 & -0.17 & -0.27\\
                                                       
        Worst-case $\ell_\infty$ & No & 0.001 &   0.19 &  0.09 & -0.37 & -0.06 & \textbf{-0.57} & -0.48\\
        Worst-case $\ell_\infty$ & No & 0.002 &   0.20 &  0.08 & \textbf{-0.51} & -0.18 & \textbf{-0.58} & \textbf{-0.51}\\
        Worst-case $\ell_\infty$ & No & 0.004 &   0.02 & -0.05 & \textbf{-0.51} & -0.27 & -0.45 & -0.33\\
        
        Avg-case $\ell_\infty$ & Yes & 0.001 &  -0.03 & -0.18 & -0.36 & -0.34 & -0.33 & -0.46\\
        Avg-case $\ell_\infty$ & Yes & 0.002 &  -0.21 & -0.32 & -0.02 & -0.27 & -0.06 & -0.21\\
        Avg-case $\ell_\infty$ & Yes & 0.004 &  -0.19 & -0.21 &  0.26 & -0.03 &  0.23 &  0.06\\
                                                       
        Avg-case $\ell_\infty$ & No & 0.001 &   0.13 & -0.01 & \textbf{-0.62} & -0.26 & \textbf{-0.67} & \textbf{-0.60}\\
        Avg-case $\ell_\infty$ & No & 0.002 &   0.06 &  0.03 & -0.34 & -0.12 & \textbf{-0.50} & -0.37\\
        Avg-case $\ell_\infty$ & No & 0.004 &   0.19 &  0.21 & -0.12 &  0.09 & -0.21 & -0.08\\
    \end{tabular}
    \caption{A summary of correlation between sharpness and generalization for all experiments on \textbf{ImageNet}. We boldface entries with $|\tau| > 0.5$ suggesting a reasonably strong correlation. LogitNorm stands for \textit{logit normalization} and IN stands for \textit{ImageNet}.}
    \label{tab:summary_imagenet}
\end{table}

\begin{table}
    \centering
    \begin{tabular}{lll|rrrr}
        \multicolumn{7}{c}{\textbf{MNLI models fine-tuned from BERT}}\\
        & & & \multicolumn{4}{|c}{\textbf{Rank correlation coefficient} $\tau$}\\
        \textbf{Sharpness} & \textbf{LogitNorm} & $\rho$ & \textbf{MNLI} & \textbf{HANS-L} & \textbf{HANS-S} & \textbf{HANS-C}\\
        \toprule
           Worst-case $\ell_\infty$ & Yes & 0.0005 &  0.04 & -0.09 & -0.14 & -0.21\\
           Worst-case $\ell_\infty$ & Yes & 0.001  & -0.09 & -0.09 & -0.13 & -0.18\\
           Worst-case $\ell_\infty$ & Yes & 0.002  &  0.05 & -0.09 & -0.14 & -0.17\\
           
           Worst-case $\ell_\infty$ & No  & 0.0005 & 0.04  & -0.24 & -0.22 & -0.07\\
           Worst-case $\ell_\infty$ & No  & 0.001  & 0.04  & -0.13 & -0.15 & -0.15\\
           Worst-case $\ell_\infty$ & No  & 0.002  & -0.11 & -0.15 & -0.12 & -0.13\\
           
           Avg-case $\ell_\infty$ & Yes & 0.1 & -0.35 & -0.46 & -0.28 & 0.17\\
           Avg-case $\ell_\infty$ & Yes & 0.2 & -0.37 & -0.48 & -0.28 & 0.24\\
           Avg-case $\ell_\infty$ & Yes & 0.4 &  0.01 & -0.29 & -0.27 & 0.05\\
           
           Avg-case $\ell_\infty$ & No  & 0.1 & -0.34 & -0.31 & -0.23 & 0.13\\
           Avg-case $\ell_\infty$ & No  & 0.2 & -0.34 & \textbf{-0.58} & -0.39 & 0.16\\
           Avg-case $\ell_\infty$ & No  & 0.4 &  0.04 & -0.16 & -0.09 & 0.05\\
    \end{tabular}
    \caption{A summary of correlation between sharpness and generalization for all experiments on \textbf{MNLI} for models fine-tuned from BERT. We boldface entries with $|\tau| > 0.5$ suggesting a reasonably strong correlation. LogitNorm stands for \textit{logit normalization}.}
    \label{tab:summary_mnli}
\end{table}

\begin{table}
    \centering
    \begin{tabular}{lll|rr}
        \multicolumn{5}{c}{\textbf{ResNets-18 trained from scratch on CIFAR-10}}\\
        & & & \multicolumn{2}{|c}{\textbf{Rank correlation coefficient} $\tau$}\\
        \textbf{Sharpness} & \textbf{LogitNorm} & $\rho$ & \textbf{CIFAR-10} & \textbf{CIFAR-10-C}\\
        \toprule
        Standard avg-case $\ell_2$ & No & 0.05 & 0.14 & 0.04\\
        Standard avg-case $\ell_2$ & No & 0.1  & 0.26 & 0.19\\
        Standard avg-case $\ell_2$ & No & 0.2  & 0.28 & 0.21\\
        Standard avg-case $\ell_2$ & No & 0.4  & 0.28 & 0.20\\
                                                         
        Standard worst-case $\ell_2$ & No & 0.25 & 0.17 & 0.10\\
        Standard worst-case $\ell_2$ & No & 0.5  & 0.24 & 0.16\\
        Standard worst-case $\ell_2$ & No & 1.0  & 0.25 & 0.18\\
        Standard worst-case $\ell_2$ & No & 2.0  & 0.22 & 0.14\\

        Adaptive avg-case $\ell_2$ & No & 0.05 & -0.37 & -0.46\\
        Adaptive avg-case $\ell_2$ & No & 0.1  & \textbf{-0.50} & \textbf{-0.53}\\
        Adaptive avg-case $\ell_2$ & No & 0.2  & -0.42 & -0.41\\
        Adaptive avg-case $\ell_2$ & No & 0.4  & -0.31 & -0.31\\
                                                         
        Adaptive worst-case $\ell_2$ & No & 0.25 & -0.36 & -0.39\\
        Adaptive worst-case $\ell_2$ & No & 0.5  & -0.42 & -0.36\\
        Adaptive worst-case $\ell_2$ & No & 1.0  & -0.27 & -0.17\\
        Adaptive worst-case $\ell_2$ & No & 2.0  & -0.17 & -0.07\\
        
        Adaptive avg-case $\ell_2$ & Yes & 0.05 &  0.18 & 0.07\\
        Adaptive avg-case $\ell_2$ & Yes & 0.1  &  0.07 & -0.04\\
        Adaptive avg-case $\ell_2$ & Yes & 0.2  & -0.14 & -0.26\\
        Adaptive avg-case $\ell_2$ & Yes & 0.4  & -0.43 & -0.58\\
                                                         
        Adaptive worst-case $\ell_2$ & Yes & 0.25 & 0.19 & 0.14\\
        Adaptive worst-case $\ell_2$ & Yes & 0.5  & 0.07 & 0.00\\
        Adaptive worst-case $\ell_2$ & Yes & 1.0  & -0.13 & -0.22\\
        Adaptive worst-case $\ell_2$ & Yes & 2.0  & \textbf{-0.52} & \textbf{-0.58}\\

        Standard avg-case $\ell_\infty$ & No & 0.1 & 0.16 & 0.08\\
        Standard avg-case $\ell_\infty$ & No & 0.2  & 0.28 & 0.21\\
        Standard avg-case $\ell_\infty$ & No & 0.4  & 0.28 & 0.20\\
        Standard avg-case $\ell_\infty$ & No & 0.8  & 0.28 & 0.20\\
                                                         
        Standard worst-case $\ell_\infty$ & No & 0.0005 & 0.29 & 0.23\\
        Standard worst-case $\ell_\infty$ & No & 0.001  & 0.30 & 0.24\\
        Standard worst-case $\ell_\infty$ & No & 0.002  & 0.30 & 0.24\\
        Standard worst-case $\ell_\infty$ & No & 0.004  & 0.29 & 0.23\\

        Adaptive avg-case $\ell_\infty$ & No & 0.1 & -0.36 & -0.47\\
        Adaptive avg-case $\ell_\infty$ & No & 0.2 & \textbf{-0.53} & \textbf{-0.56}\\
        Adaptive avg-case $\ell_\infty$ & No & 0.4 & -0.41 & -0.41\\
        Adaptive avg-case $\ell_\infty$ & No & 0.8 & -0.20 & -0.18\\
                                                         
        Adaptive worst-case $\ell_\infty$ & No & 0.001 & -0.36 & -0.42\\
        Adaptive worst-case $\ell_\infty$ & No & 0.002 & -0.05 & -0.10\\
        Adaptive worst-case $\ell_\infty$ & No & 0.004 &  0.25 &  0.20\\
        Adaptive worst-case $\ell_\infty$ & No & 0.008 &  0.26 &  0.24\\
        
        Adaptive avg-case $\ell_\infty$ & Yes & 0.1 &  0.18 & 0.07\\
        Adaptive avg-case $\ell_\infty$ & Yes & 0.2 &  0.05 & -0.06\\
        Adaptive avg-case $\ell_\infty$ & Yes & 0.4 & -0.23 & -0.37\\
        Adaptive avg-case $\ell_\infty$ & Yes & 0.8 & -0.46 & \textbf{-0.62}\\
                                                         
        Adaptive worst-case $\ell_\infty$ & Yes & 0.001 & 0.30 & 0.18\\
        Adaptive worst-case $\ell_\infty$ & Yes & 0.002 & 0.29 & 0.16\\
        Adaptive worst-case $\ell_\infty$ & Yes & 0.004 & 0.21 & 0.07\\
        Adaptive worst-case $\ell_\infty$ & Yes & 0.008 & -0.04 & -0.19\\
    \end{tabular}
    
    \caption{A summary of correlation between sharpness and generalization for all experiments on \textbf{CIFAR-10} for ResNets-18 trained from scratch. We boldface entries with $|\tau| > 0.5$ suggesting a reasonably strong correlation. LogitNorm stands for \textit{logit normalization}.}
    \label{tab:summary_cifar10_resnet18}
\end{table}

\begin{table}
    \centering
    \begin{tabular}{lll|rr}
        \multicolumn{5}{c}{\textbf{Vision transformers trained from scratch on CIFAR-10}}\\
        & & & \multicolumn{2}{|c}{\textbf{Rank correlation coefficient} $\tau$}\\
        \textbf{Sharpness} & \textbf{LogitNorm} & $\rho$ & \textbf{CIFAR-10} & \textbf{CIFAR-10-C}\\
        \toprule
        Standard avg-case $\ell_2$ & No & 0.005 & -0.45 & \textbf{-0.54}\\
        Standard avg-case $\ell_2$ & No & 0.01  & -0.39 & -0.49\\
        Standard avg-case $\ell_2$ & No & 0.02  & -0.20 & -0.31\\
        Standard avg-case $\ell_2$ & No & 0.04  & -0.08 & -0.20\\
                                                         
        Standard worst-case $\ell_2$ & No & 0.025 & \textbf{-0.59} & \textbf{-0.62}\\
        Standard worst-case $\ell_2$ & No & 0.05  & -0.37 & -0.43\\
        Standard worst-case $\ell_2$ & No & 0.1   & -0.16 & -0.24\\
        Standard worst-case $\ell_2$ & No & 0.2   & -0.12 & -0.20\\

        Adaptive avg-case $\ell_2$ & No & 0.1  & -0.45 & \textbf{-0.50}\\
        Adaptive avg-case $\ell_2$ & No & 0.2  & -0.45 & -0.45\\
        Adaptive avg-case $\ell_2$ & No & 0.4  & -0.42 & -0.47\\
        Adaptive avg-case $\ell_2$ & No & 0.8  & -0.10 &  0.08\\
                                                         
        Adaptive worst-case $\ell_2$ & No & 0.5 & \textbf{-0.64} & \textbf{-0.53}\\
        Adaptive worst-case $\ell_2$ & No & 1.0 & -0.32 & -0.19\\
        Adaptive worst-case $\ell_2$ & No & 2.0 & -0.11 & -0.01\\
        Adaptive worst-case $\ell_2$ & No & 4.0 & -0.07 & -0.03\\
        
        Adaptive avg-case $\ell_2$ & Yes & 0.1 & -0.18 & -0.31\\
        Adaptive avg-case $\ell_2$ & Yes & 0.2 & -0.28 & -0.40\\
        Adaptive avg-case $\ell_2$ & Yes & 0.4 & -0.39 & -0.46\\
        Adaptive avg-case $\ell_2$ & Yes & 0.8 & -0.44 & \textbf{-0.52}\\
                                                         
        Adaptive worst-case $\ell_2$ & Yes & 0.25 & -0.21 & -0.12\\
        Adaptive worst-case $\ell_2$ & Yes & 0.5  & -0.24 & -0.17\\
        Adaptive worst-case $\ell_2$ & Yes & 1.0  & -0.22 & -0.19\\
        Adaptive worst-case $\ell_2$ & Yes & 2.0  & -0.14 & -0.11\\

        Standard avg-case $\ell_\infty$ & No & 0.01 & -0.44 & \textbf{-0.54}\\
        Standard avg-case $\ell_\infty$ & No & 0.02 & -0.35 & -0.45\\
        Standard avg-case $\ell_\infty$ & No & 0.04 & -0.17 & -0.28\\
        Standard avg-case $\ell_\infty$ & No & 0.08 & -0.04 & -0.14\\
                                                         
        Standard worst-case $\ell_\infty$ & No & 0.00001 & \textbf{-0.61} & \textbf{-0.63}\\
        Standard worst-case $\ell_\infty$ & No & 0.00002 & -0.46 & \textbf{-0.51}\\
        Standard worst-case $\ell_\infty$ & No & 0.00004 & -0.25 & -0.31\\
        Standard worst-case $\ell_\infty$ & No & 0.00008 & -0.16 & -0.22\\

        Adaptive avg-case $\ell_\infty$ & No & 0.1 & -0.45 & \textbf{-0.53}\\
        Adaptive avg-case $\ell_\infty$ & No & 0.2 & -0.46 & \textbf{-0.50}\\
        Adaptive avg-case $\ell_\infty$ & No & 0.4 & -0.45 & -0.44\\
        Adaptive avg-case $\ell_\infty$ & No & 0.8 & -0.41 & -0.47\\
                                                         
        Adaptive worst-case $\ell_\infty$ & No & 0.0005 & \textbf{-0.68} & \textbf{-0.63}\\
        Adaptive worst-case $\ell_\infty$ & No & 0.001  & -0.43 & -0.40\\
        Adaptive worst-case $\ell_\infty$ & No & 0.002  & -0.26 & -0.23\\
        Adaptive worst-case $\ell_\infty$ & No & 0.004  & -0.18 & -0.18\\
        
        Adaptive avg-case $\ell_\infty$ & Yes & 0.1 & -0.11 & -0.23\\
        Adaptive avg-case $\ell_\infty$ & Yes & 0.2 & -0.16 & -0.29\\
        Adaptive avg-case $\ell_\infty$ & Yes & 0.4 & -0.31 & -0.42\\
        Adaptive avg-case $\ell_\infty$ & Yes & 0.8 & -0.40 & -0.47\\
                                                         
        Adaptive worst-case $\ell_\infty$ & Yes & 0.0005 & -0.20 & -0.23\\
        Adaptive worst-case $\ell_\infty$ & Yes & 0.001  & -0.22 & -0.26\\
        Adaptive worst-case $\ell_\infty$ & Yes & 0.002  & -0.29 & -0.34\\
        Adaptive worst-case $\ell_\infty$ & Yes & 0.004  & -0.39 & -0.44\\
    \end{tabular}
    
    \caption{A summary of correlation between sharpness and generalization for all experiments on \textbf{CIFAR-10} for ViTs trained from scratch. We boldface entries with $|\tau| > 0.5$ suggesting a reasonably strong correlation. LogitNorm stands for \textit{logit normalization}.}
    \label{tab:summary_cifar10_vit}
\end{table}

\clearpage

\section{Omitted Proofs}
\label{sec:app_omitted_proofs}

\subsection{Asymptotic Analysis of Adaptive Sharpness Measures}
\label{sec:app_asymptotic}
For the convenience of the reader we repeat here quickly the definitions of adaptive sharpness measures.
Let $L_\Scal(\wv) = \frac{1}{|S|} \sum_{(\xv, \yv) \in \Scal} \l_{\xv\yv}(\wv)$ be the loss on a set of \textit{training} points $\Scal$.
For arbitrary weights $\wv$ (i.e., not necessarily a minimum), then the \textit{average-case} and \textit{worst-case $m$-sharpness} is defined as:
\begin{align} %
S_{avg,p}^\rho(\wv, \cv) \triangleq \E_{\substack{\Scal \sim P_m \ \ \ \ \ \ \\ \deltav \sim \mathcal{N}(0, \rho^2 diag(\cv^2))}} \hspace{-8mm} L_\Scal(\wv + \deltav) - L_\Scal(\wv) \qquad
S_{max,p}^\rho(\wv, \cv) \triangleq \E_{\Scal \sim P_m} \maxop_{\norm{\deltav \odot \cv^{-1}}_p \leq \rho} L_\Scal(\wv + \deltav) - L_\Scal(\wv), \nonumber
\end{align}
where $\odot$/$^{-1}$ denotes elementwise multiplication/inversion and $P_m$ is the data distribution that returns $m$ training pairs $(\xv, \yv)$.

If $\cv=|\wv|$ then the perturbation set is $\norm{\delta \odot |\wv|^{-1}}_p \leq \rho$. We first introduce a new variable $\gammav=\deltav \odot |\wv|^{-1}$ and do a Taylor expansion around $w$:
\begin{align*}
L_\Scal(\wv+\deltav) = L_\Scal(\wv+\gammav \odot |\wv|)
= L_\Scal(\wv) + \inner{\nabla L_\Scal(\wv), |\wv| \odot \gammav} + \frac{1}{2}\inner{\gammav \odot |\wv|, \nabla^2 L_\Scal(\wv) \gammav \odot |\wv|} + O(\norm{\gammav}_p^3),
\end{align*}
where $\nabla^2 L_\Scal(\wv)$ denotes the Hessian of $L_\Scal$ at $\wv$.
\begin{proposition}\label{pro:first}
    Let $L_\Scal\in C^3(\R^s)$, $S$ be a finite sample of training points $(x_i,y_i)_{i=1}^n$ and let $P_m$ denote the uniform distribution over subsamples of size $m\leq n$ from $S$. Then we define for $p\geq 1$, $q \in \R$ such that $\frac{1}{p}+\frac{1}{q}=1$, then it holds
    \begin{align*}
    &\lim_{\rho \to 0} S_{max,p}^\rho(\wv, |\wv|) \\
    =& \E_{\Scal \sim P_m}
    \begin{cases} 
    \norm{\nabla L_\Scal(\wv) \odot |\wv|}_q \rho + O(\rho^2) 
    & \textrm{ if } \nabla L_\Scal(\wv) \odot |\wv| \neq 0,\\
    \ \\
    \frac{\rho^2}{2}\max\limits_{\gammav \neq 0} \frac{\inner{\gammav,\Big( \nabla^2 L_\Scal(\wv) \odot (|\wv| |\wv|^T)\Big) \gammav}}{\norm{\gammav}^2_p}+ O(\rho^3) 
    & \textrm{ if } \nabla L_\Scal(\wv) \odot |\wv| =0 \textrm{ and }\\
    & \quad \nabla^2 L_\Scal(\wv) \odot (|\wv| |\wv|^T) \textrm{ not negative definite}\\
    \ \\
    O(\rho^3) 
    & \textrm{ if } \nabla L_\Scal(\wv) \odot |\wv| =0 \textrm{ and }\\
    & \quad \nabla^2 L_\Scal(\wv) \odot (|\wv| |\wv|^T)  \textrm{ is negative definite } 
    \end{cases}                                                     
    \end{align*}
\end{proposition}
\begin{proof} We get 
    \begin{align*}
    \maxop_{\norm{\gammav}_p \leq \rho} L_\Scal(\wv+\gammav \odot |\wv|) - L_\Scal(\wv)&= 
    \max_{\norm{\gammav}_p\leq \rho}  \inner{\nabla L_\Scal(\wv), |\wv| \odot \gammav} + \frac{1}{2}\inner{\gammav \odot |\wv|,\nabla^2 L_\Scal(\wv) \gammav \odot |\wv|} + O(\norm{\gammav}_p^3)\\
    &=\max_{\norm{\gammav}_p\leq \rho}  \inner{\nabla L_\Scal(\wv) \odot |\wv|, \gammav} + \frac{1}{2}\inner{\gammav ,\Big(\nabla^2 L_\Scal(\wv) \odot (|\wv||\wv|^T)\Big)\gammav} + O(\norm{\gammav}_p^3)
    \end{align*}
    If $\nabla L_\Scal(\wv) \odot |\wv|\neq 0$, then the first order term dominates for $\rho$ sufficiently small and we get
    \[ \max_{\norm{\gammav}_p\leq \rho}  \inner{\nabla L_\Scal(\wv) \odot |\wv|, \gammav} 
    = \max_{\norm{\gammav}_p\leq \rho} \norm{\nabla L_\Scal(\wv) \odot |\wv|}_q \norm{\gammav}_p = \rho  \norm{\nabla L_\Scal(\wv) \odot |\wv|}_q.\]
    Otherwise we have to consider
    \[ \max_{\norm{\gammav}_p\leq \rho} \frac{1}{2}\inner{\gammav ,\Big(\nabla^2 L_\Scal(\wv) \odot (|\wv||\wv|^T)\Big)\gammav}.\]
    If $\nabla^2 L(\wv) \odot (|\wv||\wv|^T)$ is negative definite, then the maximum is zero attained at $\gammav=0$. In the other case, we get
    \[ \max_{\norm{\gammav}_p\leq \rho} \frac{1}{2}\inner{\gammav ,\Big(\nabla^2 L_\Scal(\wv) \odot (|\wv||\wv|^T)\Big)\gammav} = \frac{\rho^2}{2} \max\limits_{\gammav \neq 0} \frac{\inner{\gammav,\Big( \nabla^2 L_\Scal(\wv) \odot (|\wv| |\wv|^T)\Big) \gammav}}{\norm{\gammav}^2_p}.\]
This almost finishes the proof. Finally, it holds
\begin{align*}
 \lim_{\rho \to 0} S_{max,p}^\rho(\wv, |\wv|) &=  \lim_{\rho \to 0}\E_{\Scal \sim P_m}\left[ \maxop_{\norm{\gammav}_p \leq \rho} L_\Scal(\wv+\gammav \odot |\wv|) - L_\Scal(\wv)\right],\\
 &=\E_{\Scal \sim P_m}\left[\lim_{\rho \to 0} \maxop_{\norm{\gammav}_p \leq \rho} L_\Scal(\wv+\gammav \odot |\wv|) - L_\Scal(\wv)\right]
\end{align*}
where for the last step we have used that $\E_{\Scal \sim P_m}$ is the expectation over all possible subsamples of size $m$ and thus boils down to a finite sum for which we can drag the limit inside.
\end{proof}
We note that for $p=2$ it holds $q=2$ and
\[ \max\limits_{\gammav \neq 0} \frac{\inner{\gammav,\Big( \nabla^2 L_\Scal(\wv) \odot (|\wv| |\wv|^T)\Big) \gammav}}{\norm{\gammav}^2_2} = \lambda_{\max}\Big(\nabla^2 L_\Scal(\wv) \odot (|\wv| |\wv|^T)\Big),\]
which is the result used in the main paper.
\begin{proposition}
    Let $L_\Scal\in C^3(\R^s)$, $S$ be a finite sample of training points $(x_i,y_i)_{i=1}^n$ and let $P_m$ denote the uniform distribution over subsamples of size $m\leq n$ from $S$. %
    Then
    \begin{align*}
    \lim_{\rho \to 0} \frac{2}{\rho^2} S_{avg}^\rho(\wv, |\wv|) 
    &=\E_{\Scal \sim P_m} \left[ \tr(\nabla^2 L_\Scal(\wv) \odot |\wv| |\wv|^\top) \right] + O(\rho)
    \end{align*}
\end{proposition}
\begin{proof} Let us consider the loss without the subcript for clarity. Then we consider
    \[  \E_{\deltav \sim \mathcal{N}(0, \rho^2 diag(\cv^2))} L_\Scal(\wv + \deltav) - L_\Scal(\wv) \]
    When plugging in the Taylor expansion of the loss, we see that 
    \begin{align*}  \E_{\deltav \sim \mathcal{N}(0, \rho^2 diag(\cv^2))} &L_\Scal(\wv + \deltav) - L_\Scal(\wv)\\
    =& \E_{\gammav \in \mathcal{N}(0, \rho^2 \vI)} \Big[ \inner{\nabla L_\Scal(\wv), |\wv| \odot \gammav} + \frac{1}{2}\inner{\gammav \odot |\wv|,\nabla^2 L_\Scal(\wv) \gammav \odot |\wv|} + O(\norm{\gammav}_2^3)\Big]\\
    =&\frac{1}{2}\E_{\gammav \in \mathcal{N}(0, \rho^2 \vI)} 
    \Big[ \inner{\gammav \odot |\wv|,\nabla^2 L_\Scal(\wv) \gammav \odot |\wv|} \Big] + O(\rho^3)\\
    =&\frac{1}{2}\E_{\gammav \in \mathcal{N}(0, \rho^2 \vI)} \Big[\inner{\gammav, \big(\nabla^2 L_\Scal(\wv) \odot |\wv| |\wv|^T\big) \gammav}\Big] + O(\rho^3)\\
    =&\frac{\rho^2}{2}\tr(\nabla^2 L_\Scal(\wv) \odot |\wv| |\wv|^\top) + O(\rho^3) 
    \end{align*}
    where we use that the components of $\gammav$ are independent and have zero mean  and thus the first order term vanishes and for the second order term only the diagonal entries remain which are equal to the variance $\rho^2$. Finally, we take the expectation with respect to $P_m$. As in the proof of Proposition \ref{pro:first} we can drag the limit inside as the expectation with respect to $P_m$ corresponds to a finite sum.
\end{proof}

\subsection{Derivations for Diagonal Linear Networks}
\label{sec:app_proofs_dln}

\myparagraph{Hessian for diagonal linear networks.}
Denote $\rv = \vX (\uv \odot \vv) - \yv$, $\vV = \diag(\vv)$, $\vU = \diag(\uv)$, then the Hessian of the loss $\nabla^2 L(\wv)$ for diagonal linear networks is given by:
\begin{align}
L(\wv) =
\begin{bmatrix}
\vV \vX^\top \vX \vV & \vV \vX^\top \vX \vU + \diag(\vX^\top \rv) \\
\vV \vX^\top \vX \vU + \diag(\vX^\top \rv) & \vU \vX^\top \vX \vU 
\end{bmatrix}.
\end{align}
It is easy to verify that the data-dependent terms disappear due to the assumption of whitened data $\vX^\top \vX = \vI$ and zero residuals $\rv$ at a minimum. Thus, we arrive at a much simpler expression for the Hessian:
\begin{align} 
L(\wv) =
\begin{bmatrix}
\diag(\vv \odot \vv) & \diag(\vv \odot \uv)\\ 
\diag(\vv \odot \uv) & \diag(\uv \odot \uv)
\end{bmatrix},
\end{align}

\myparagraph{Maximum eigenvalue for diagonal linear networks.} 
Since the Hessian has a simple block structure, we can rearrange the rows and columns coherently and get a block-diagonal structure as follows
\begin{align}
\begin{bmatrix}
v_1^2   & v_1 u_1 & 0     & \dots & 0 \\ 
v_1 u_1 & u_1^2   & 0     & \dots & 0 \\
0       & 0       & \dots & \dots & 0 \\
\dots & \dots & \dots     & \dots & \dots \\
0 & \dots & 0 & v_d^2     & v_d u_d \\ 
0 & \dots & 0 & v_d u_d   & u_d^2
\end{bmatrix}
\end{align}
where eigenvalues of each $2\times2$ submatrix are $u_i^2 + v_i^2$ and $0$. Thus, $\lm = \maxop_{1 \leq i \leq d} v_i^2 + u_i^2$ by using the property of block-diagonal matrices.

\clearpage

\section{Correlation Between Sharpness and Generalization Gap}
\label{sec:app_gen_gap}
Throughout the paper we focused on correlation between sharpness and \textit{test error}, but it is natural to ask if the picture differs if we consider correlation between sharpness and 
\textit{generalization gap}, i.e., the difference between the test error and training error. We note that in the experiments on CIFAR-10 in Section~\ref{subsec:controlled_setup_cifar10}, since we consider only models with $\leq 1\%$ training error and since the test error is significantly larger than $1\%$, the behavior of generalization gap vs. sharpness has to be almost identical to that of test error vs. sharpness. For other datasets, however, the training error is not necessarily close to $0$, thus in Figure~\ref{fig:gen_gap_IN1k-scratch-main} and Figure~\ref{fig:gen_gap_clip_singlerho}, we additionally plot the \textit{generalization gap} vs. sharpness (and side-by-side the test error vs. sharpness for the sake of convenience) for the ImageNet experiments. We observe only small differences in the correlation values which do not alter the conclusions about the relationship of sharpness and generalization.
\begin{figure*}[h!] \centering \small
    \begin{tabular}{c} 
    \textbf{With logit normalization} \\
    \includegraphics[width=1.0\textwidth]{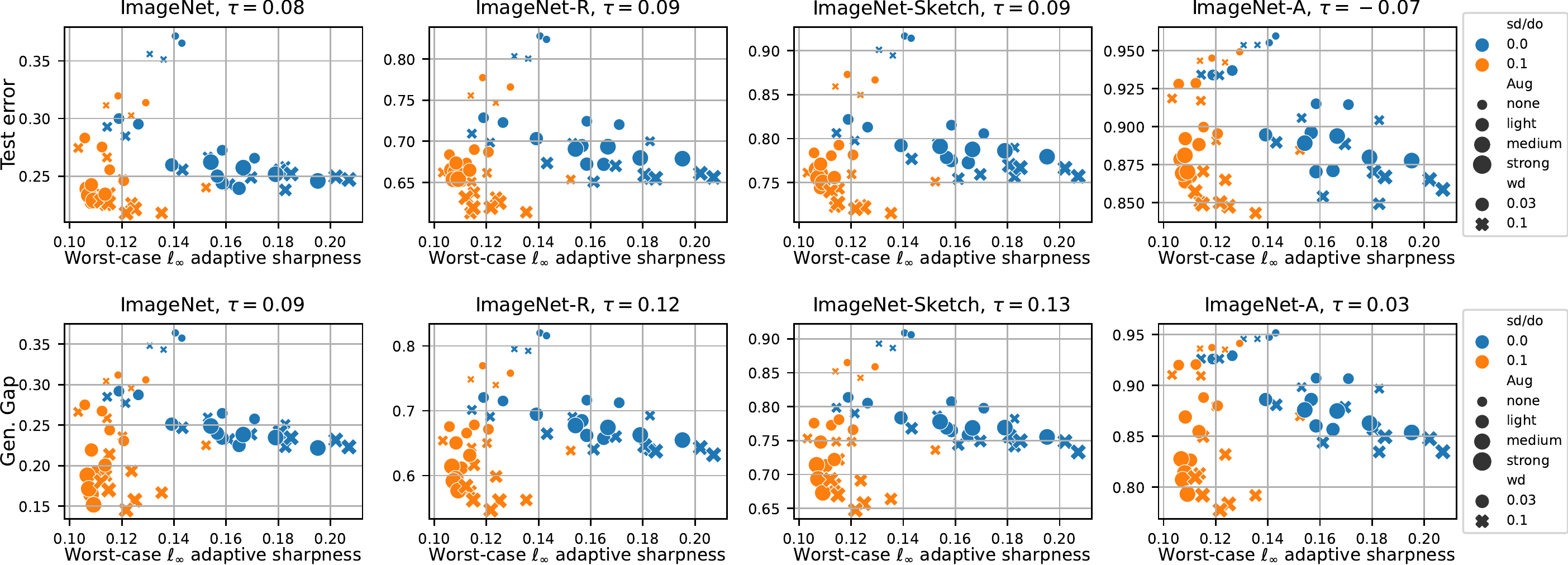} \vspace{1.5mm}\\
    \ \\
    \textbf{Without logit normalization} \\
    \includegraphics[width=1.0\textwidth]{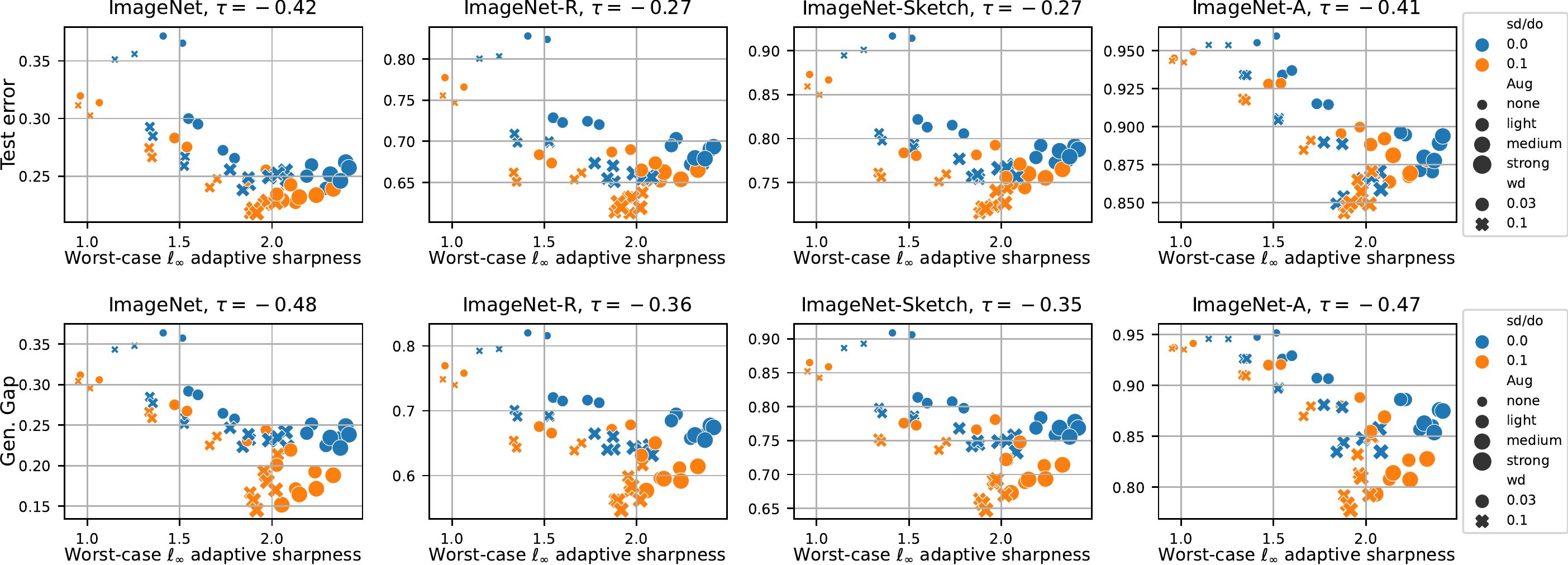}
    \end{tabular}
    \vspace{-1mm}
    \caption{\textbf{ViT-B/16 trained from scratch on ImageNet-1k}. 
    We show side-by-side the test error and \textbf{generalization gap} (\textbf{Gen. Gap}) for 56 models from Steiner et al. (2021) on ImageNet and its OOD variants vs. worst-case $\ell_\infty$ sharpness with (top) or without (bottom) normalization at $\rho=0.002$. The color indicates models trained with stochastic depth (sd) and dropout (do), markers and their size indicate the strength of weight decay (wd) and augmentations (aug), and $\tau$ indicates the rank correlation coefficient.}
    \label{fig:gen_gap_IN1k-scratch-main}
\end{figure*}
\begin{figure*}[h!] \centering \small
    \begin{tabular}{c} 
    \textbf{With logit normalization} \\
    \includegraphics[width=1.0\textwidth]{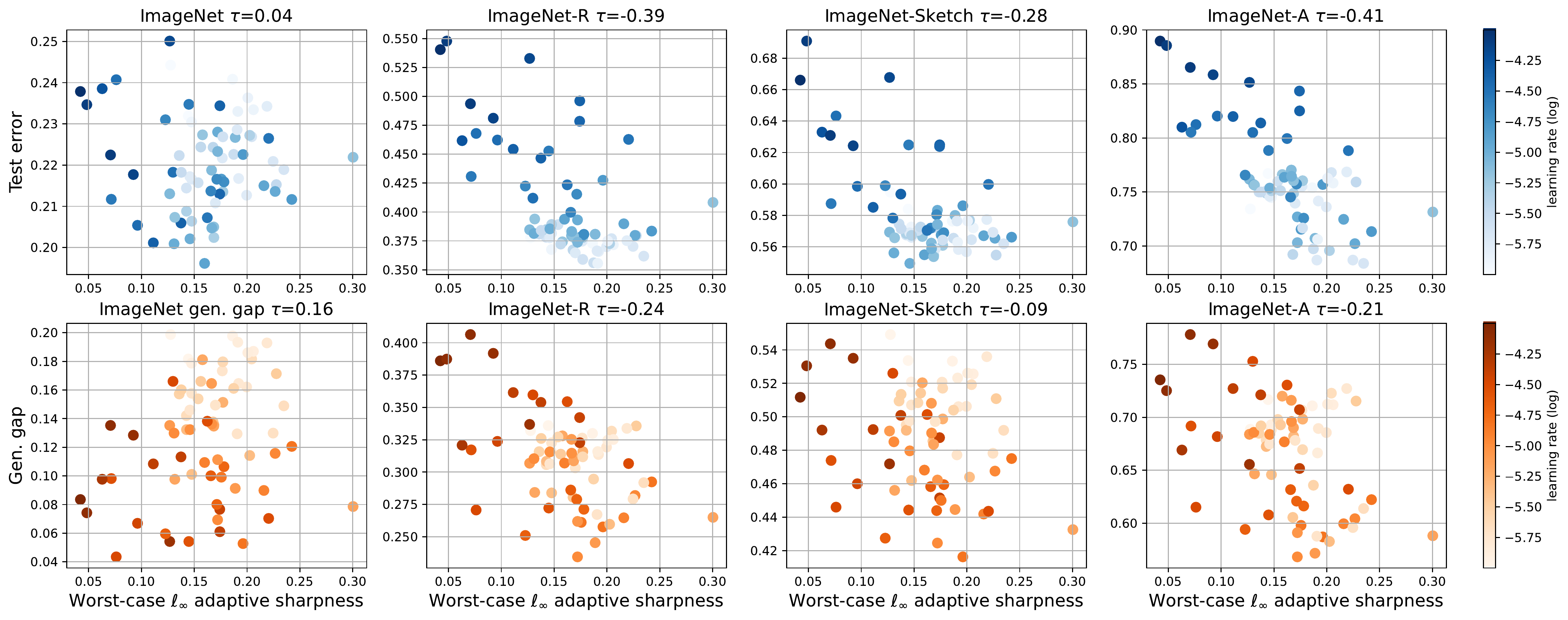} \vspace{1.5mm}\\
    \ \\
    \textbf{Without logit normalization} \\
    \includegraphics[width=1.0\textwidth]{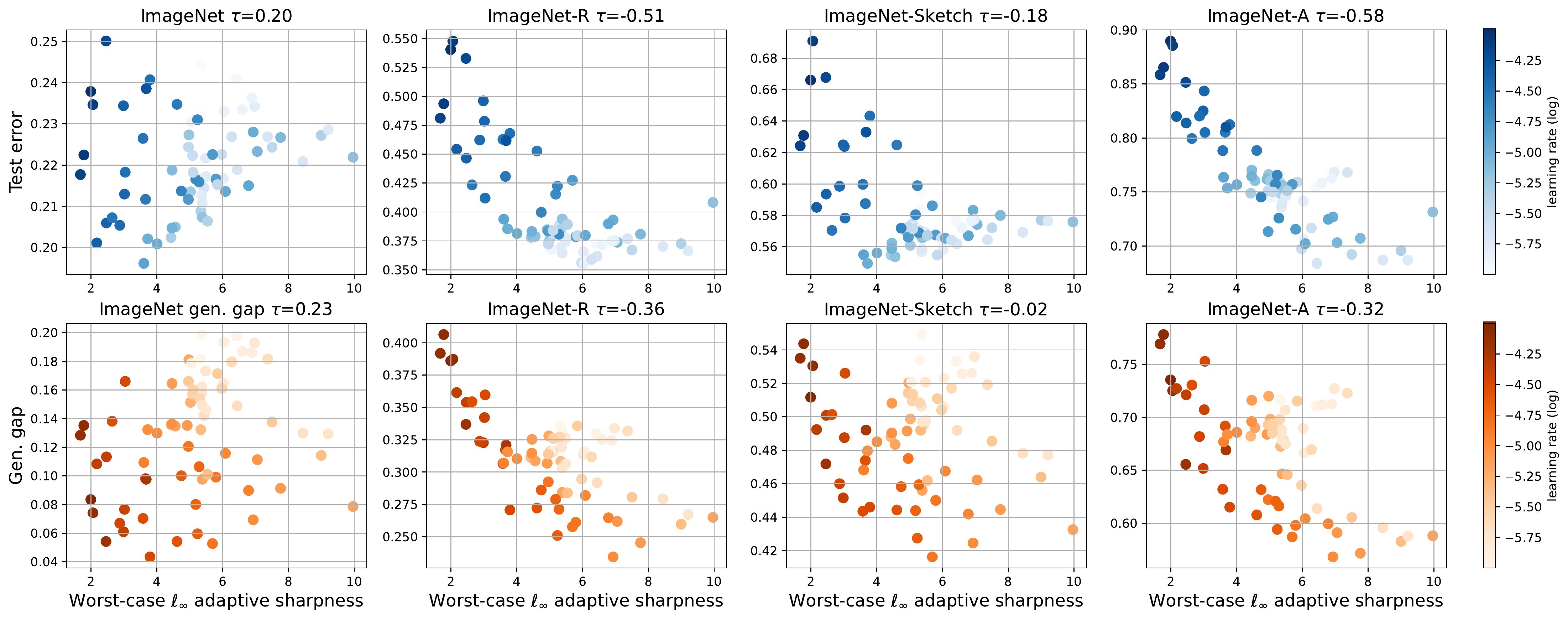}
    \end{tabular}
    \vspace{-1mm}
    \caption{\textbf{Fine-tuning CLIP ViT-B/32 on ImageNet-1k.} We show side-by-side the test error and \textbf{generalization gap} (\textbf{Gen. gap}) for 72 models from Wortsman et al. (2022) on ImageNet and its OOD variants vs. worst-case $\ell_\infty$ sharpness with (top) or without (bottom) normalization at $\rho=0.002$. Darker color indicates larger learning rate used for fine-tuning.}
    \label{fig:gen_gap_clip_singlerho}
\end{figure*}

\clearpage

\section{ImageNet-1k Models Trained from Scratch from \cite{steiner2021train}: Extra Details and Figures}
\label{sec:app_IN-1k-extrafigures}

\myparagraph{Experimental details.}
As explained in the main paper, the ViT-B/16-224 weights were trained on ImageNet-1k for 300 epochs with different hyperparameter settings, and subsequently fine-tuned on the same dataset for 20.000 steps with 2 different learning rates (0.01 and 0.03). The pretraining hyperparameters include 7 augmentation types (\textit{none, light0, light1, medium0, medium1, strong0, strong1}), which we group into (\textit{none, light, medium, strong}) in the plots. Weight decay was either 0.1 or 0.03, and dropout and stochastic depth were either both set to 0 or both set to 0.1. We evaluated the resulting 56 configurations. The model weights can be obtained from \url{https://github.com/google-research/vision_transformer}.

\myparagraph{Sharpness evaluation.} For sharpness evaluation we use 2048 data points from the training set split in 8 batches: we compute sharpness on each of them and report the average. For worst-case sharpness we use Auto-PGD for 20 steps (for each batch) with random uniform initialization in the feasible set, while for average-case sharpness we sample 100 different weights perturbations for every batch. We use the same sharpness evaluation for all ImageNet-1k and MNLI models.
For convenience we restate the algorithm of Auto-PGD in Algorithm~\ref{alg:apgd}: it follows the original version presented in \citet{croce2020reliable} while using the network weights $\wv$ as optimization variables instead of the input image components. In Alg.~\ref{alg:apgd} we denote $f$ the target objective function (cross-entropy loss on the batch of images in our experiments), $S$ the feasible set of perturbations and $P_S$ the projection onto it. Also, $\eta$ and $W$ are fixed hyperparameters (we keep the original values), and the two conditions in Line 13 can be found in \citet{croce2020reliable}.

\begin{algorithm}[h]
    \caption{Auto-PGD}\label{alg:apgd}
    \begin{algorithmic}[1]
        \STATE {\bfseries Input:} objective function $f$, perturbation set $S$, $\wv^{(0)}$, $\eta$, $N_\textrm{iter}$, $W=\{w_0, \ldots, w_n\}$ %
        \STATE {\bfseries Output:} $\wv_\textrm{max}$, $f_\textrm{max}$
        \STATE $\iter{\wv}{1} \gets P_\mathcal{S}\left(\wv^{(0)} + \eta \nabla f(\wv^{(0)})\right)$
        \STATE $f_\textrm{max}\gets \max\{f(\iter{\wv}{0}), f(\iter{\wv}{1})\}$
        \STATE $\wv_\textrm{max} \gets \iter{\wv}{0}$ \textbf{if} $f_\textrm{max}\equiv f(\iter{\wv}{0})$ \textbf{else} $\wv_\textrm{max} \gets \iter{\wv}{1}$ %
        \FOR{$k=1$ {\bfseries to}  $N_\textrm{iter}- 1$}
        \STATE $\iter{\zv}{k+1} \gets P_\mathcal{S}\left(\iter{\wv}{k} + \eta\nabla f(\iter{\wv}{k})\right)$
        \STATE $\iter{\wv}{k+1} \gets  P_\mathcal{S} \left(\iter{\wv}{k} + \alpha(\iter{\zv}{k+1}- \iter{\wv}{k})  + (1-\alpha) (\iter{\wv}{k} - \iter{\wv}{k-1})\right)$
        \IF{$f(\iter{\wv}{k+1}) > f_\textrm{max}$}
        \STATE $\wv_\textrm{max} \gets \iter{\wv}{k+1}$ and $f_\textrm{max}\gets f(\iter{\wv}{k+1})$
        \ENDIF
        \IF{$k\in W$}
        \IF{Condition 1 {\bfseries or} Condition 2}
        \STATE $\eta \gets \eta / 2$ and  $\iter{\wv}{k+1} \gets \wv_\textrm{max}$
        \ENDIF
        \ENDIF
        \ENDFOR
    \end{algorithmic}
\end{algorithm}

\myparagraph{Extra figures.} For each sharpness definition we show for three values of $\rho$ the correlation between test error on ImageNet (in-distribution) and on the various distribution shifts.  In particular, we use worst-case $\ell_\infty$ adaptive sharpness with  (Fig.~\ref{fig:i1k_worst-adaptive-logit}) and without (Fig.~\ref{fig:i1k_worst-adaptive-normal}) logit normalization, and average-case adaptive sharpness with  (Fig.~\ref{i1k_avg-adpative-logit}) and without (Fig.~\ref{i1k_avg-adpative-normal}) logit normalization. For all figures the color shows stochastic depth / dropout, the marker size corresponds to augmentation strength, and the marker type to  weight decay. In addition to the OOD-datasets from the main paper, we here report the results for ImageNet-V2 \citep{recht2019imagenet} and ObjectNet \citep{barbu2019objectnet}. ImageNet-V2 consists in a new test set for ImageNet models and is sampled from the same image distribution as the existing validation set: then, the performance of the classifiers on it are highly correlated to that on ImageNet validation set, and ImageNet-V2 cannot be considered a distribution shift in the same sense as the other datasets. In general, we observe that sharpness variants are not predictive of the performance on ImageNet and the OOD datasets, typically only separating models by stochastic depth / dropout, but not ranking them according to generalization properties, and often even yielding a negative correlation with OOD test error. The only case where low sharpness indicates low test-error is for logit-normalized average-case adaptive sharpness on ImageNet and ImageNet-v2. For the remaining OOD datasets, however, there are always models with low sharpness and larger test error.

\begin{figure}[p] \centering \small
    \begin{tabular}{c}
        \textbf{Worst-case $\ell_\infty$ adaptive sharpness with logit normalization}\\
        \includegraphics[width=.8\columnwidth]{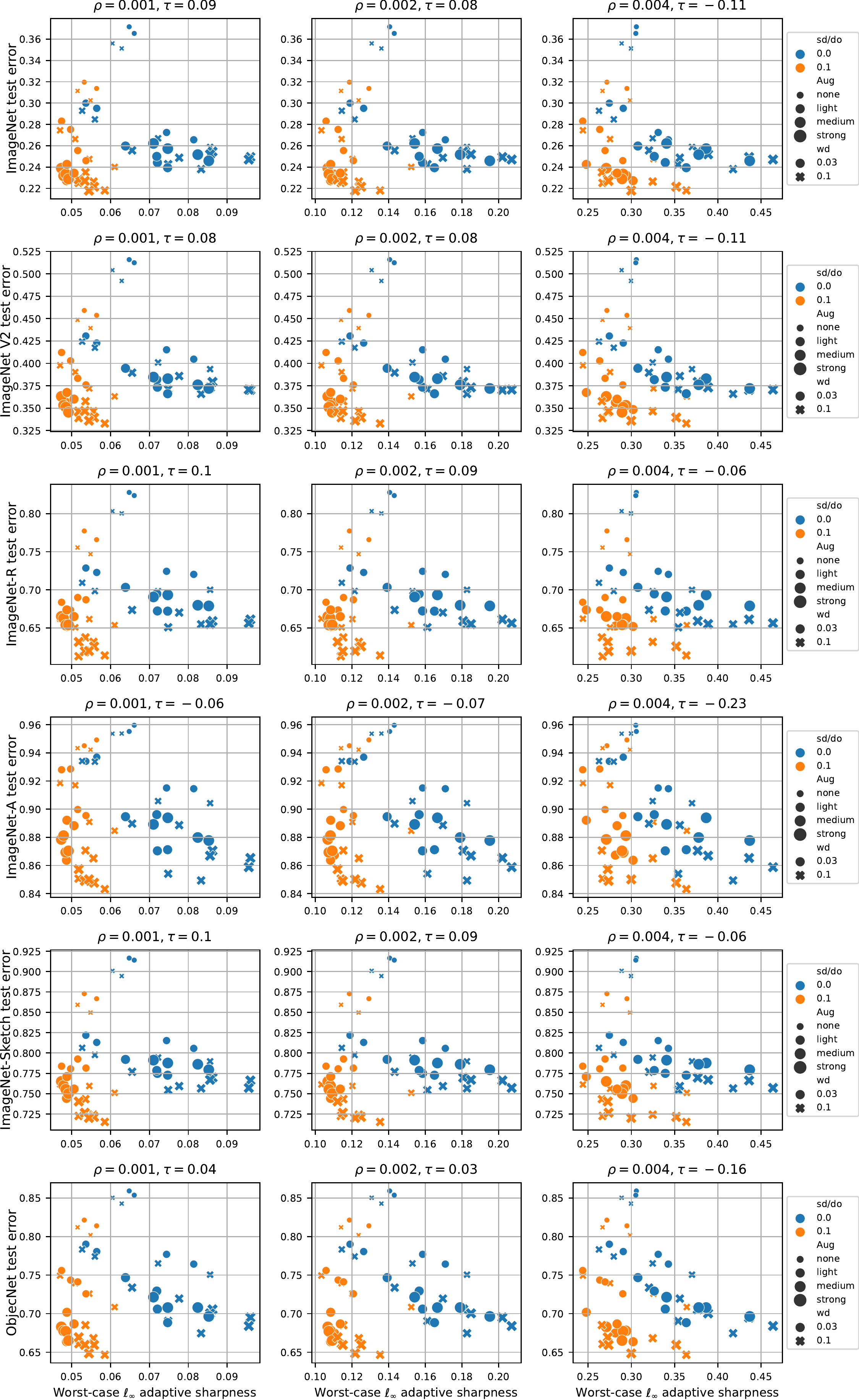}
    \end{tabular}
    \caption{Correlation of sharpness with generalization on ImageNet for different $\rho$ and  for different distribution shifts.}\label{fig:i1k_worst-adaptive-logit}
\end{figure}

\begin{figure}[p] \centering \small
    \begin{tabular}{c}
        \textbf{Worst-case $\ell_\infty$ adaptive sharpness without logit normalization}\\
        \includegraphics[width=.8\columnwidth]{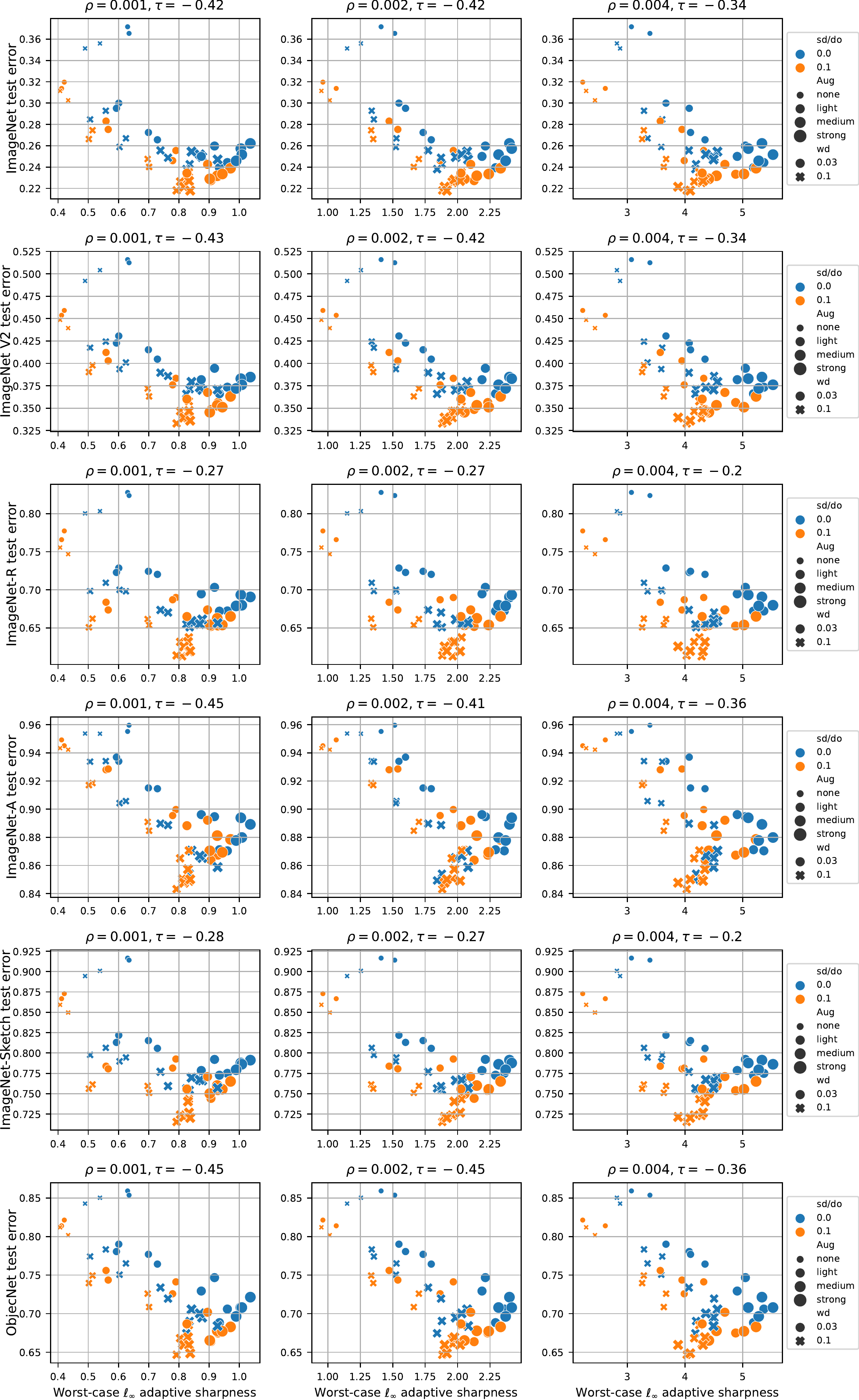}
    \end{tabular}
    \caption{Correlation of sharpness with generalization on ImageNet for different $\rho$ and for different distribution shifts.}\label{fig:i1k_worst-adaptive-normal}
\end{figure}

\begin{figure}[p] \centering \small
    \begin{tabular}{c}
        \textbf{Average-case adaptive sharpness with logit normalization}\\
        \includegraphics[width=.8\columnwidth]{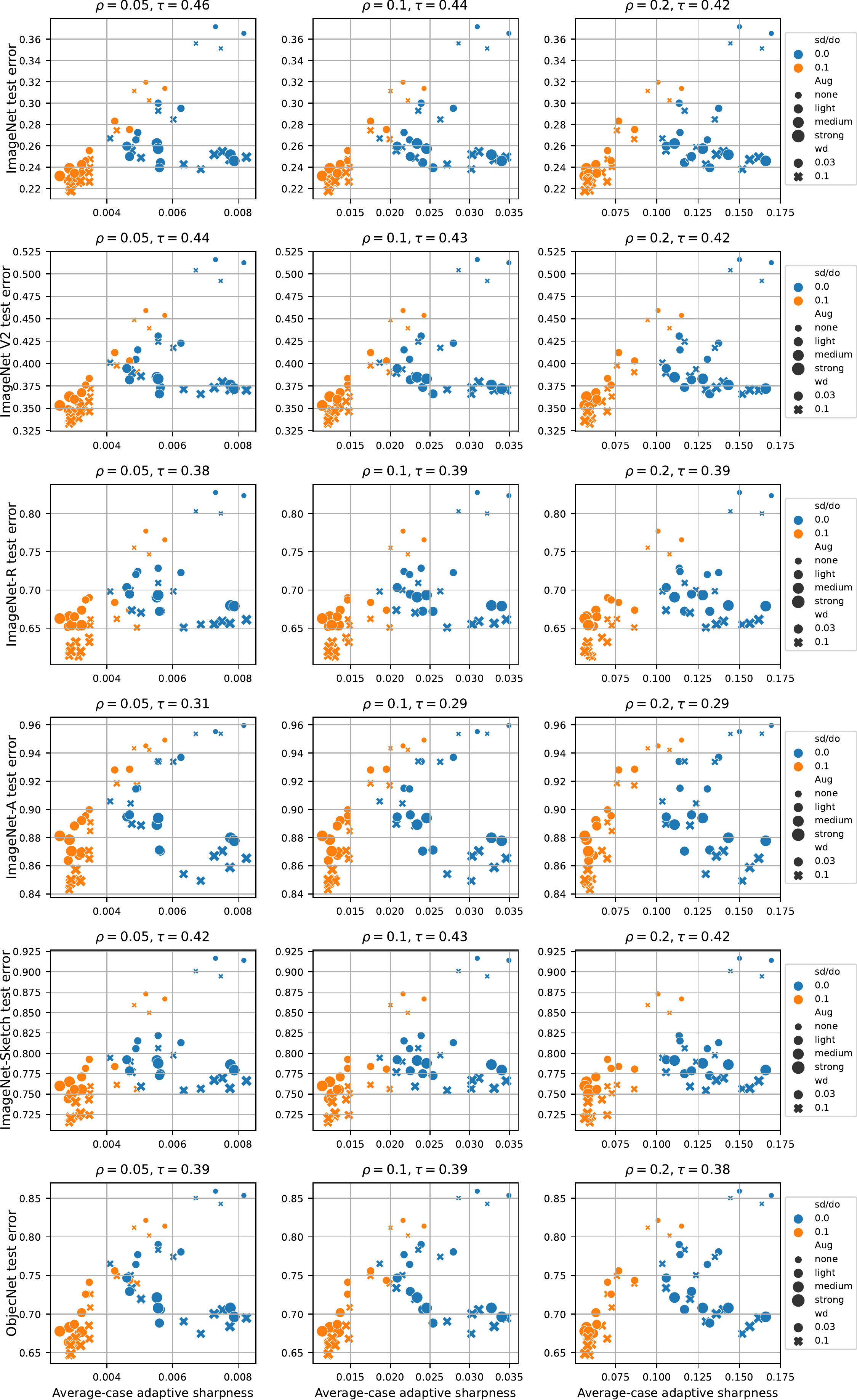}
    \end{tabular}
    \caption{Correlation of sharpness with generalization on ImageNet for different $\rho$ and for different distribution shifts.}\label{i1k_avg-adpative-logit}
\end{figure}

\begin{figure}[p] \centering \small
    \begin{tabular}{c}
        \textbf{Average-case adaptive sharpness without logit normalization}\\
        \includegraphics[width=.8\columnwidth]{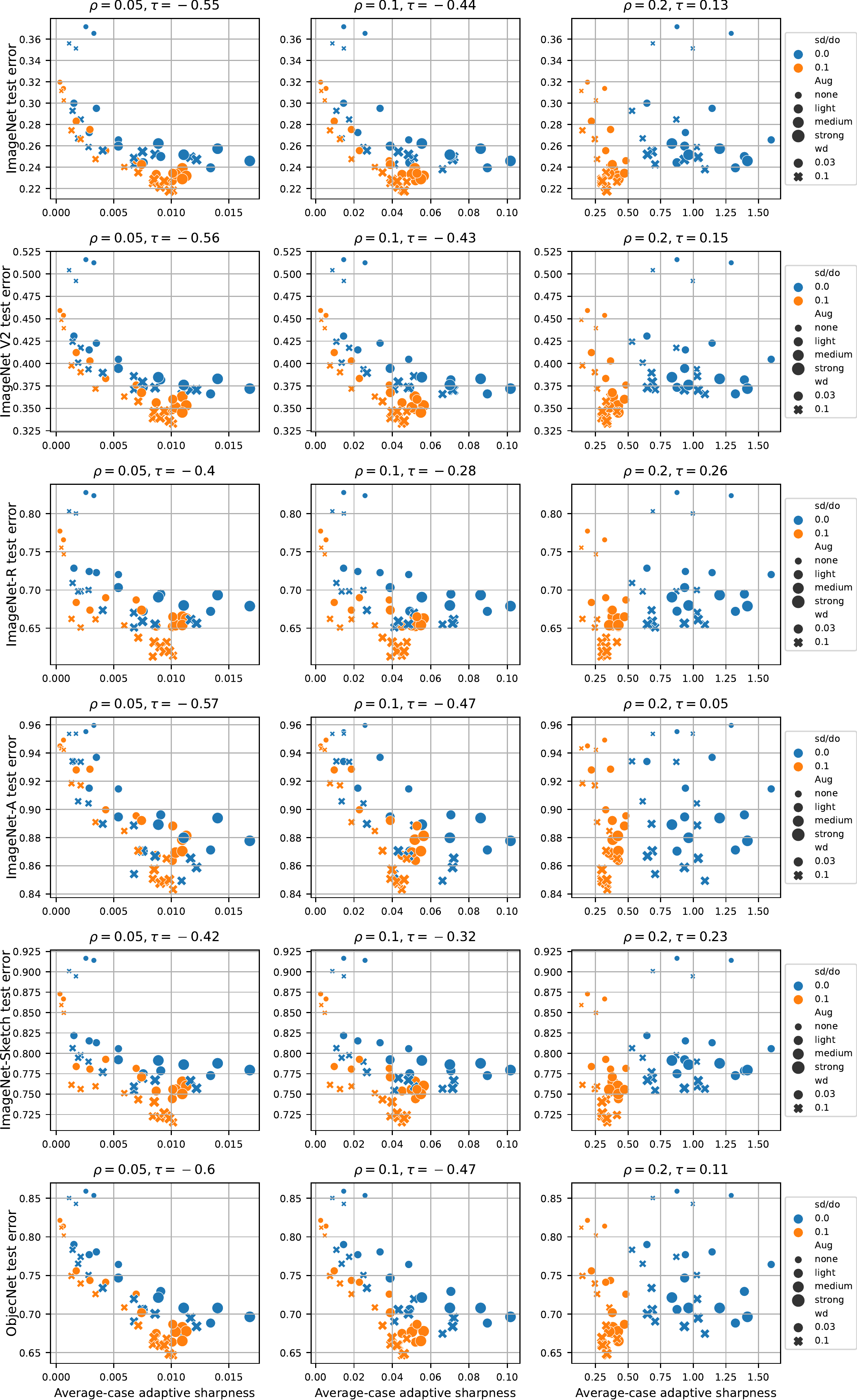}
    \end{tabular}
    \caption{Correlation of sharpness with generalization on ImageNet for different $\rho$ and for different distribution shifts.}\label{i1k_avg-adpative-normal}
\end{figure}

\clearpage

\section{Fine-tuning of ImageNet-1k Models Pretrained on ImageNet-21k from \citet{steiner2021train}: Extra Figures and Details}
\label{sec:app_IN21k-IN1k-extrafigures}
\textbf{Experimental details.} All hyperparameter settings are identical to those explained in Appendix~\ref{sec:app_IN-1k-extrafigures}, only the pretraining dataset is ImageNet-21k instead of ImageNet-1k. Since two of the models showed close to 100\% test error, we did not evaluate them, resulting in 54 instead of 56 models.

\textbf{Extra figures.} Like in Appendix~\ref{sec:app_IN-1k-extrafigures} we show each sharpness definition for three values of $\rho$ and its the correlation to test error on ImageNet (in-distribution) and on the various distribution shifts. The observations are very similar to those on ImageNet-1k pretraining: sharpness variants are not predictive of the performance on ImageNet and the distribution shift datasets, typically only separating models by stochastic depth / dropout, and often even yielding a negative correlation with OOD test error. 

\begin{figure}[p] \centering \small
    \begin{tabular}{c}
        \textbf{Worst-case $\ell_\infty$ adaptive sharpness with logit normalization}\\
        \includegraphics[width=.8\columnwidth]{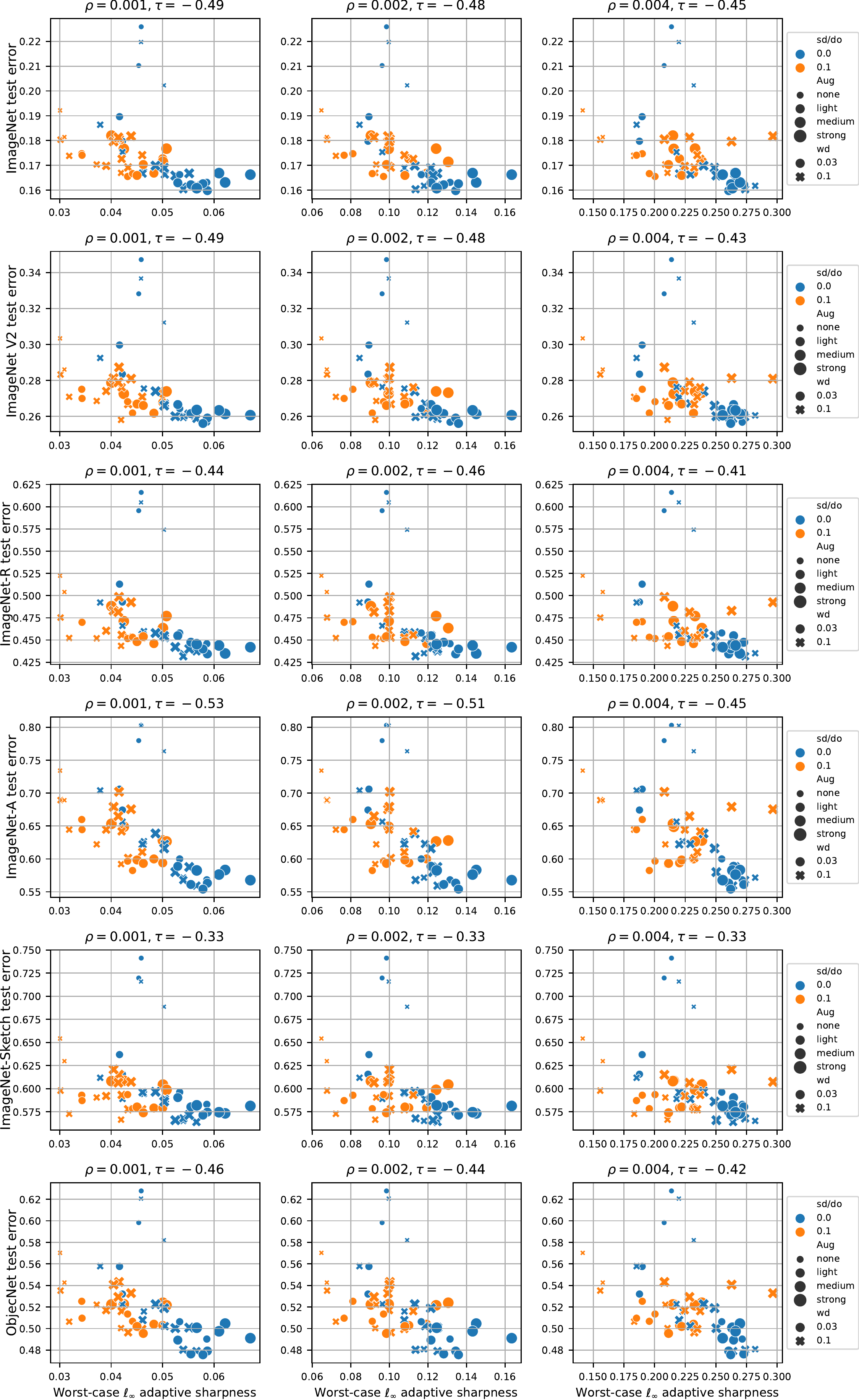}
    \end{tabular}
    \caption{Correlation of sharpness with generalization on ImageNet for different $\rho$ and for different distribution shifts.}\label{i21k_worst-adaptive-logit}
\end{figure}

\begin{figure}[p] \centering \small
    \begin{tabular}{c}
        \textbf{Worst-case $\ell_\infty$ adaptive sharpness without logit normalization}\\
        \includegraphics[width=.8\columnwidth]{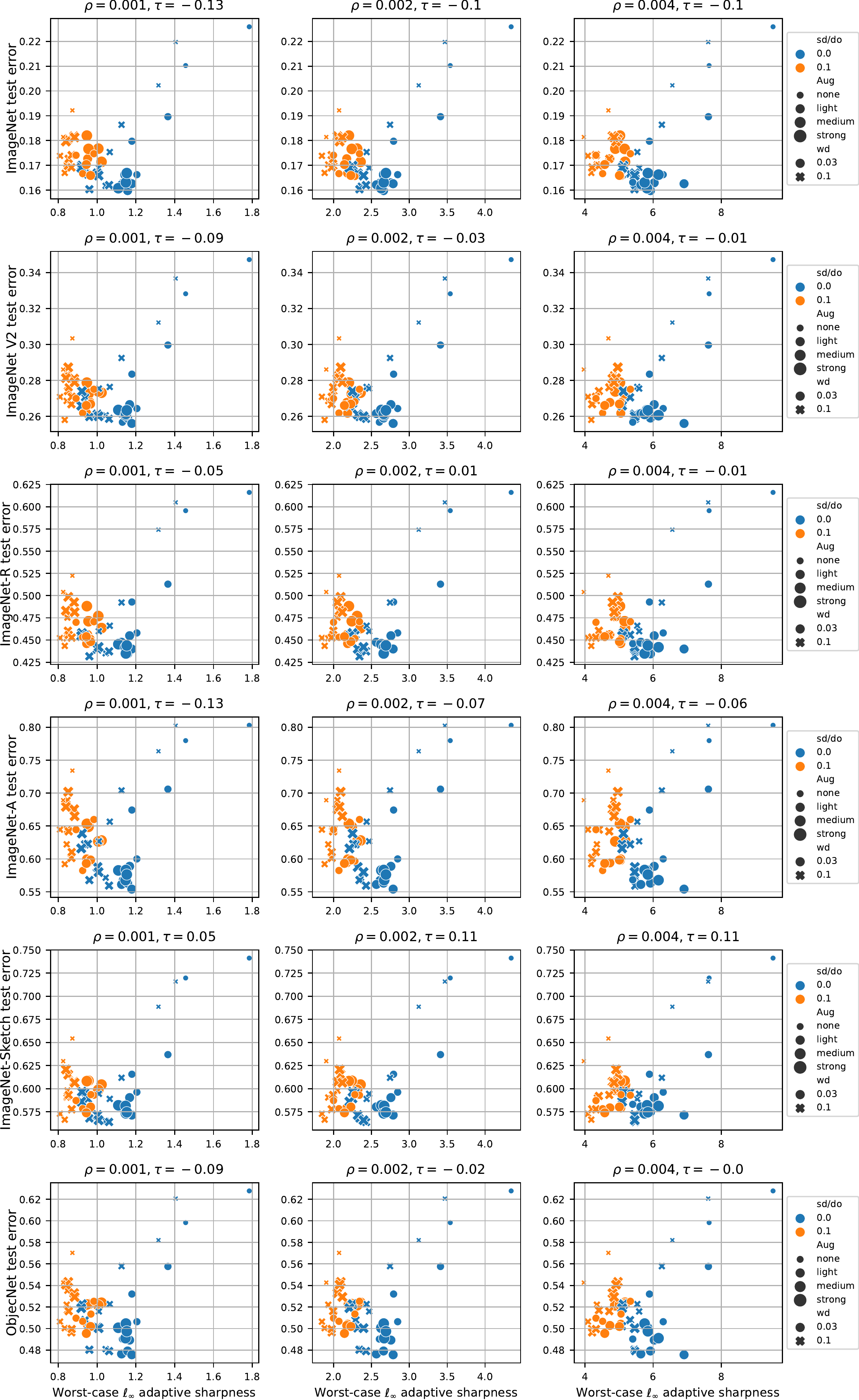}
    \end{tabular}
    \caption{Correlation of sharpness with generalization on ImageNet for different $\rho$ and for different distribution shifts.}\label{i21k_worst-adaptive-normal}
\end{figure}

\begin{figure}[p] \centering \small
    \begin{tabular}{c}
        \textbf{Average-case adaptive sharpness with logit normalization}\\
        \includegraphics[width=.8\columnwidth]{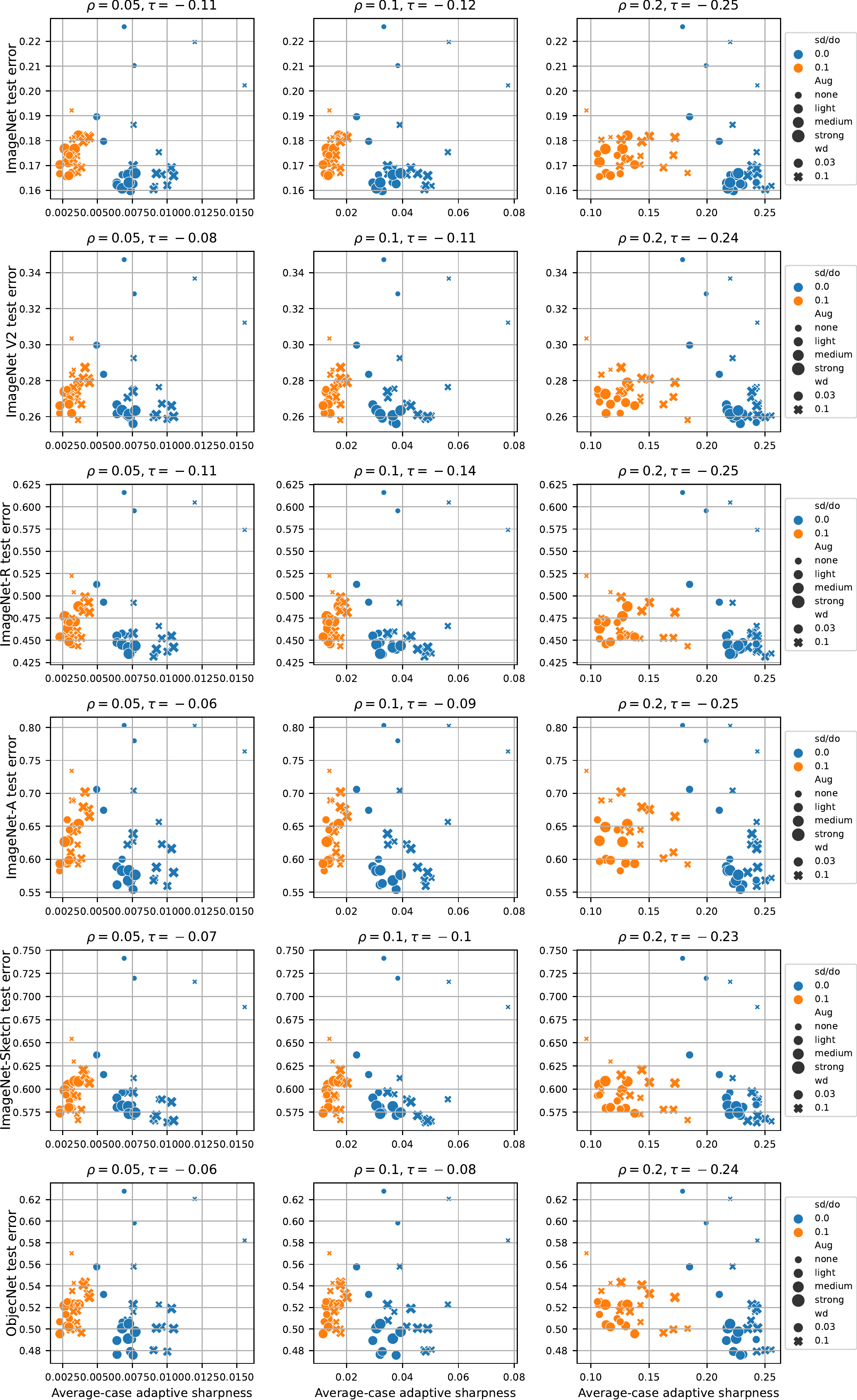}
    \end{tabular}
    \caption{Correlation of sharpness with generalization on ImageNet for different $\rho$ and for different distribution shifts.}\label{i21k_avg-adpative-logit}
\end{figure}

\begin{figure}[p] \centering \small
    \begin{tabular}{c}
        \textbf{Average-case adaptive sharpness without logit normalization}\\
        \includegraphics[width=.8\columnwidth]{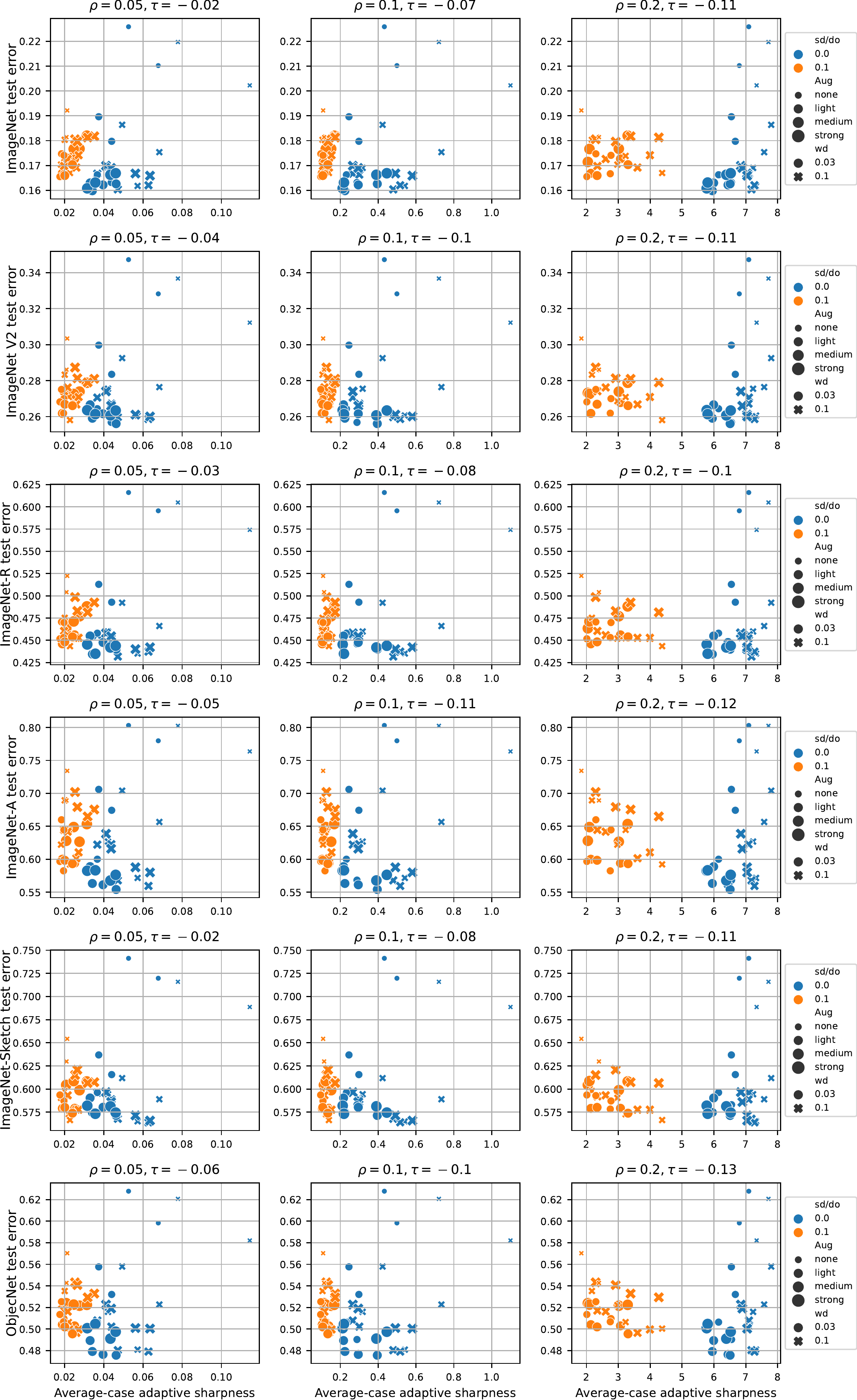}
    \end{tabular}
    \caption{Correlation of sharpness with generalization on ImageNet for different $\rho$ and for different distribution shifts.}\label{i21k_avg-adpative-normal}
\end{figure}

\clearpage

\section{ImageNet Models both Pretrained on ImageNet-1k and ImageNet-21k from \citet{steiner2021train}}
\label{sec:app_bothIN21k-IN1k-extrafigures}
For completeness, we here show for two sharpness definitions the models pretrained on ImageNet-21k and ImageNet-1k together. We find the better-generalizing models pretrained on ImageNet-21k to have significantly higher worst-case sharpness, and roughly equal or higher logit-normalized average-case adaptive sharpness, underlining that the models generalization properties resulting from different pretraining datasets are not captured.

\begin{figure}[p] \centering \small
    \begin{tabular}{c}
        \textbf{Worst-case $\ell_\infty$ adaptive sharpness without logit normalization}\\
        \includegraphics[width=.8\columnwidth]{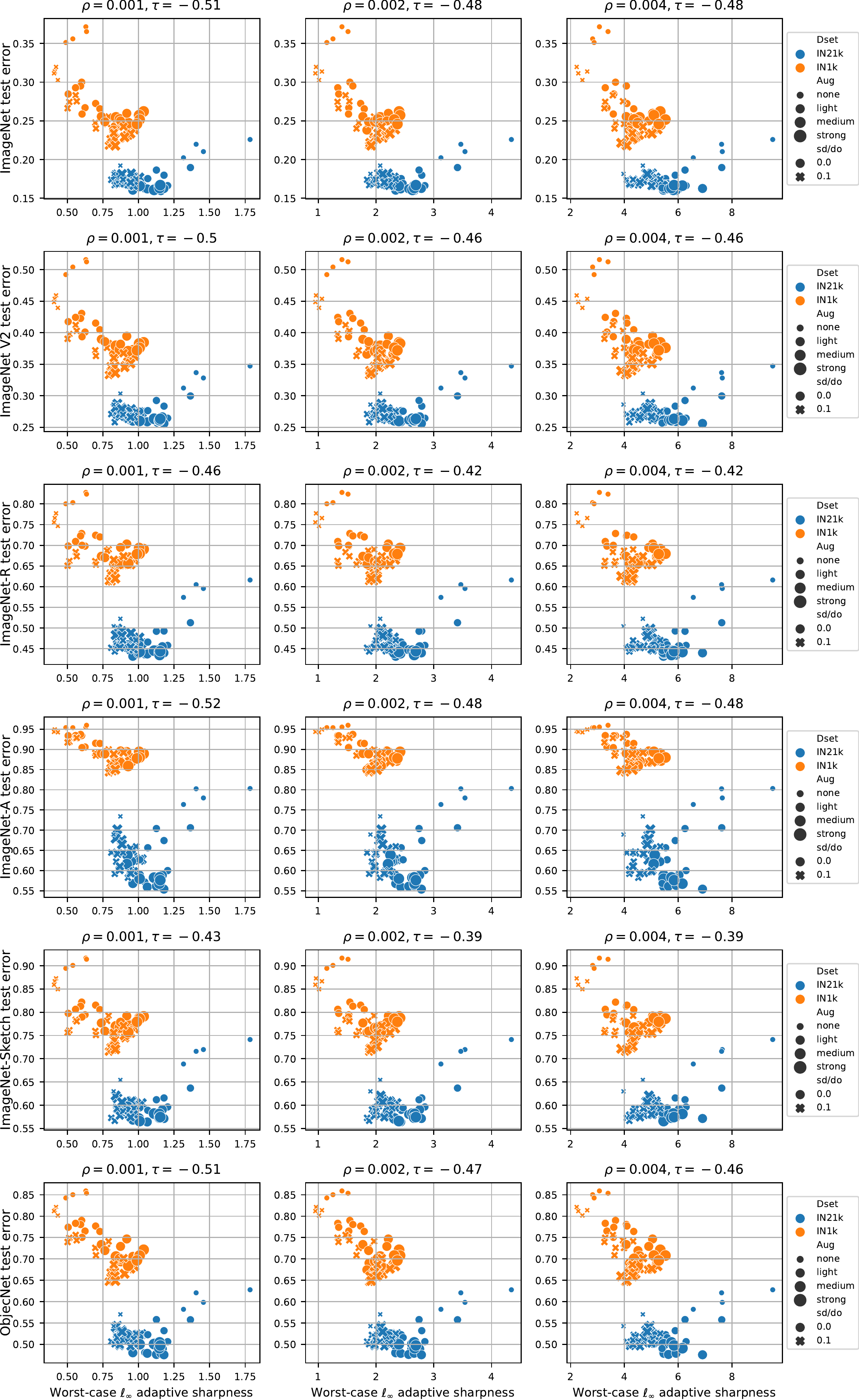}
    \end{tabular}
    \caption{Correlation of sharpness with generalization on ImageNet-1k for different $\rho$ and for different distribution shifts.}
\end{figure}

\begin{figure}[p] \centering \small
    \begin{tabular}{c}
        \textbf{Average-case adaptive sharpness with logit normalization}\\
        \includegraphics[width=.8\columnwidth]{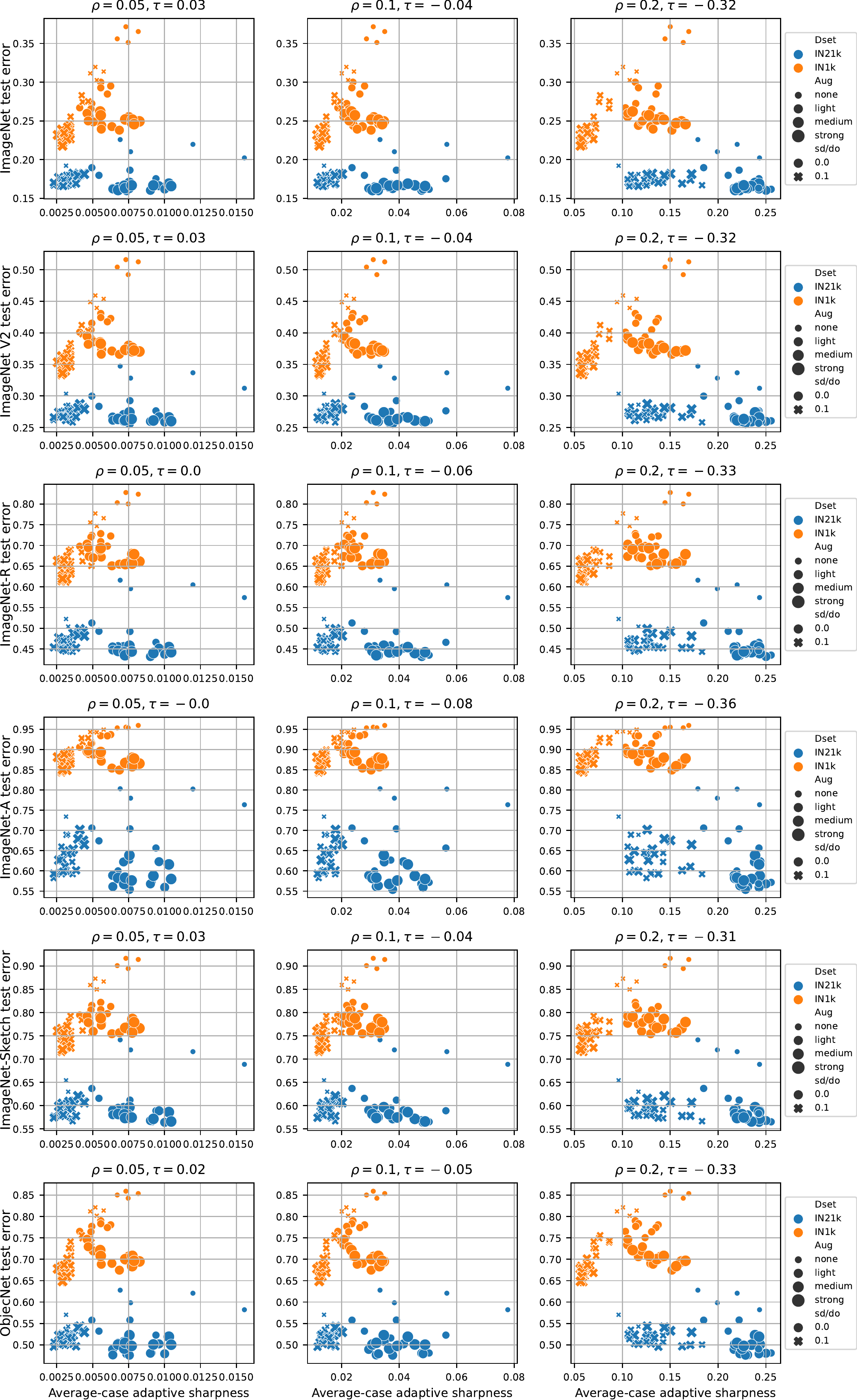}
    \end{tabular}
    \caption{Correlation of sharpness with generalization on ImageNet-1k for different $\rho$ and for different distribution shifts.}\label{both-avg-logit}
\end{figure}

\clearpage

\section{Fine-tuning CLIP Models on ImageNet: Extra Details and Figures}
\label{sec:app_finetuning_clip_on_imagenet}

\begin{figure}[h] \centering\small
    \tabcolsep=1.1pt
    \begin{tabular}{c c cc}
        \multicolumn{2}{c}{\textbf{With logit normalization}} &\multicolumn{2}{c}{\textbf{Without logit normalization}}\\
        \includegraphics[width=.25\columnwidth]{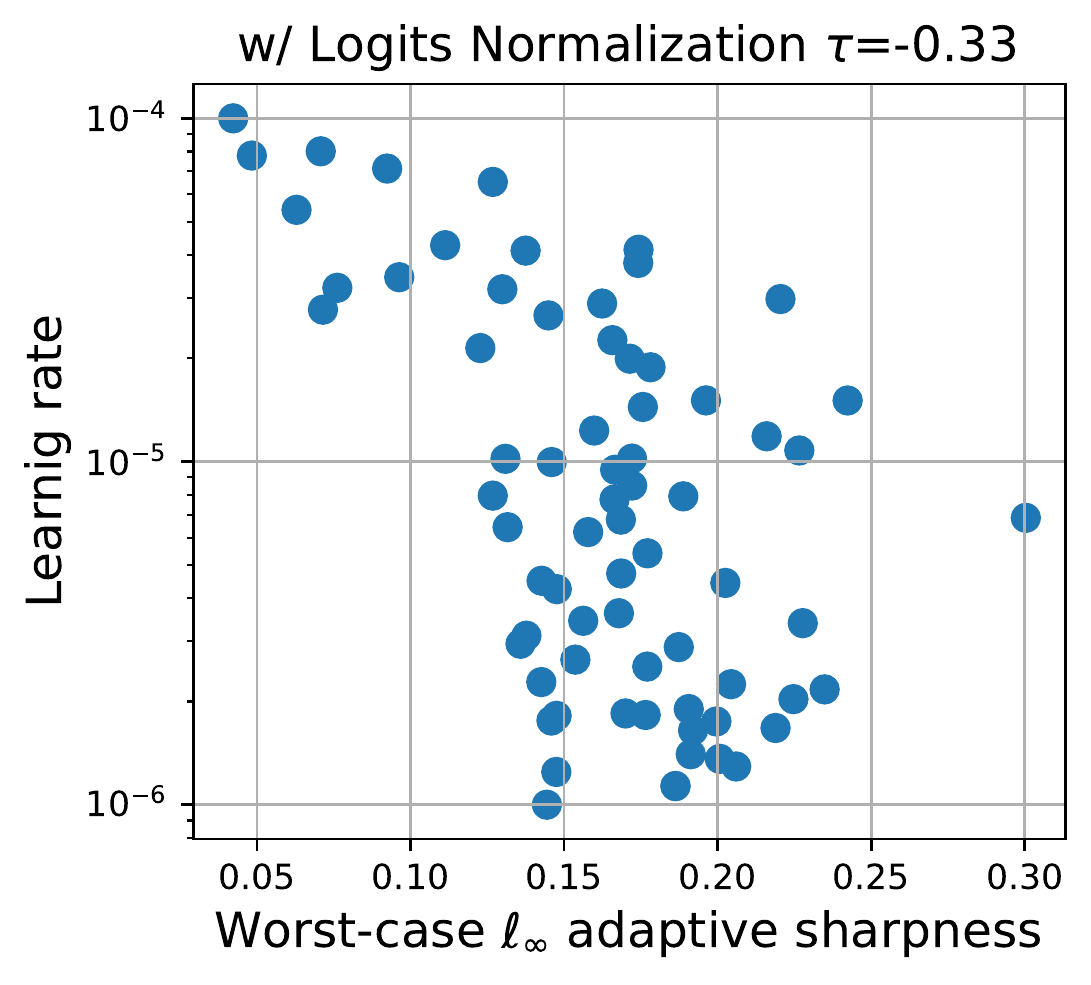}&\includegraphics[width=.25\columnwidth]{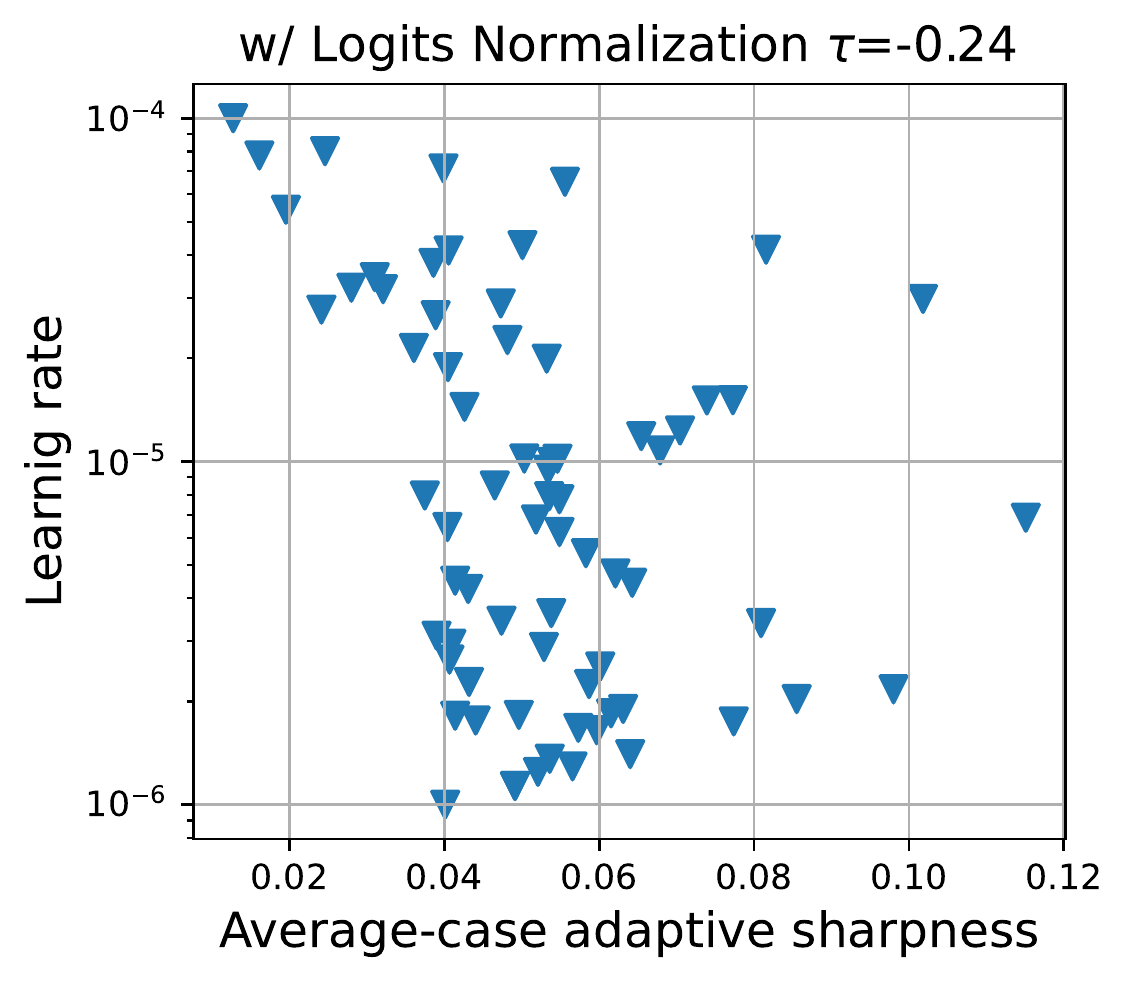}
        &
        \includegraphics[width=.25\columnwidth]{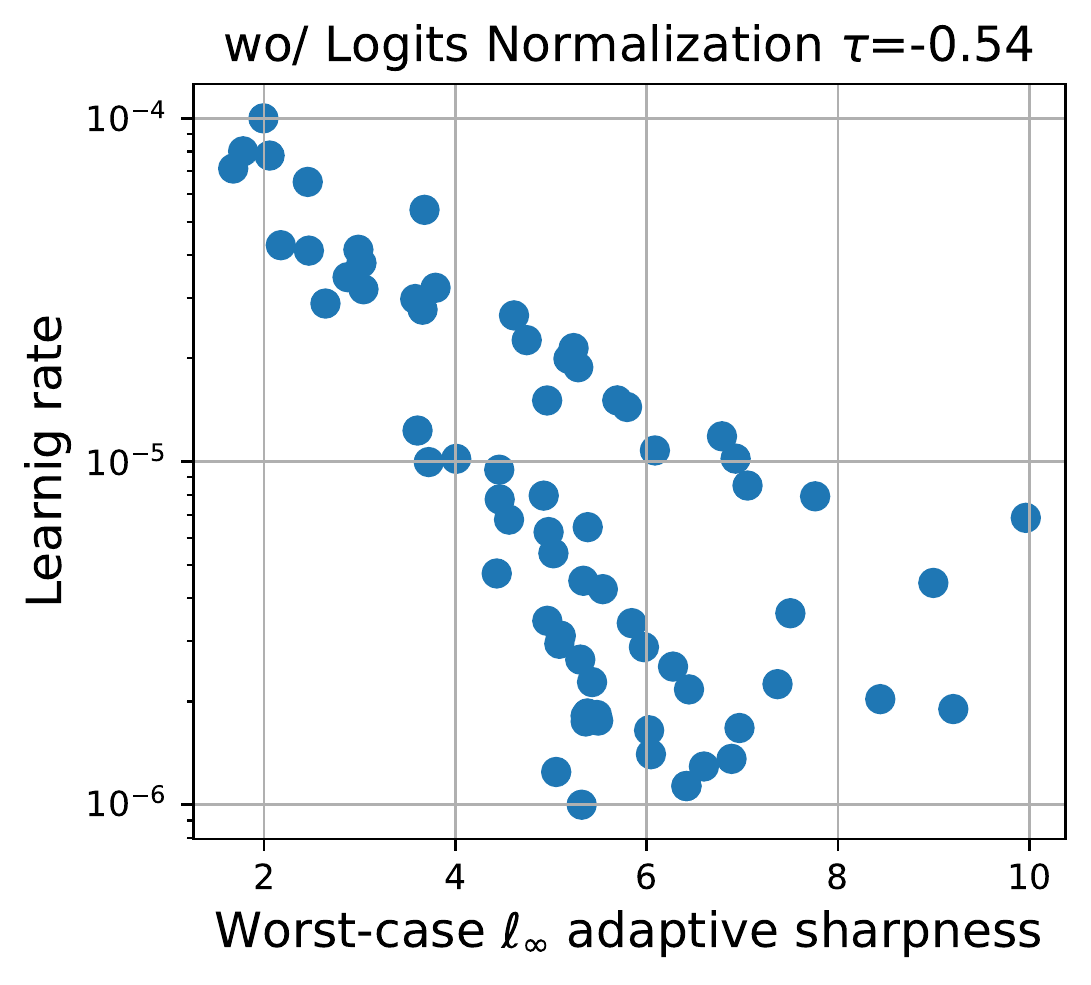}&\includegraphics[width=.25\columnwidth]{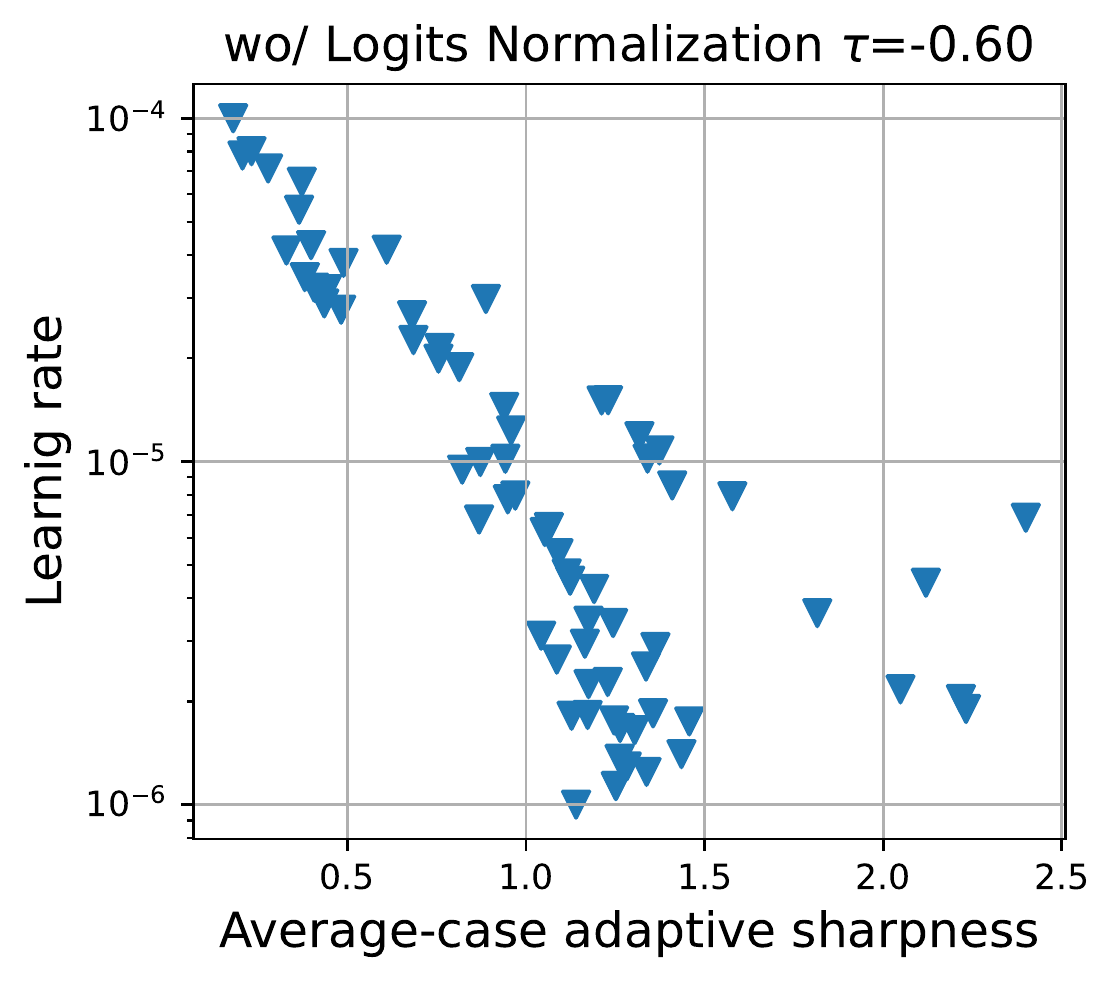}
        
    \end{tabular}
    \caption{\textbf{Fine-tuning CLIP ViT-B/32 on ImageNet-1k.} Sharpness negatively correlates with the size of learning rate for fine-tuning, both with (left) and without (right) using logit normalization. For worst-case sharpness $\rho=0.002$ is used, for average-case sharpness $\rho=0.1$.} 
    \label{fig:imagenet_clip_lr_sharpness}
    
\end{figure}

\begin{figure}[p] \centering \small
    \begin{tabular}{c}
        \textbf{Worst-case $\ell_\infty$ adaptive sharpness with logit normalization}\\
        \includegraphics[width=.8\columnwidth]{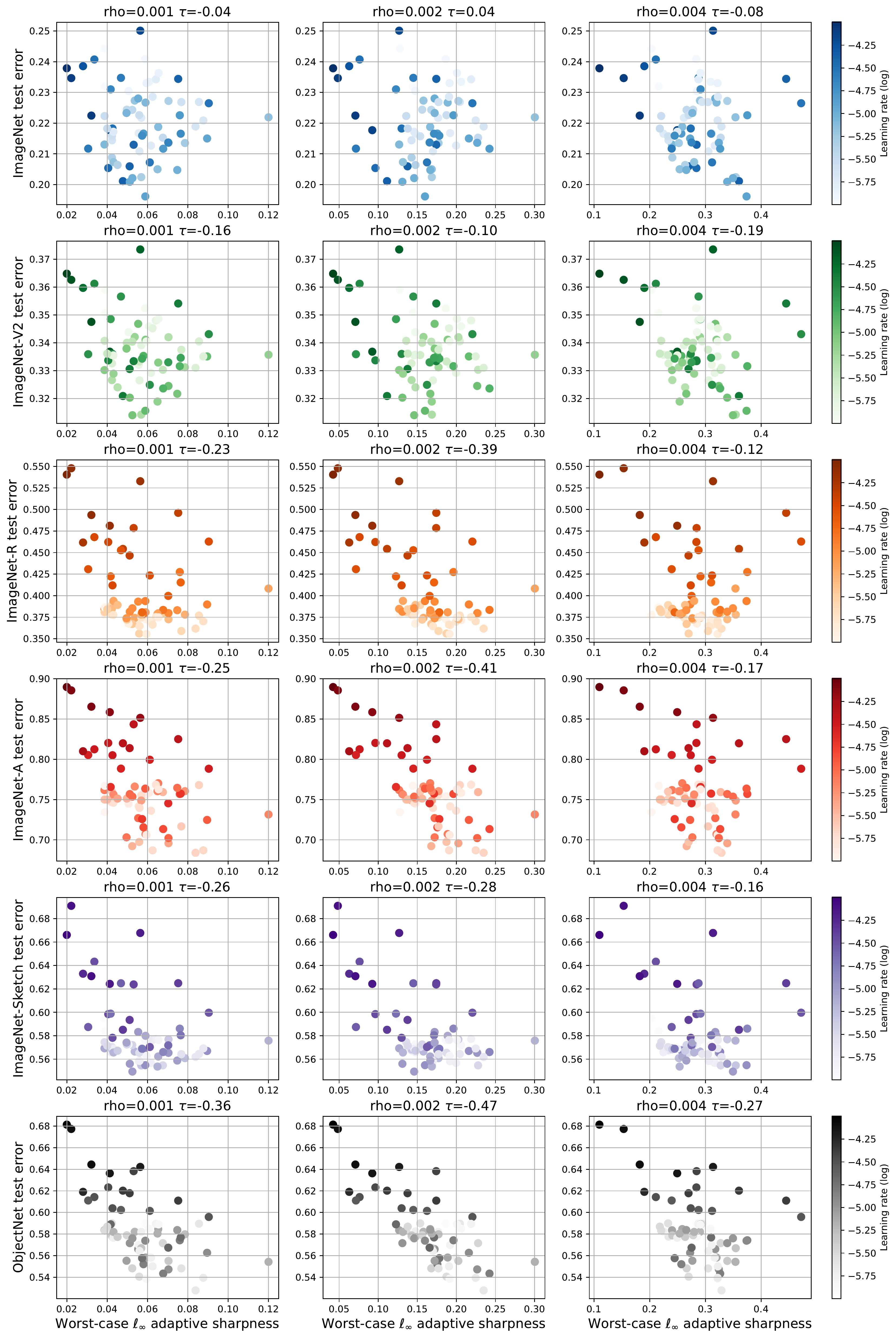}
    \end{tabular}
    \caption{Correlation of sharpness with varying $\rho$ with generalization on ImageNet for different distribution shifts.}\label{fig:imagenet_allshifts_worstcase_sh_logitsnorm}
\end{figure}

\begin{figure}[p] \centering \small
    \begin{tabular}{c}
        \textbf{Worst-case $\ell_\infty$ adaptive sharpness without logit normalization}\\
        \includegraphics[width=.8\columnwidth]{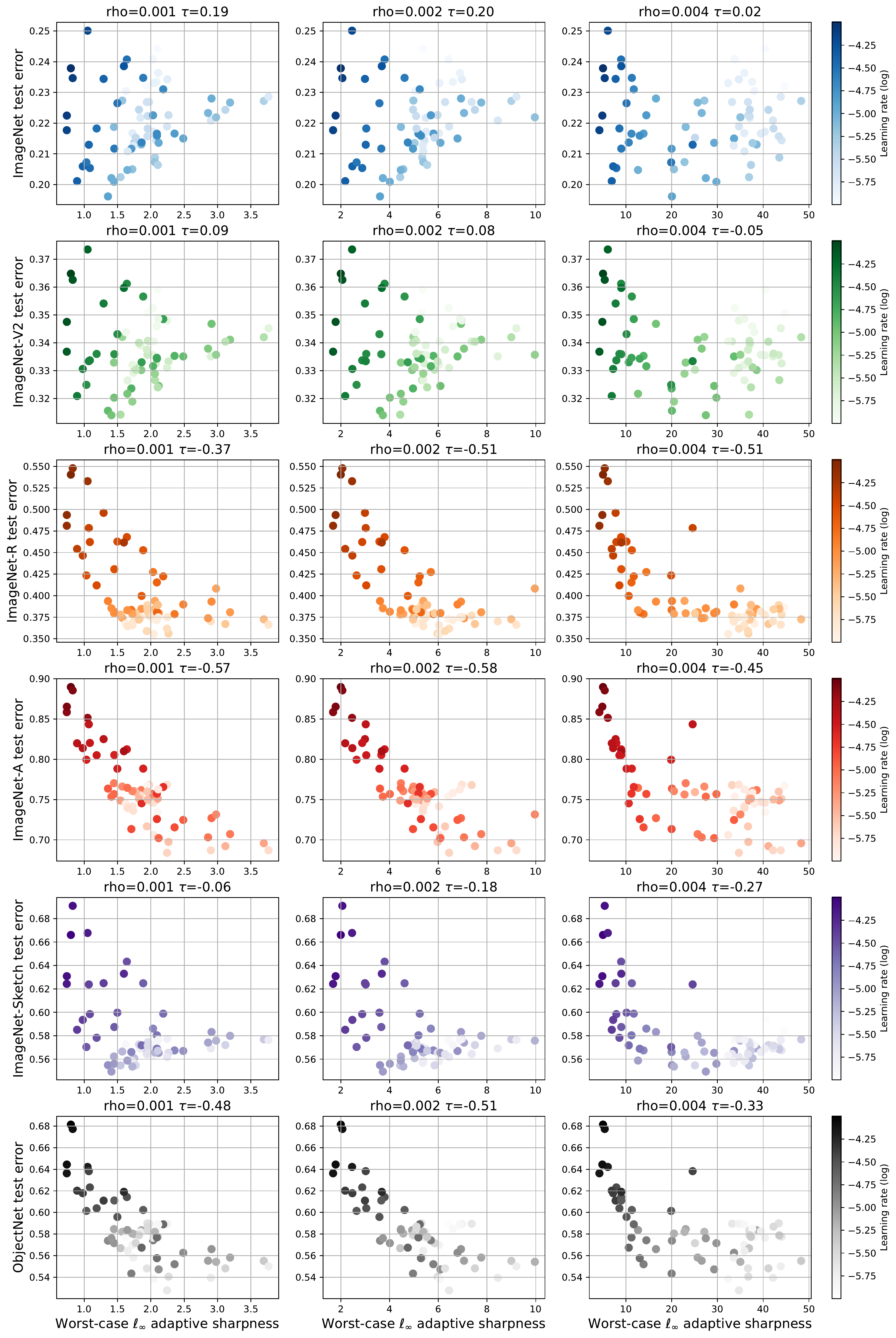}
    \end{tabular}
    \caption{Correlation of sharpness with varying $\rho$ with generalization on ImageNet for different distribution shifts.}\label{fig:imagenet_allshifts_worstcase_sh}
\end{figure}

\begin{figure}[p] \centering \small
    \begin{tabular}{c}
        \textbf{Average-case adaptive sharpness with logit normalization}\\
        \includegraphics[width=.8\columnwidth]{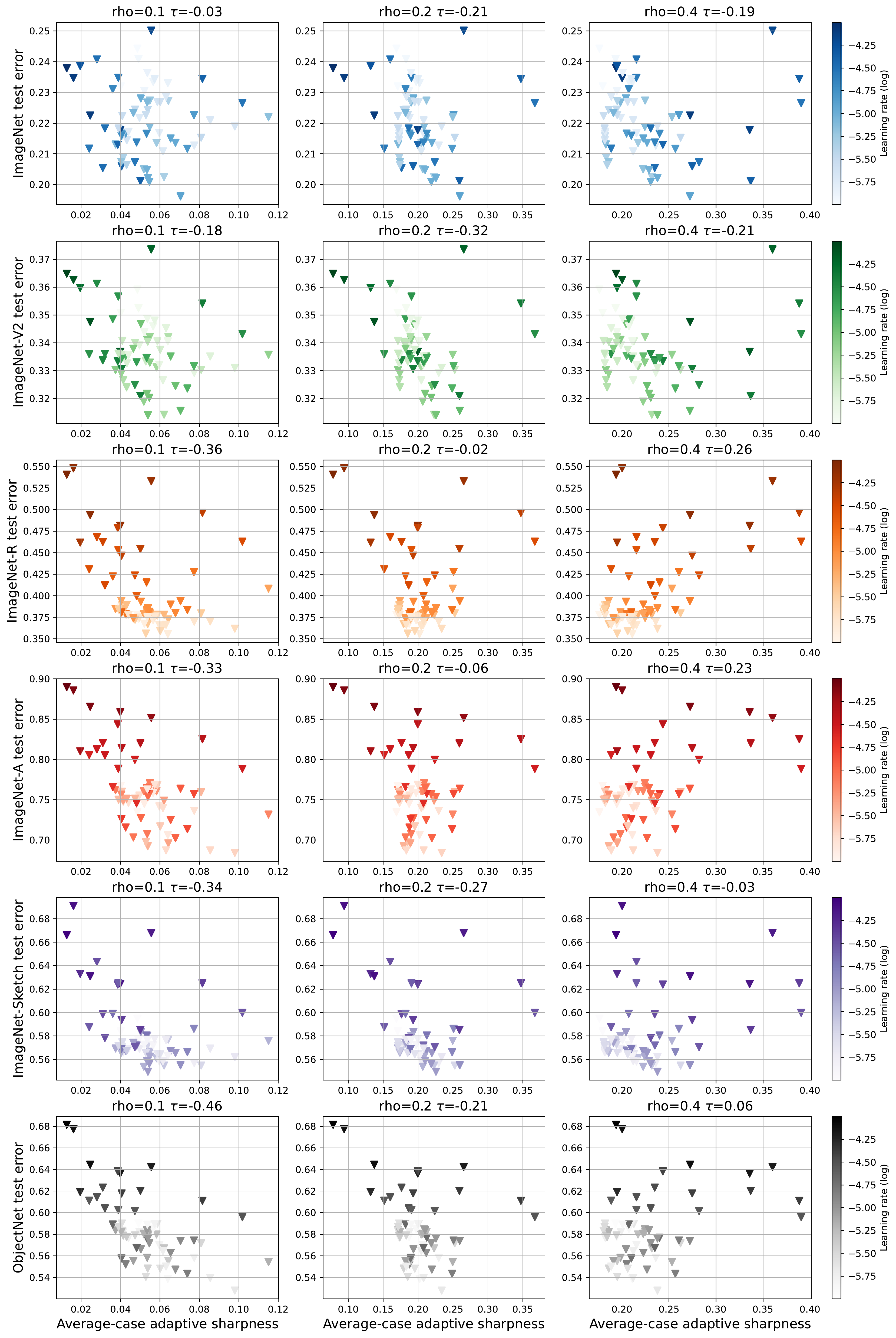}
    \end{tabular}
    \caption{Correlation of sharpness with varying $\rho$ with generalization on ImageNet for different distribution shifts.} \label{fig:imagenet_allshifts_averagecase_sh_logitsnorm}
\end{figure}

\begin{figure}[p] \centering \small
    \begin{tabular}{c}
        \textbf{Average-case adaptive sharpness without logit normalization}\\
        \includegraphics[width=.8\columnwidth]{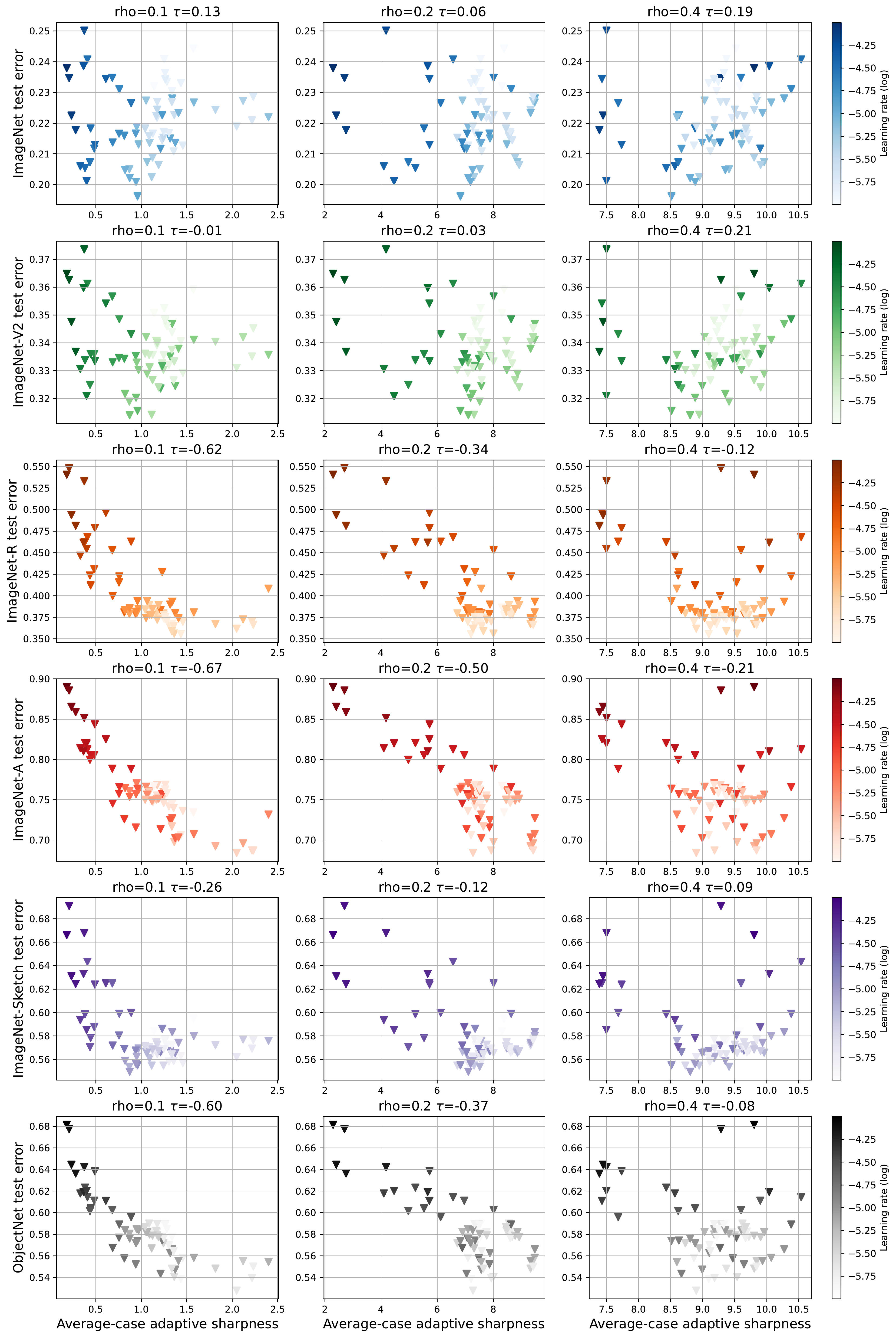}
    \end{tabular}
    \caption{Correlation of sharpness with varying $\rho$ with generalization on ImageNet for different distribution shifts.}\label{fig:imagenet_allshifts_averagecase_sh}
\end{figure}

\myparagraph{Experimental details.} We take advantage of the models fine-tuned by \citet{wortsman2022model} from a pre-trained CLIP ViT-B/32, with randomly sampled training hyperparameters (see \textit{random search} setup in \citet{wortsman2022model}), for which the evaluation of ImageNet validation set and distribution shifts are provided.

\myparagraph{Extra figures.} For each sharpness definition we show for three values of $\rho$ the correlation between test error on ImageNet (in-distribution) and on the various distribution shifts.  In particular, we use worst-case $\ell_\infty$ adaptive sharpness with  (Fig.~\ref{fig:imagenet_allshifts_worstcase_sh_logitsnorm}) and without (Fig.~\ref{fig:imagenet_allshifts_worstcase_sh}) logit normalization, and average-case adaptive sharpness with  (Fig.~\ref{fig:imagenet_allshifts_averagecase_sh_logitsnorm}) and without (Fig.~\ref{fig:imagenet_allshifts_averagecase_sh}) logit normalization. For all figures we represent with colors represent the size of the learning rate used for fine-tuning (darker color means larger learning rate). In addition to the datasets shown in Sec.~\ref{sec:clip_fine-tuning}, we here report the results for ImageNet-V2 \citep{recht2019imagenet} and ObjectNet \citep{barbu2019objectnet}. ImageNet-V2 consists in a new test set for ImageNet models and is sampled from the same image distribution as the existing validation set: then, the performance of the classifiers on it are highly correlated to that on ImageNet validation set, and ImageNet-V2 cannot be considered a distribution shift in the same sense as the other datasets. In general, we observe that sharpness variants are not predictive of the performance on ImageNet and ImageNet-V2. Moreover, there is in most cases a negative correlation with test error in presence of distribution shifts. We hypothesize that this is related to the influence that the learning rate has on sharpness (see Fig.~\ref{fig:imagenet_clip_lr_sharpness}), i.e. lower values lead to sharper models.

\clearpage

\section{Fine-tuning on MNLI: Extra Details and Figures}\label{sec:app_mnli}

\begin{figure}[p] \centering \small
    \begin{tabular}{c}
        \textbf{Worst-case $\ell_\infty$ adaptive sharpness with logit normalization}\\
        \includegraphics[width=.8\columnwidth]{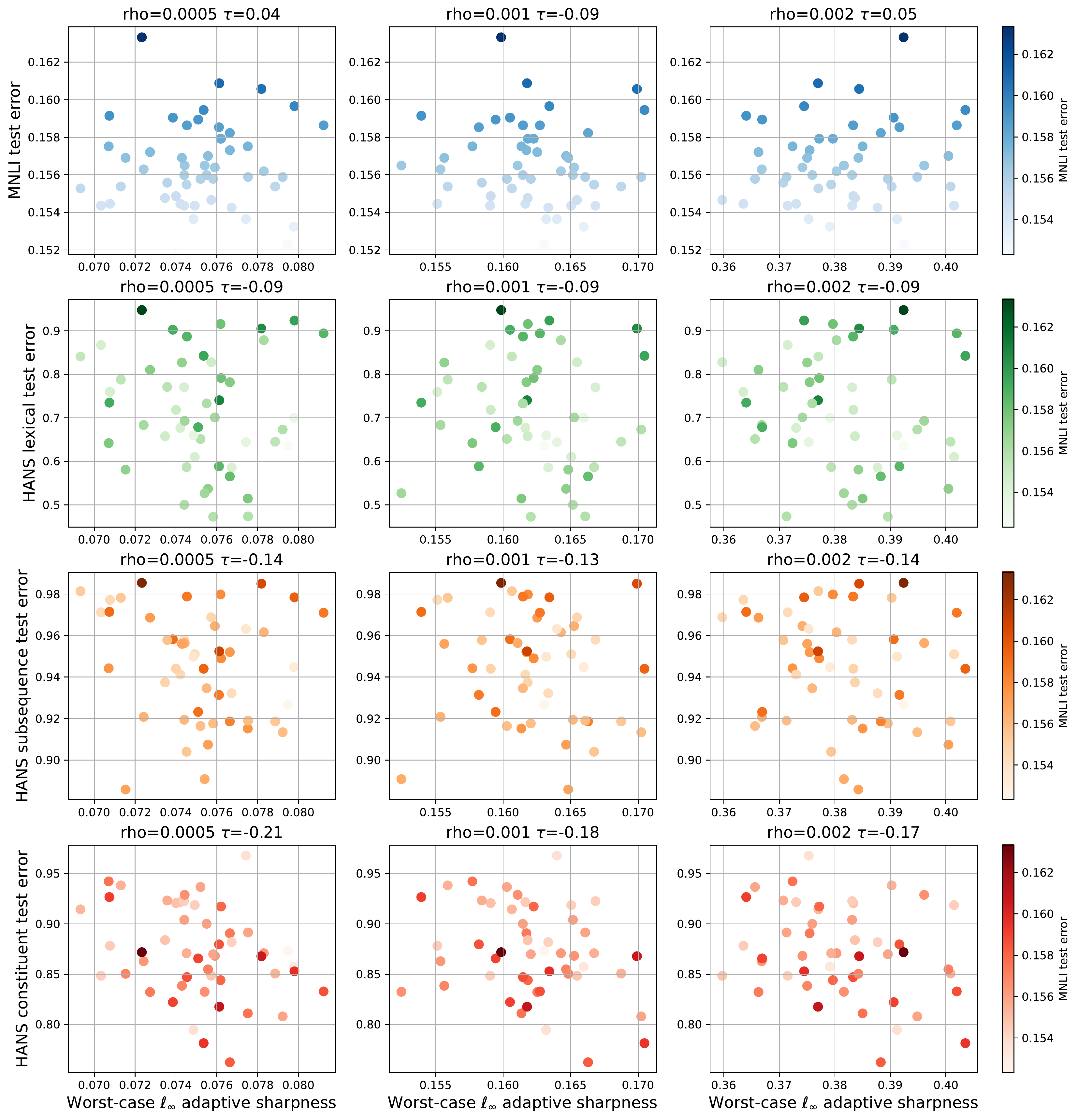}
    \end{tabular}
    \caption{Correlation of sharpness with varying $\rho$ with generalization on MNLI for different distribution shifts.}
    \label{fig:mnli_allshifts_worstcase_sh_logitsnorm}
\end{figure}

\begin{figure}[p] \centering \small
    \begin{tabular}{c}
        \textbf{Worst-case $\ell_\infty$ adaptive sharpness without logit normalization}\\
        \includegraphics[width=.8\columnwidth]{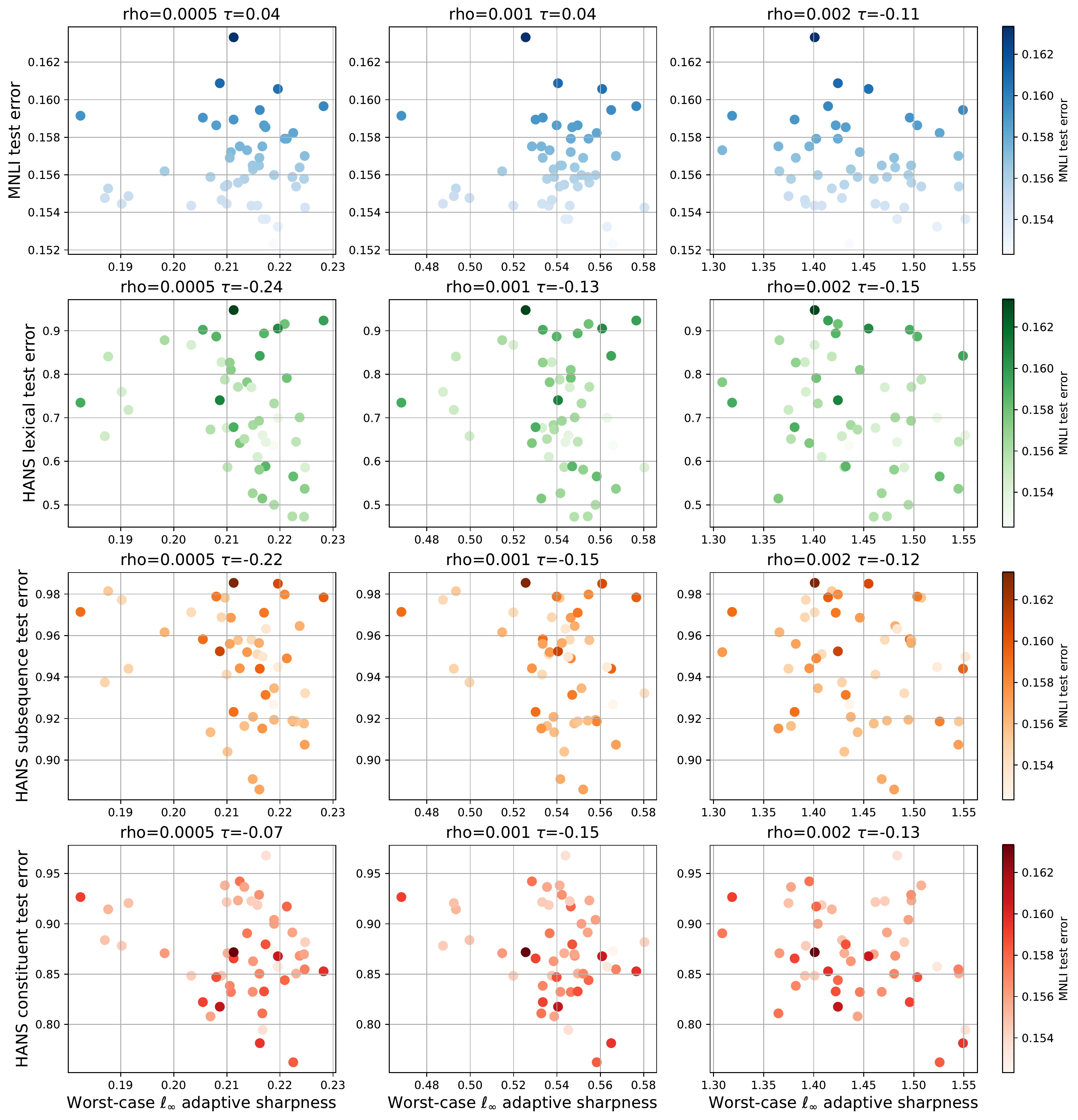}
    \end{tabular}
    \caption{Correlation of sharpness with varying $\rho$ with generalization on MNLI for different distribution shifts.}\label{fig:mnli_allshifts_worstcase_sh}
\end{figure}

\begin{figure}[p] \centering \small
    \begin{tabular}{c}
        \textbf{Average-case adaptive sharpness with logit normalization}\\
        \includegraphics[width=.8\columnwidth]{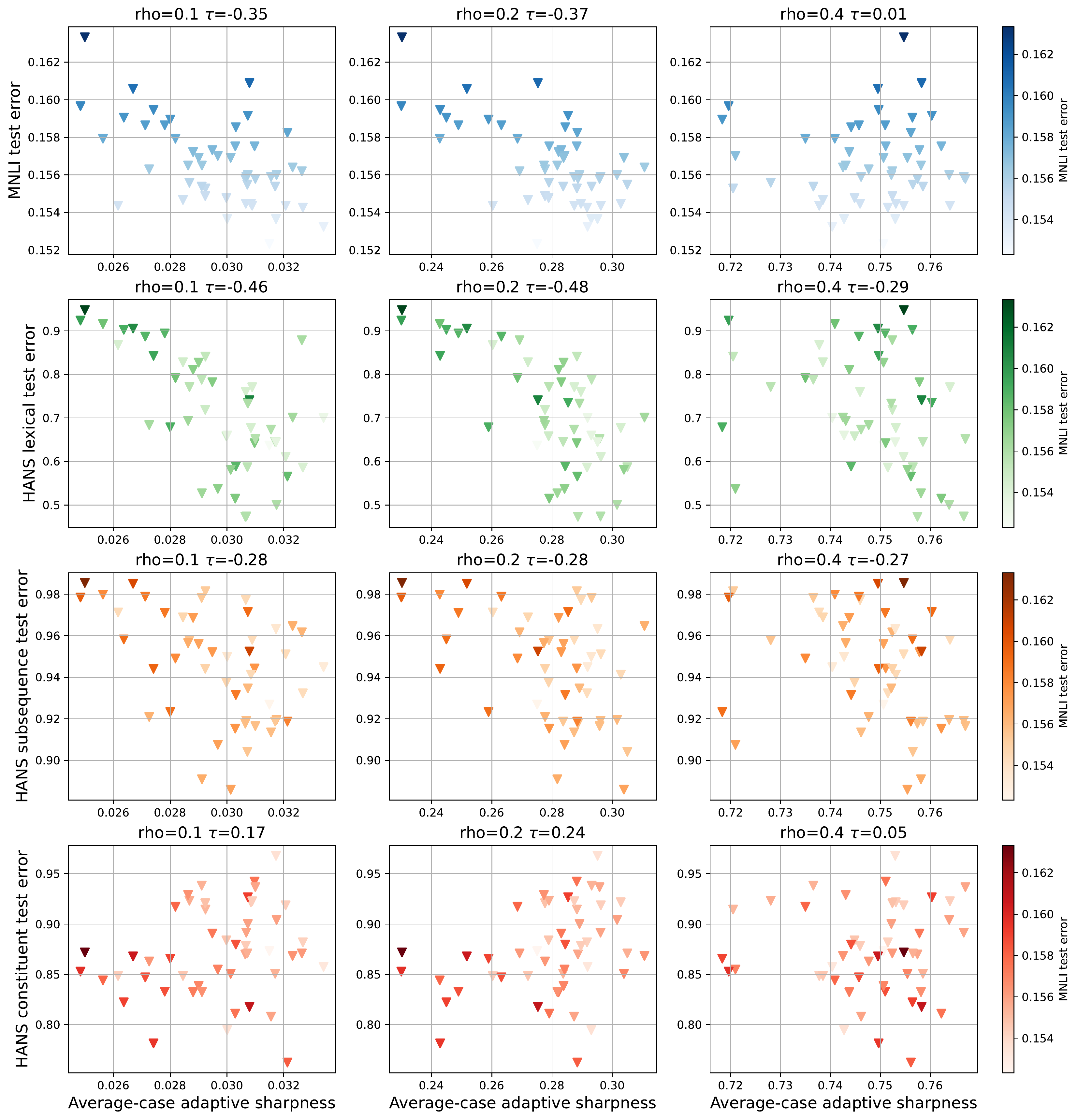}
    \end{tabular}
    \caption{Correlation of sharpness with varying $\rho$ with generalization on MNLI for different distribution shifts.}\label{fig:mnli_allshifts_averagecase_sh_logitsnorm}
\end{figure}

\begin{figure}[p] \centering \small
    \begin{tabular}{c}
        \textbf{Average-case adaptive sharpness without logit normalization}\\
        \includegraphics[width=.8\columnwidth]{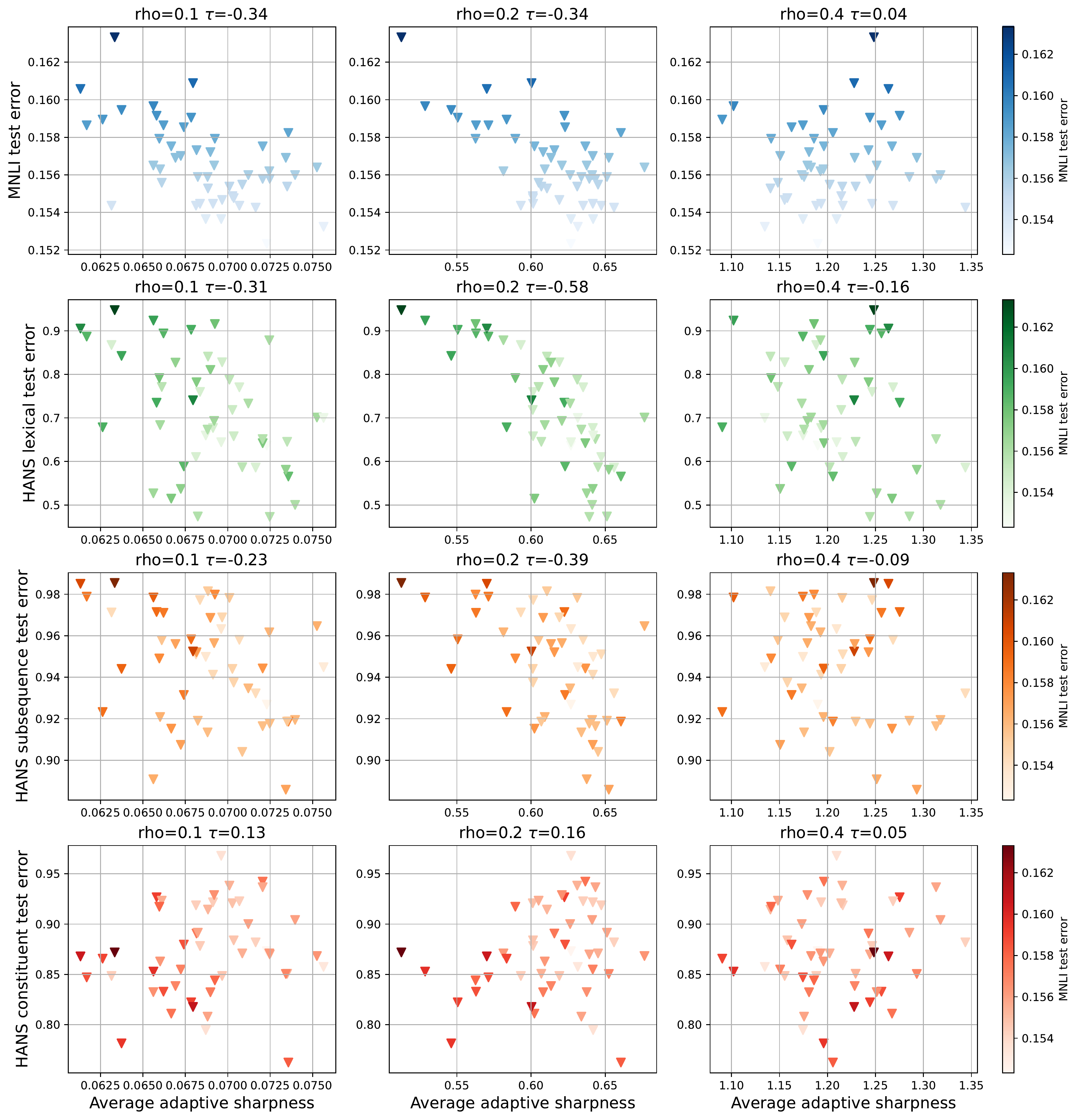}
    \end{tabular}
    \caption{Correlation of sharpness with varying $\rho$ with generalization on MNLI for different distribution shifts.}\label{fig:mnli_allshifts_averagecase_sh}
\end{figure}

\myparagraph{Experimental details.} The models from \citet{mccoy-etal-2020-berts} we use are BERT fine-tuned with initialization weights of \texttt{bert-case-uncased}. The in-distribution test error is computed on the MNLI \texttt{matched} development set, that is a classification task with three classes. As out-of-distribution datasets we use three categories of HANS considered ``Inconsistent with heuristic'' (see \citet{mccoy-etal-2020-berts}: \textit{Lexical overlap}, on which the classifiers show the largest variance in test error, \textit{Subsequence} and \textit{Constituent}. In this case, there are only two possible classes.

\myparagraph{Extra figures.} For each sharpness definition we show for three values of $\rho$ the correlation between test error on MNLI (in-distribution) and on various HANS subsets (out-of-distribution). In particular, we use worst-case $\ell_\infty$ adaptive sharpness with  (Fig.~\ref{fig:mnli_allshifts_worstcase_sh_logitsnorm}) and without (Fig.~\ref{fig:mnli_allshifts_worstcase_sh}) logit normalization, and average-case adaptive sharpness with  (Fig.~\ref{fig:mnli_allshifts_averagecase_sh_logitsnorm}) and without (Fig.~\ref{fig:mnli_allshifts_averagecase_sh}) logit normalization. For all figures we represent with darker colors the models with higher test error on MNLI. In general, all sharpness variants we consider are not predictive of the generalization performance of the model, and in some cases (e.g. Fig.~\ref{fig:mnli_allshifts_averagecase_sh}) there is rather a weak negative correlation between sharpness and test error on out-of-distribution tasks.

\clearpage

\section{Training from Scratch on CIFAR-10: Extra Details and Figures}
\label{sec:app_cifar10_extra_figures}

\myparagraph{Extra details.} 
We train $200$ ResNet-18 and $200$ ViT models for $200$ epochs using SGD with momentum and linearly decreasing learning rates after a linear warm-up for the first $40\%$ iterations. We found that adding such warm-up to SGD allows us to bridge the gap between SGD and Adam training for ViTs. We use the SimpleViT architecture from the \texttt{vit-pytorch} library which is a modification of the standard ViT \citep{dosovitskiy2020image} with a fixed positional embedding and global average pooling instead of the CLS embedding. We use a ViT model with $4\times4$ patches, depth of $6$ blocks, with $16$ heads,  embedding size $512$, and MLP dimension of $1024$. We sample the learning rate from the log-uniform distribution in the range $[0.005, 0.5]$ for ViTs and $[0.05, 5.0]$ for ResNets. We sample uniformly $\rho \in \{0, 0.05, 0.1\}$ of SAM \citep{foret2021sharpnessaware}, with probability $50\%$ mixup ($\alpha=0.5$) \citep{zhang2017mixup}, and with probability $50\%$ standard augmentations combined with RandAugment (with parameters $N=2$, $M=14$) \citep{cubuk2019randaugment}. We use $2\times$ repeated augmentations to reduce the augmentation variance from RandAugment \citep{fort2021drawing}. For CIFAR-10 models, we only show sharpness for well-trained models that have $\leq 1\%$ training error. We note that this selection criterion leaves more ResNets than ViTs on the figures below.

\myparagraph{Sharpness evaluation.}
For sharpness evaluation we use 1024 data points from the training set split in 8 batches: we compute sharpness on each of them and report the average. For worst-case sharpness we use Auto-PGD for 20 steps (for each batch) with random uniform / Gaussian initialization in the feasible set depending on the $\ell_\infty$ vs. $\ell_2$ norm of sharpness. For average-case sharpness, we sample 100 different weights perturbations for every batch.

\myparagraph{Extra figures.}
We present additional figures in Sec.~\ref{sec:app_cifar10_role_data} on the role of data used to evaluate sharpness, in Sec.~\ref{sec:app_cifar10_role_n_iter} on the role of the number of iterations in Auto-PGD to estimate sharpness, in Sec.~\ref{sec:app_cifar10_role_m} on the role of $m$ in $m$-sharpness, and in Sec.~\ref{sec:app_cifar10_role_defns_radii} on the influence of different sharpness definitions and radii on correlation with generalization.

\subsection{The Role of Data Used for Sharpness Evaluation}
\label{sec:app_cifar10_role_data}
We emphasize that for all experiments, we evaluate sharpness on the original training set (CIFAR-10, ImageNet or MNLI) \textit{without augmentations}. However, one may wonder how sensitive this choice is compared to evaluation on the \textit{augmented} training set, particularly in presence of strong data augmentations such as RandAugment \citep{cubuk2019randaugment} used for training of some models. To test this, in Fig.~\ref{fig:cifar10_resnet_diff_data}, we compare adaptive average-case sharpness computed on the original training set and on augmented training set of CIFAR-10 for ResNets-18. We find that the overall trend is nearly the same for small $\rho$ and differs more strongly for larger $\rho$ where the overall correlation with generalization becomes significantly negative ($-0.74$ for the largest $\rho$) on augmented data. In addition, a side-by-side comparison of sharpness on standard vs. augmented training shows that the relationship between them does not deviate too much from a linear trend, especially when considering separately models trained with and without augmentations.
\begin{figure*}[h]
    \centering
    \textbf{Adaptive average-case $\ell_\infty$ (uniform perturbations) sharpness (normalized) for ResNets-18 on \textit{original} training data}
    \includegraphics[width=1.0\textwidth]{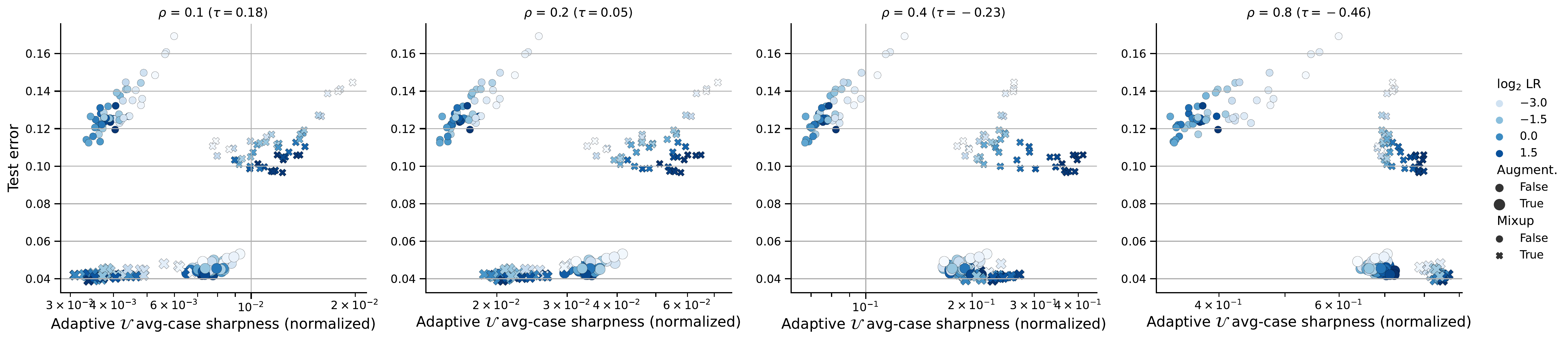}
    \textbf{Adaptive average-case $\ell_\infty$ (uniform perturbations) sharpness (normalized) for ResNets-18 on \textit{augmented} training data}
    \includegraphics[width=1.0\textwidth]{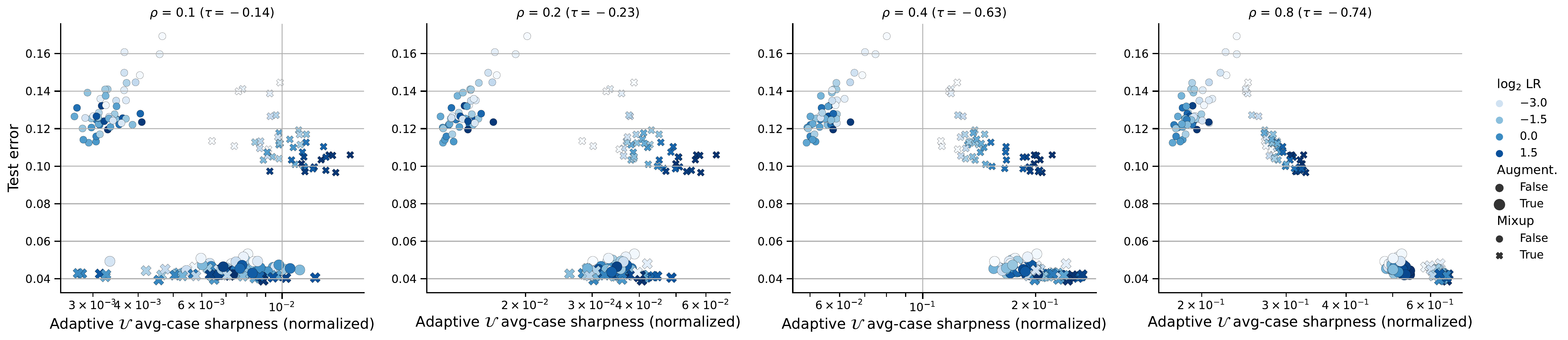}
    
    \vspace{3mm}
    \textbf{A side-by-side comparison of sharpness on original vs. augmented training data}\\
    \includegraphics[width=0.245\textwidth]{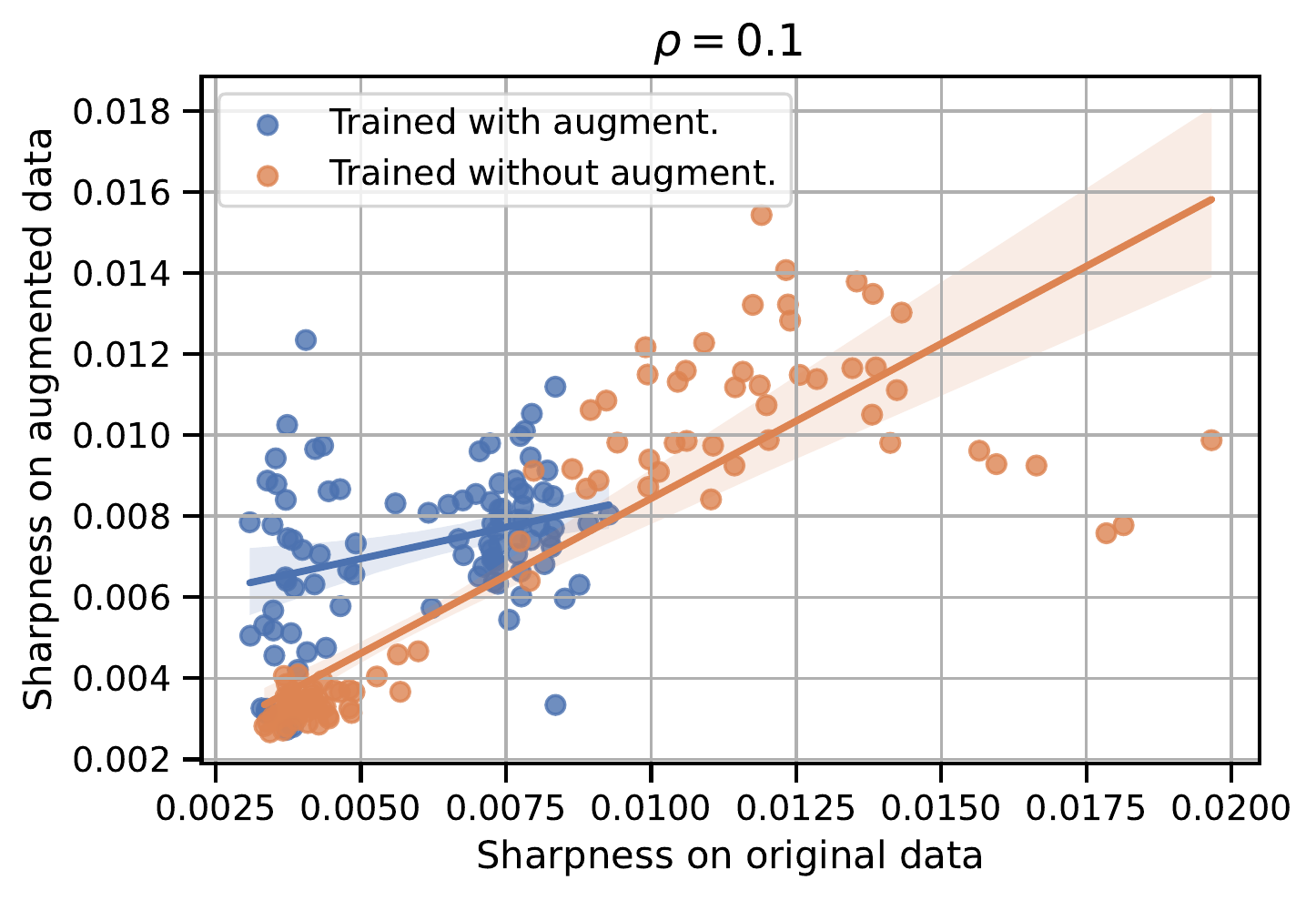}
    \includegraphics[width=0.245\textwidth]{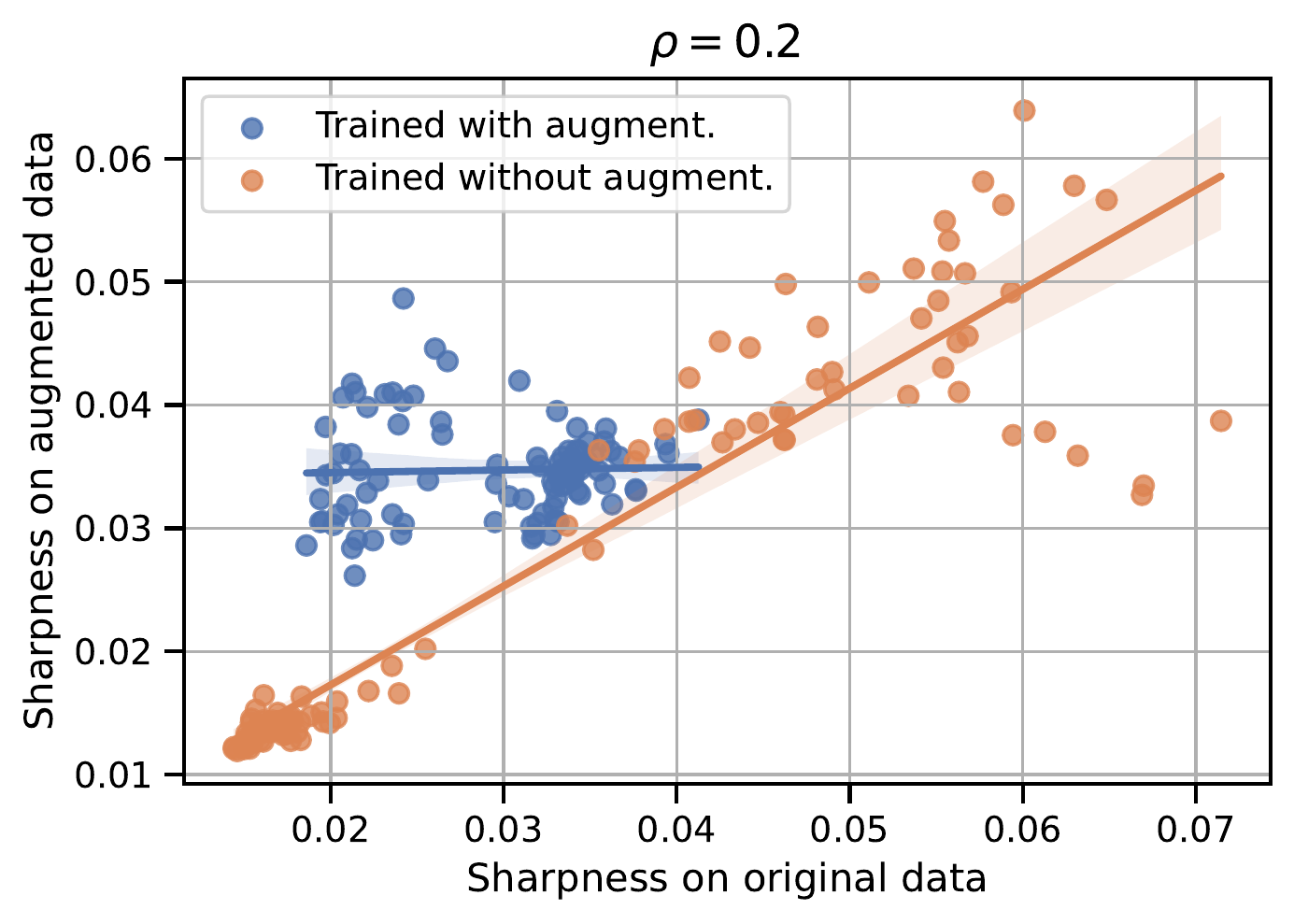}
    \includegraphics[width=0.245\textwidth]{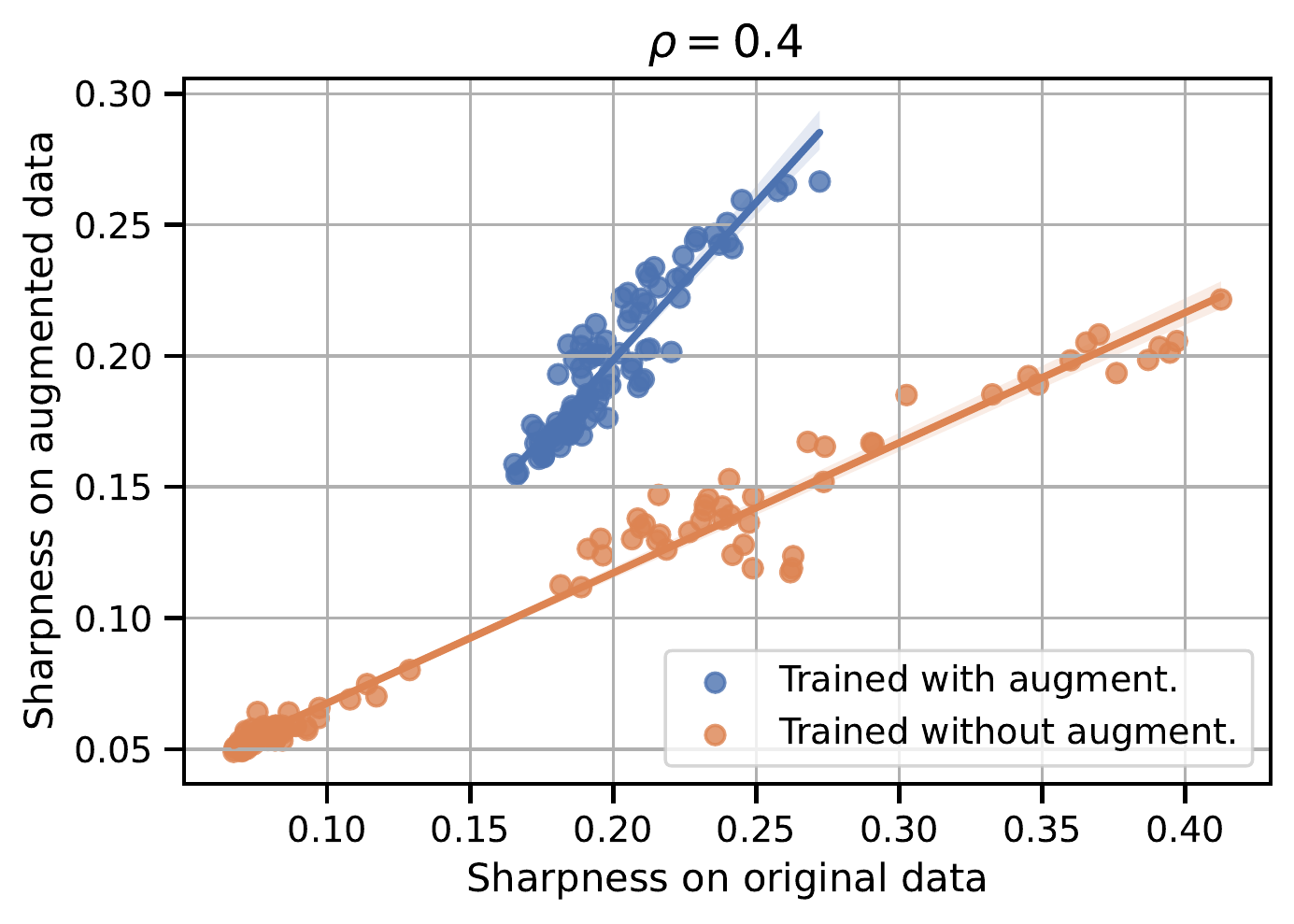}
    \includegraphics[width=0.245\textwidth]{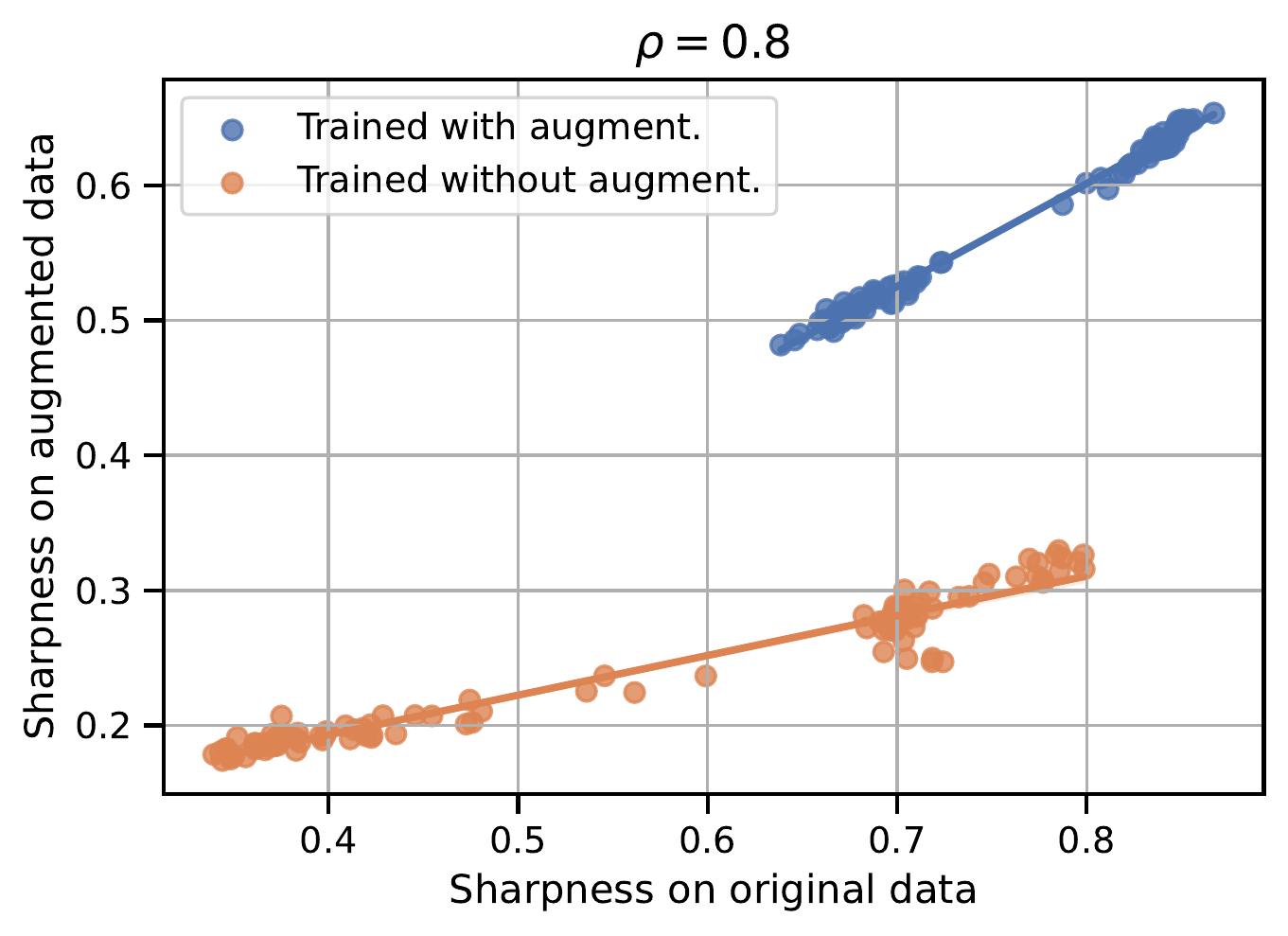}
    \caption{Adaptive average-case $\ell_\infty$ (uniform perturbations) sharpness (normalized) for ResNets-18 on \textit{original} vs. \textit{augmented} CIFAR-10 training data for ResNets-18 for different radii $\rho$.}
    \label{fig:cifar10_resnet_diff_data}
\end{figure*}

\clearpage

\subsection{The Role of the Number of Iterations in Auto-PGD}
\label{sec:app_cifar10_role_n_iter}
Here we aim to justify the choice of 20 iterations of Auto-PGD in our experiments. In Fig.~\ref{fig:cifar10_resnet_diff_n_iter}, we present results for adaptive worst-case $\ell_\infty$ sharpness (normalized) for ResNets-18 on CIFAR-10 for 20, 50, 100, and 200 iterations. We can see that the sharpness values are not visibly affected by increasing the number of iterations and the overall trend stays exactly the same.
\begin{figure*}[h]
    \centering
    \textbf{Adaptive worst-case $\ell_\infty$ sharpness (normalized) for ResNets-18 (20 iterations)}
    \includegraphics[width=0.95\textwidth]{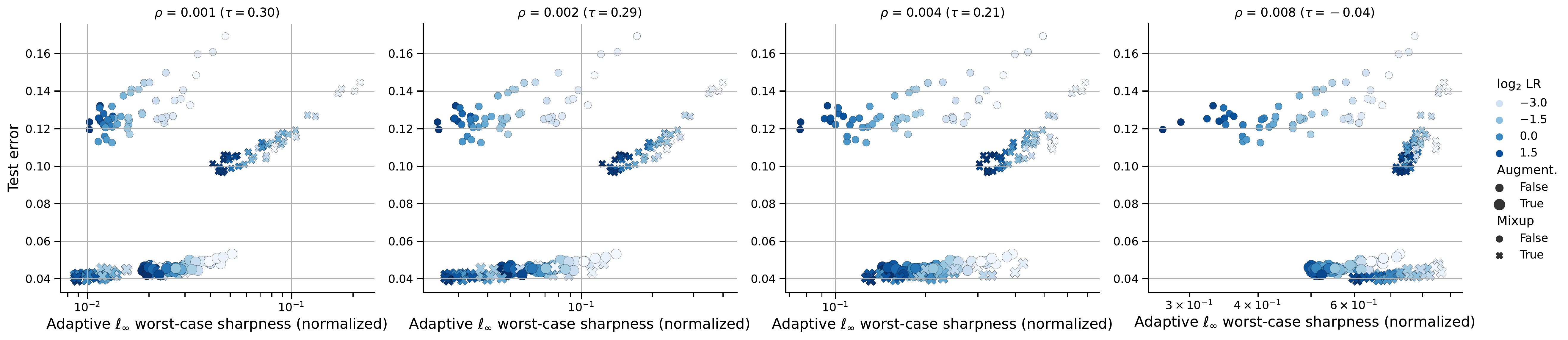}
    \textbf{Adaptive worst-case $\ell_\infty$ sharpness (normalized) for ResNets-18 (50 iterations)}
    \includegraphics[width=0.95\textwidth]{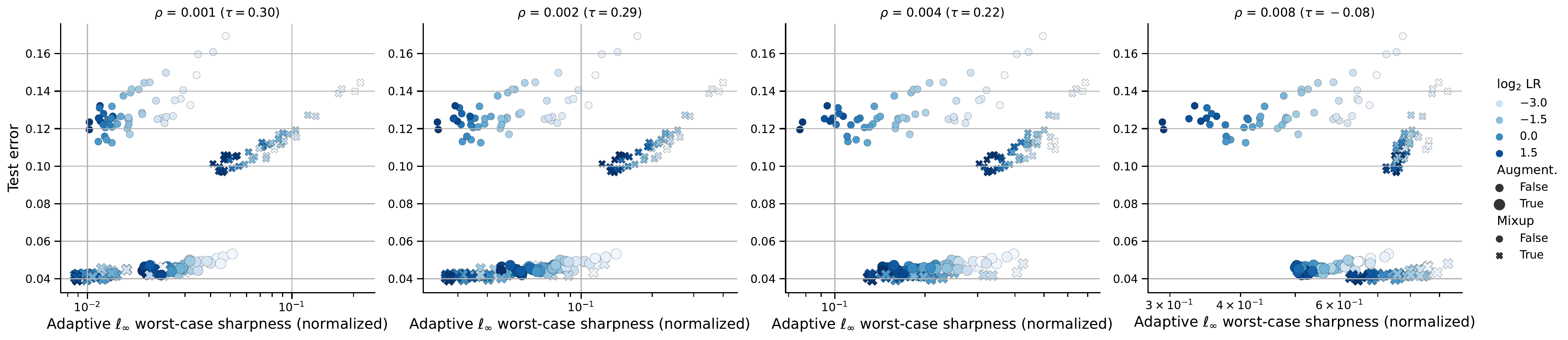}
    \textbf{Adaptive worst-case $\ell_\infty$ sharpness (normalized) for ResNets-18 (100 iterations)}
    \includegraphics[width=0.95\textwidth]{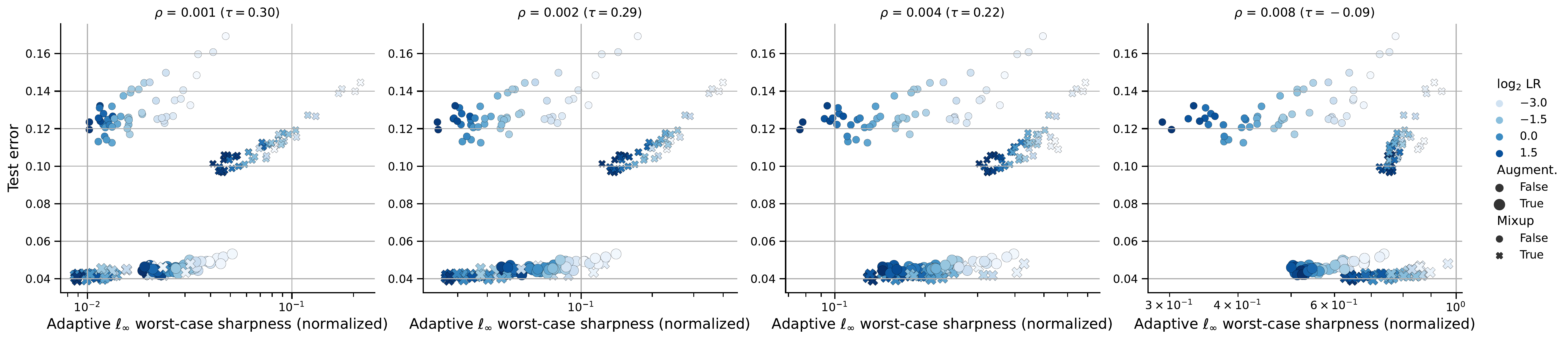}
    \textbf{Adaptive worst-case $\ell_\infty$ sharpness (normalized) for ResNets-18 (200 iterations)}
    \includegraphics[width=0.95\textwidth]{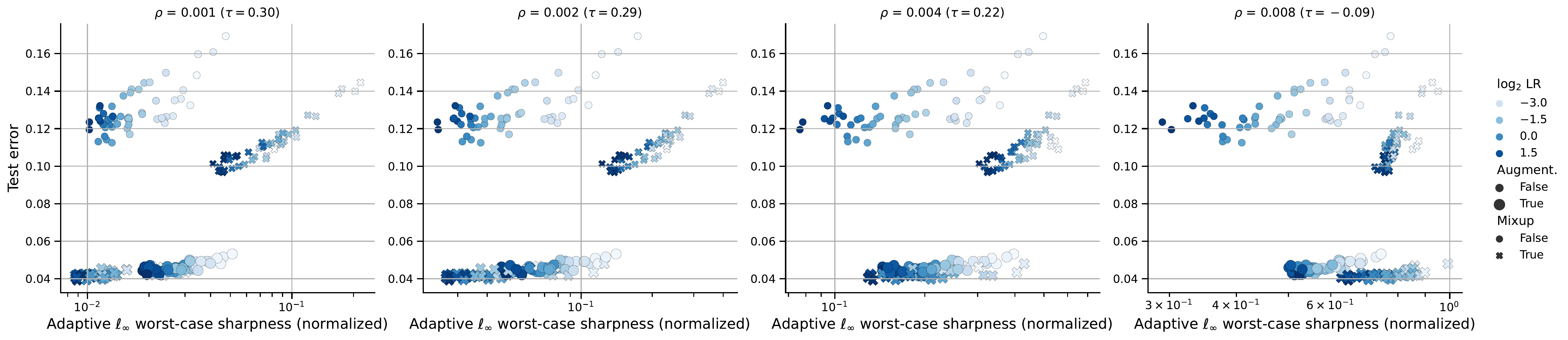}
    \caption{Adaptive worst-case $\ell_\infty$ sharpness (normalized) \textbf{for different number of iterations in Auto-PGD} vs. test error on CIFAR-10 for ResNets-18 for different radii $\rho$.}
    \label{fig:cifar10_resnet_diff_n_iter}
\end{figure*}
\clearpage

\subsection{The Role of m in m-Sharpness}
\label{sec:app_cifar10_role_m}
\citet{foret2021sharpnessaware} suggested that a lower $m$ in $m$-sharpness, i.e., the batch size used for maximizing sharpness, can lead to a higher correlation with generalization in some settings. We note that we have already used a small $m$ for all our experiments ($m=128$ on CIFAR-10 and $m=256$ on ImageNet and MNLI), but here we check additionally whether even smaller $m$ change the trend. Fig.~\ref{fig:cifar10_resnet_diff_m} shows the results sharpness for adaptive worst-case $\ell_\infty$ sharpness (normalized) for ResNets-18 and ViTs on CIFAR-10 for $m \in \{16, 32, 64, 128\}$. We can see that different $m$ only slightly affects the sharpness values and the overall trend stays unaffected.
\begin{figure*}[h]
    \centering
    \textbf{Adaptive worst-case $\ell_\infty$ sharpness (normalized) for ResNets-18 ($\bold{m=16}$)}
    \includegraphics[width=0.95\textwidth]{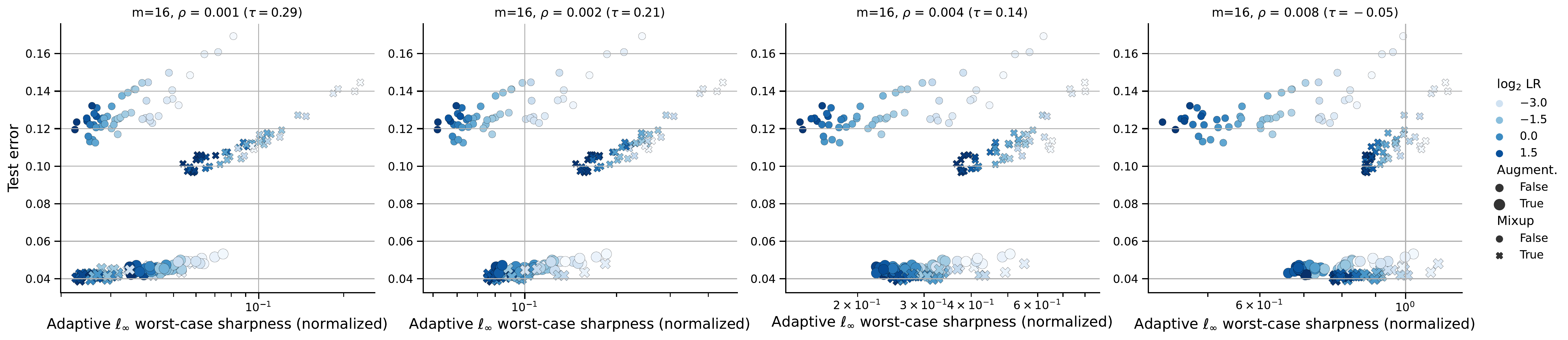}
    \textbf{Adaptive worst-case $\ell_\infty$ sharpness (normalized) for ResNets-18 ($\bold{m=32}$)}
    \includegraphics[width=0.95\textwidth]{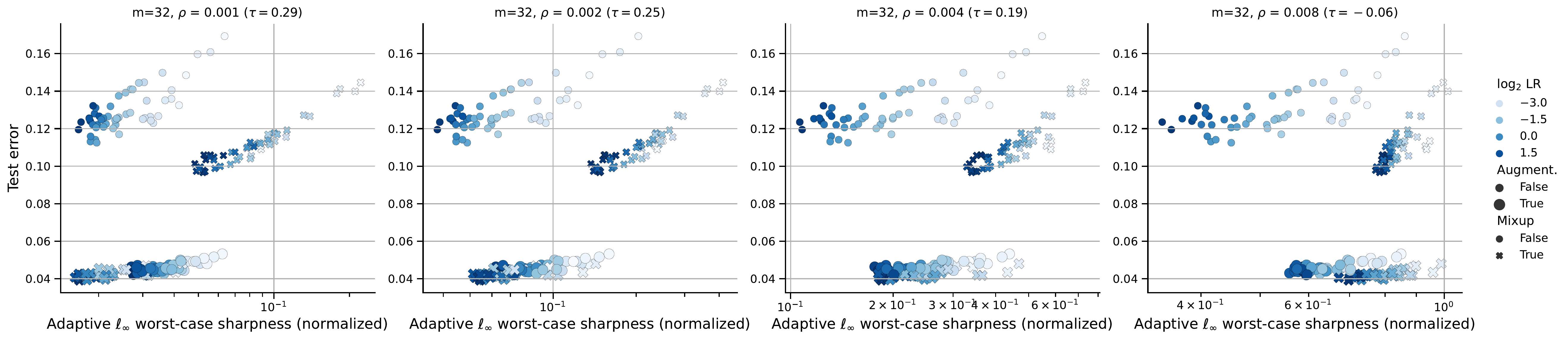}
    \textbf{Adaptive worst-case $\ell_\infty$ sharpness (normalized) for ResNets-18 ($\bold{m=64}$)}
    \includegraphics[width=0.95\textwidth]{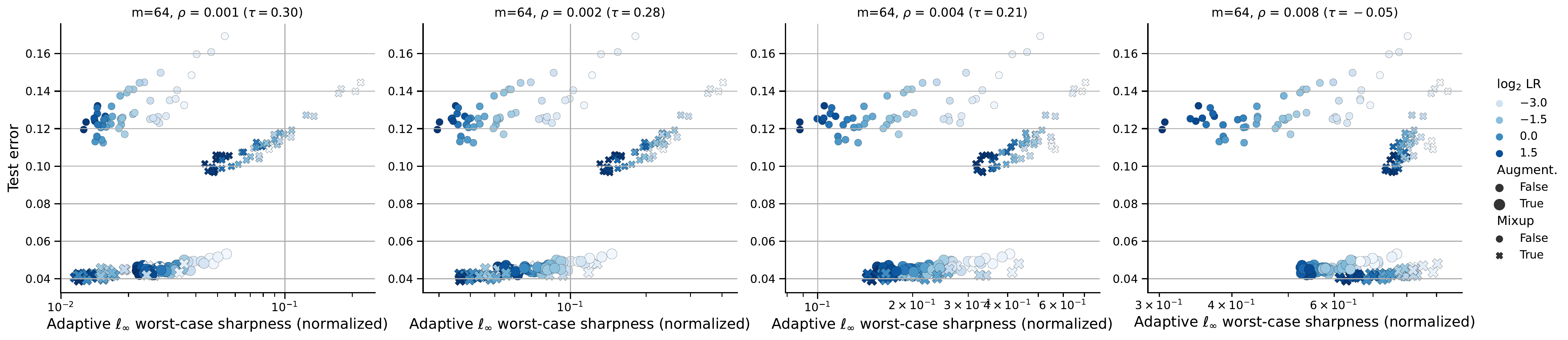}
    \textbf{Adaptive worst-case $\ell_\infty$ sharpness (normalized) for ResNets-18 ($\bold{m=128}$)}
    \includegraphics[width=0.95\textwidth]{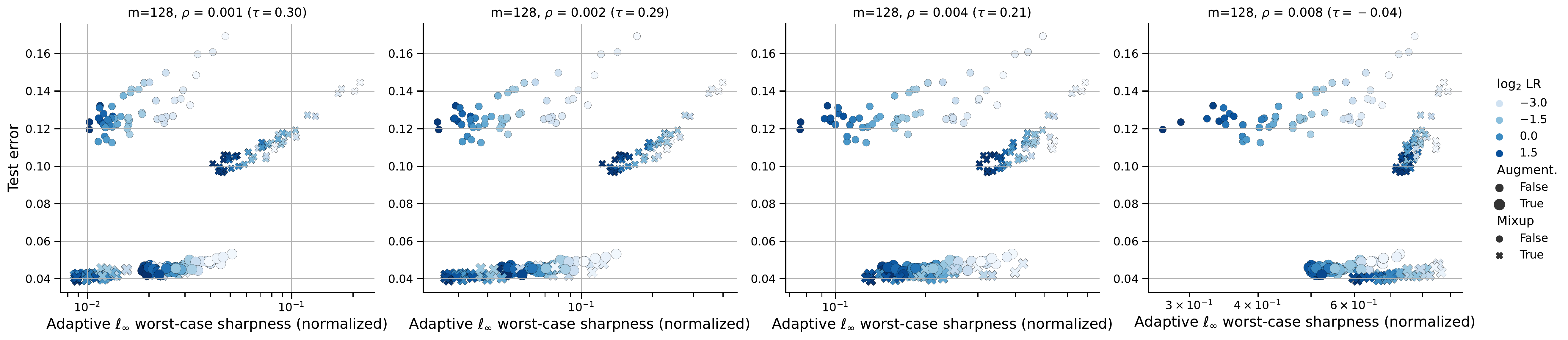}
    \caption{Adaptive worst-case $\ell_\infty$ sharpness (normalized) \textbf{for different $\boldsymbol{m}$} in $m$-sharpness vs. test error on CIFAR-10 for ResNets-18 for different radii $\rho$.}
    \label{fig:cifar10_resnet_diff_m}
\end{figure*}

\begin{figure*}[h]
    \centering
    \textbf{Adaptive worst-case $\ell_\infty$ sharpness (normalized) for ViTs ($\bold{m=16}$)}
    \includegraphics[width=0.95\textwidth]{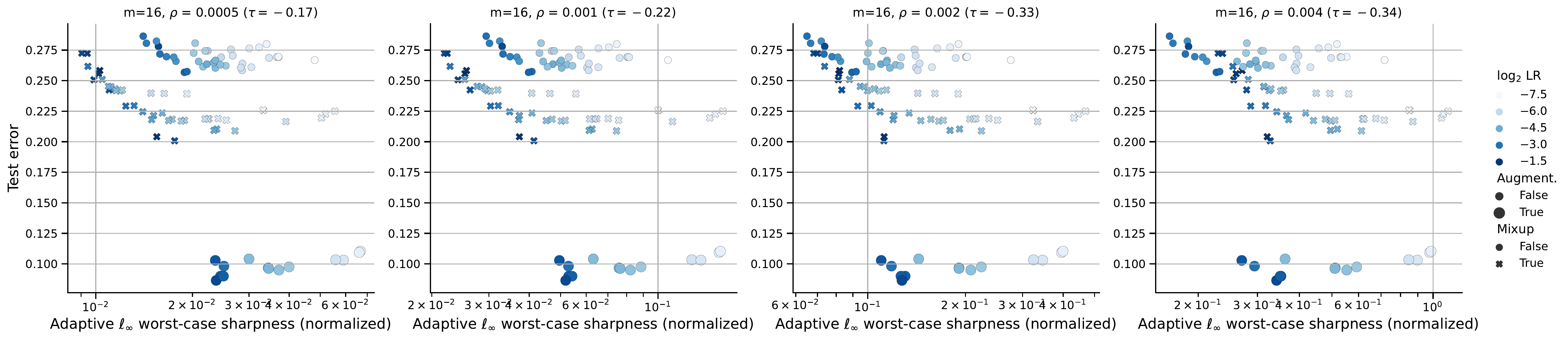}
    \textbf{Adaptive worst-case $\ell_\infty$ sharpness (normalized) for ViTs ($\bold{m=32}$)}
    \includegraphics[width=0.95\textwidth]{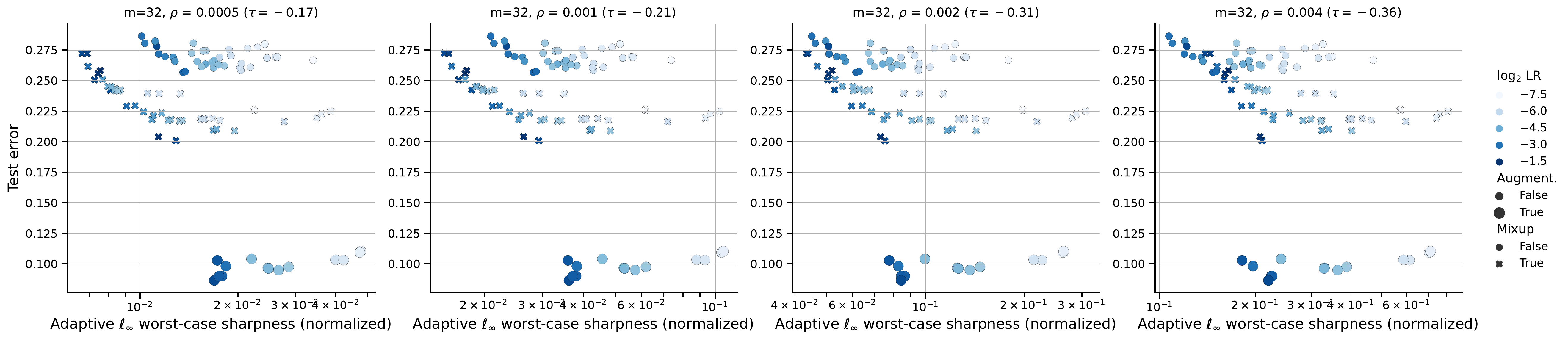}
    \textbf{Adaptive worst-case $\ell_\infty$ sharpness (normalized) for ViTs ($\bold{m=64}$)}
    \includegraphics[width=0.95\textwidth]{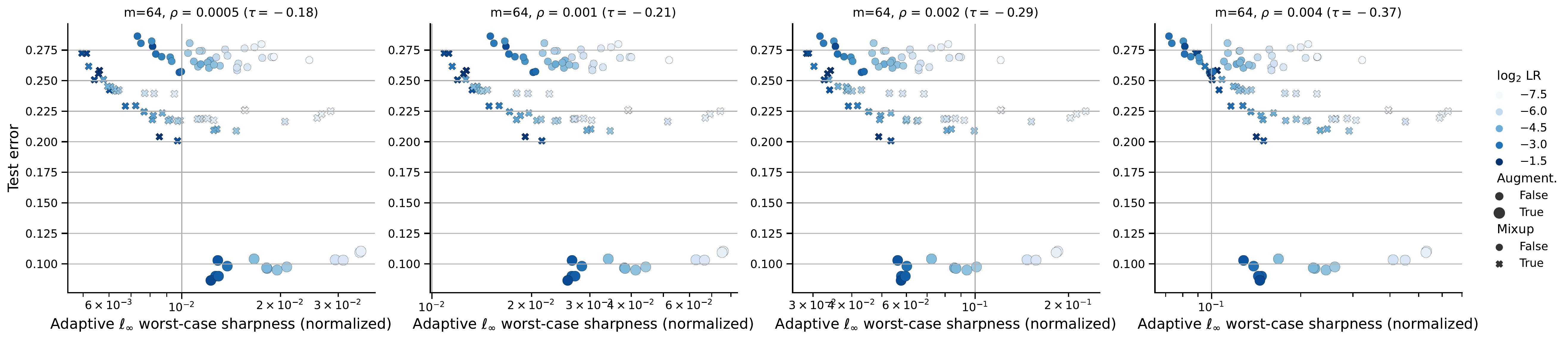}
    \textbf{Adaptive worst-case $\ell_\infty$ sharpness (normalized) for ViTs ($\bold{m=128}$)}
    \includegraphics[width=0.95\textwidth]{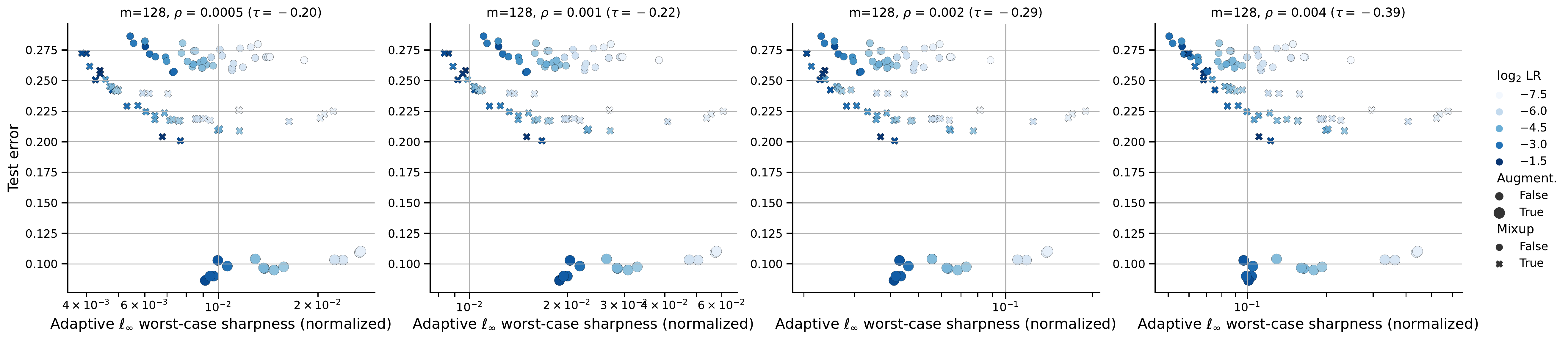}
    \caption{Adaptive worst-case $\ell_\infty$ sharpness (normalized) \textbf{for different $\boldsymbol{m}$} in $m$-sharpness vs. test error on CIFAR-10 for ViTs for different radii $\rho$.}
    \label{fig:cifar10_vit_diff_m}
\end{figure*}

\clearpage

\subsection{The Role of Different Sharpness Definitions and Radii}
\label{sec:app_cifar10_role_defns_radii} 
Here we present results for $12$ different sharpness definitions: 
\begin{itemize}
    \item standard average-case $\ell_2$ (Gaussian perturbations) sharpness without logit normalization,
    \item standard worst-case $\ell_2$ sharpness without logit normalization,
    \item adaptive average-case $\ell_2$ (Gaussian perturbations) sharpness without logit normalization,
    \item adaptive worst-case $\ell_2$ sharpness without logit normalization,
    \item adaptive average-case $\ell_2$ (Gaussian perturbations) sharpness with logit normalization,
    \item adaptive worst-case $\ell_2$ sharpness with logit normalization,
    \item standard average-case $\ell_\infty$ (uniform perturbations) sharpness without logit normalization,
    \item standard worst-case $\ell_\infty$ sharpness without logit normalization,
    \item adaptive average-case $\ell_\infty$ (uniform perturbations) sharpness without logit normalization,
    \item adaptive worst-case $\ell_\infty$ sharpness without logit normalization (\textit{shown in the main part for a single $\rho$}),
    \item adaptive average-case $\ell_\infty$ (uniform perturbations) sharpness with logit normalization,
    \item adaptive worst-case $\ell_\infty$ sharpness with logit normalization (\textit{shown in the main part for a single $\rho$}).
\end{itemize}
We evaluate a wide range of radii for each sharpness definition to make sure that we do not miss the right scale of sharpness. We present results first for ResNets and then for ViTs.

\myparagraph{Observations for ResNets.}
For ResNets, we observe that many sharpness definitions can successfully capture correlation with standard generalization within each subgroup defined by the values of \texttt{augment} $\times$ \texttt{mixup}. In particular, on average, \textit{adaptive} sharpness shows a better correlation with generalization within each subgroup, and the best correlation within each subgroup is achieved by $\ell_\infty$ adaptive worst-case sharpness with logit normalization for a small $\rho$. In many cases, the correlation of sharpness with OOD generalization on CIFAR-10-C is noticeably lower compared to the correlation of sharpness with standard generalization. Overall, we see that there is no coherent global trend of correlation with generalization that would apply to all models at once.
We also observe that for some sharpness definitions, the flattest models generalize best (for adaptive worst-case $\ell_2$ sharpness with normalization for the smallest $\rho$ and for adaptive worst-case $\ell_\infty$ sharpness without normalization for the largest $\rho$) but this appears to be unsystematic and there exist nearly equally flat solutions that generalize much worse.

\myparagraph{Observations for ViTs.}
For ViTs, in contrast to ResNets, we do not observe a consistent correlation with generalization even within subgroups. The only exception is the subgroup of points with augmentations where multiple definitions of sharpness tend to correlate with generalization and capture the effect of larger learning rate. We think it is likely due to the fact that with heavy augmentations optimizing the training objective to smaller values is helpful for generalization, while without augmentations all runs have converged within $200$ epochs and the learning rate plays no visible role for generalization there. Globally, when taken over all models, the correlation with standard generalization is close to $0$ and tends to slightly decrease when we measure OOD generalization on CIFAR-10-C. Finally, we note that there are no cases where the flattest ViT models achieve the best generalization. Thus, even our weak hypothesis about the role of sharpness is not confirmed here.

\begin{figure*}[t]
    \centering
    \textbf{Standard average-case $\ell_2$ (Gaussian perturbations) sharpness (unnormalized) for ResNets-18}
    \includegraphics[width=1.0\textwidth]{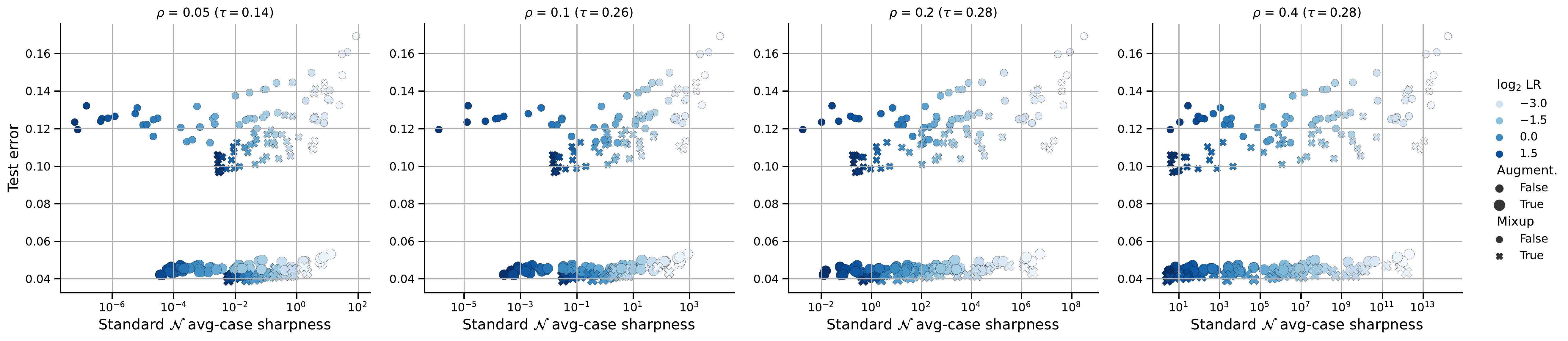}
    \includegraphics[width=1.0\textwidth]{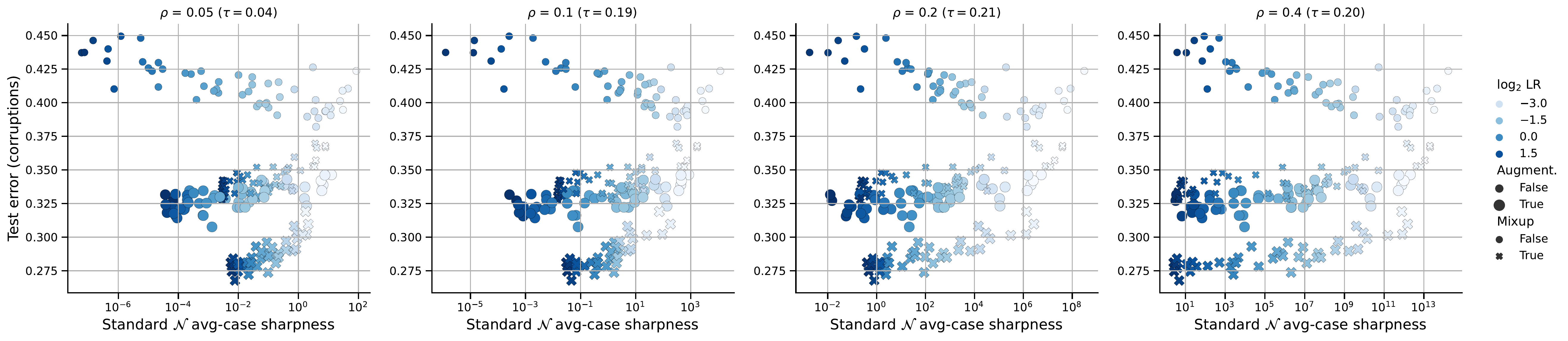}
    \textbf{Standard worst-case $\ell_2$ sharpness (unnormalized) for ResNets-18}
    \includegraphics[width=1.0\textwidth]{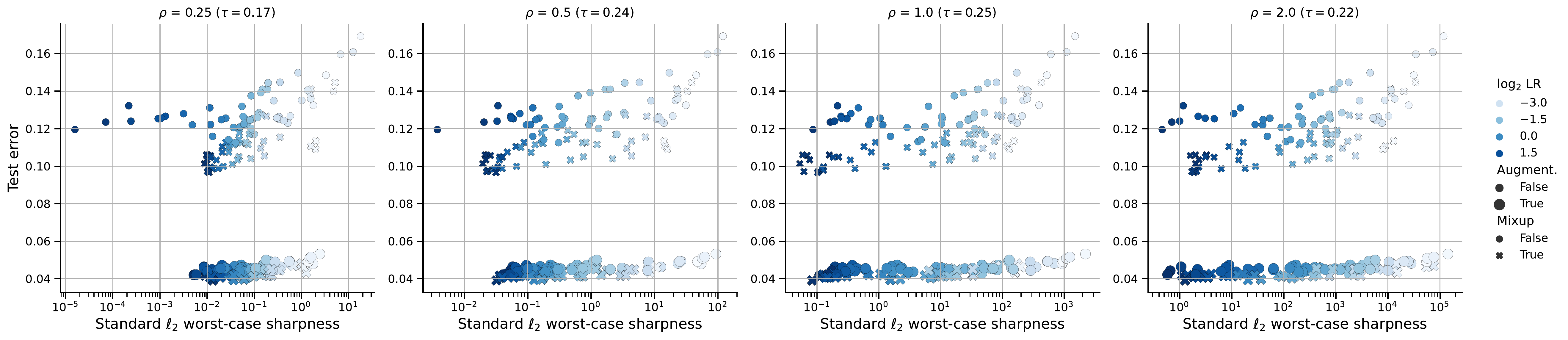}
    \includegraphics[width=1.0\textwidth]{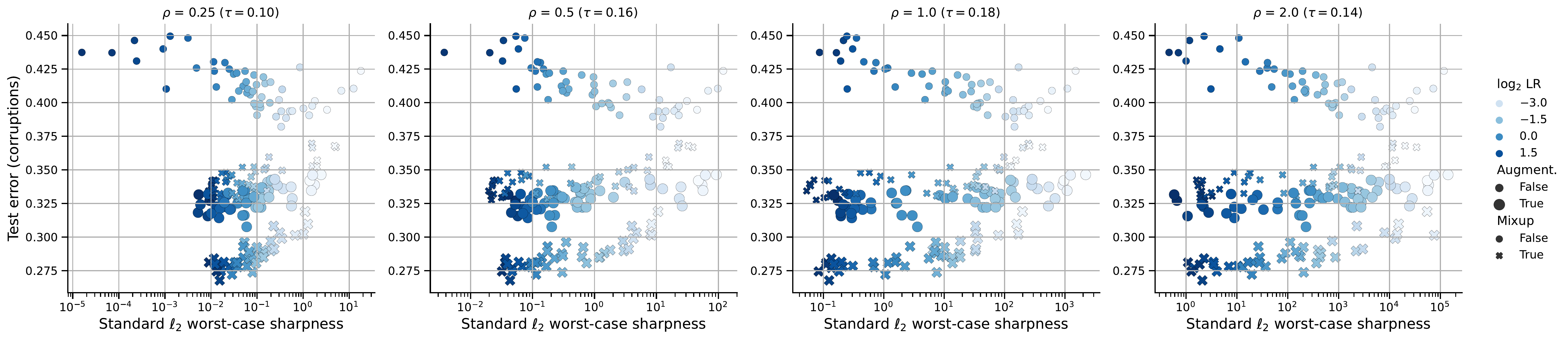}
    \caption{Average and worst-case $\ell_2$ standard sharpness definitions (unnormalized) vs. test error and OOD test error (common corruptions) on CIFAR-10 for ResNets-18 for different radii $\rho$.}
    \label{fig:cifar10_resnet_standard_l2_sharpness}
\end{figure*}
\begin{figure*}[t!]
    \centering
    \textbf{Adaptive average-case $\ell_2$ (Gaussian perturbations) sharpness (unnormalized) for ResNets-18}
    \includegraphics[width=1.0\textwidth]{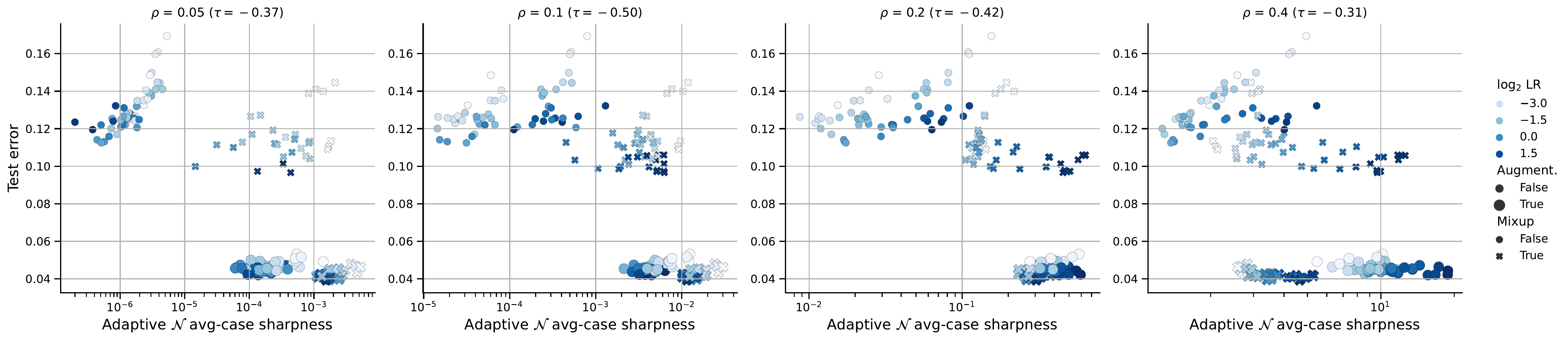}
    \includegraphics[width=1.0\textwidth]{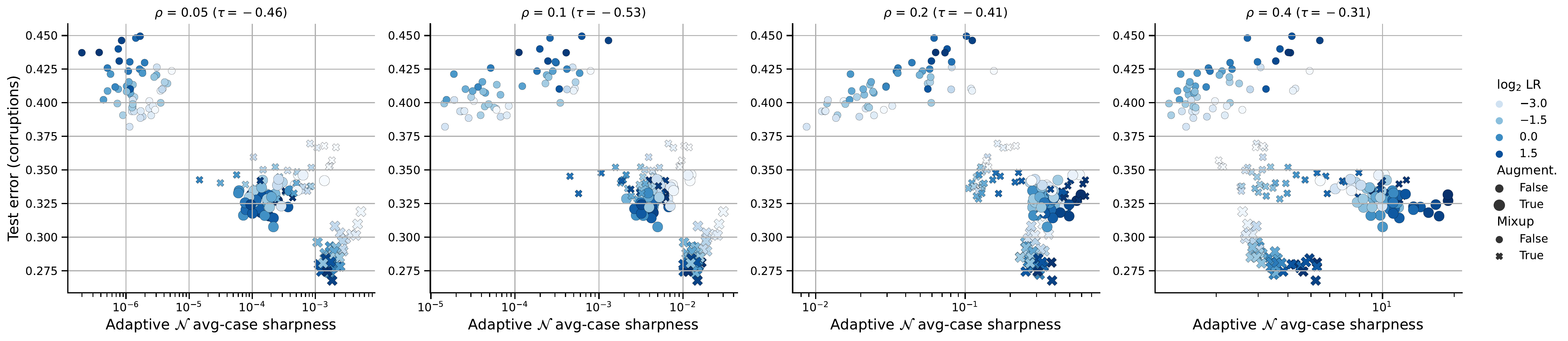}
    \textbf{Adaptive worst-case $\ell_2$ sharpness (unnormalized) for ResNets-18}
    \includegraphics[width=1.0\textwidth]{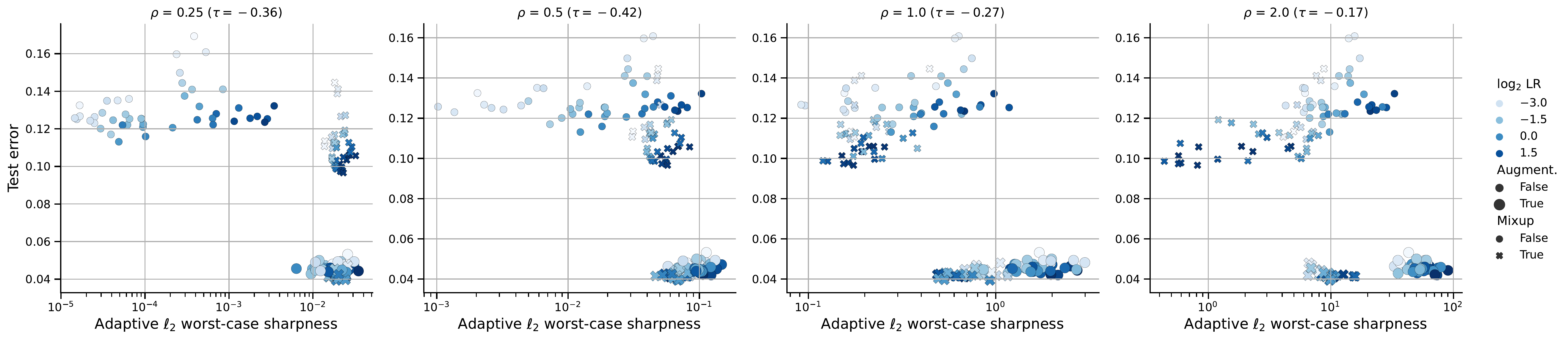}
    \includegraphics[width=1.0\textwidth]{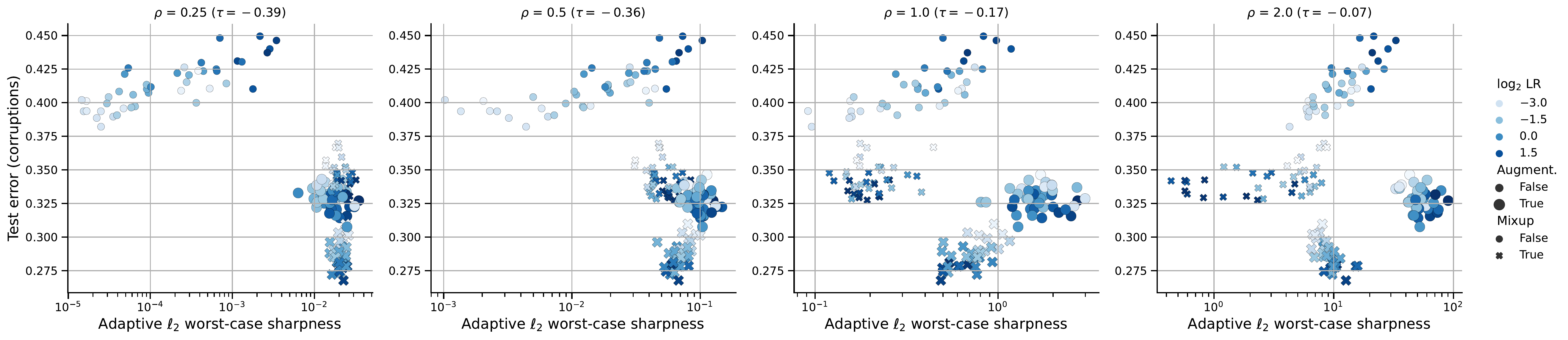}
    \caption{Average and worst-case $\ell_2$ adaptive sharpness definitions (unnormalized) vs. test error and OOD test error (common corruptions) on CIFAR-10 for ResNets-18 for different radii $\rho$.}
    \label{fig:cifar10_resnet_adaptive_l2_sharpness}
\end{figure*}
\begin{figure*}[t!]
    \centering
    \textbf{Adaptive average-case $\ell_2$ (Gaussian perturbations) sharpness (normalized) for ResNets-18}
    \includegraphics[width=1.0\textwidth]{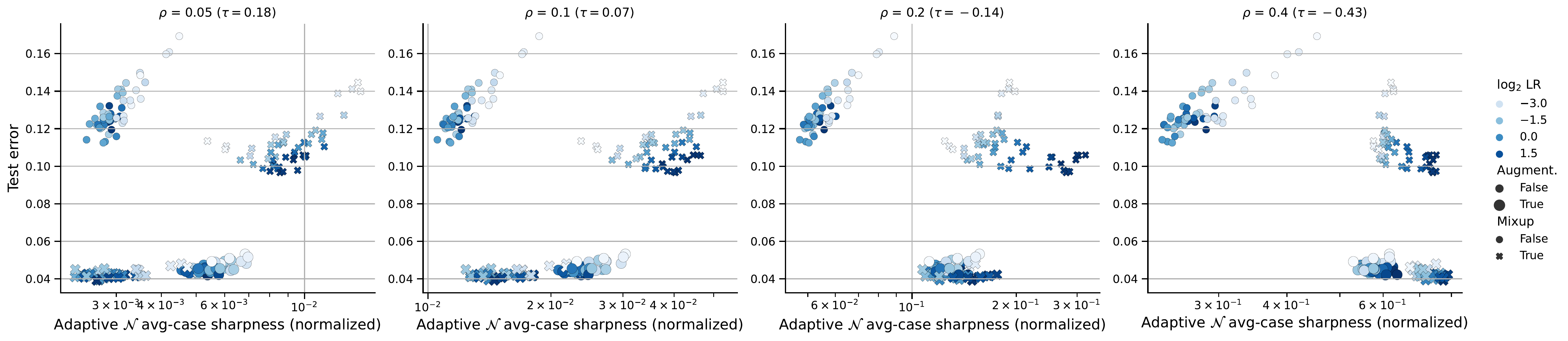}
    \includegraphics[width=1.0\textwidth]{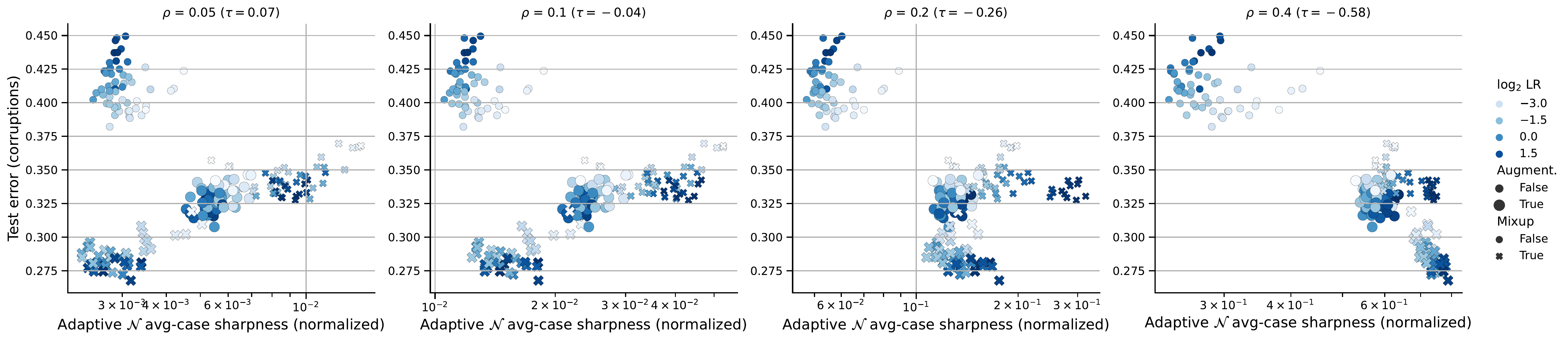}
    \textbf{Adaptive worst-case $\ell_2$ sharpness (normalized) for ResNets-18}
    \includegraphics[width=1.0\textwidth]{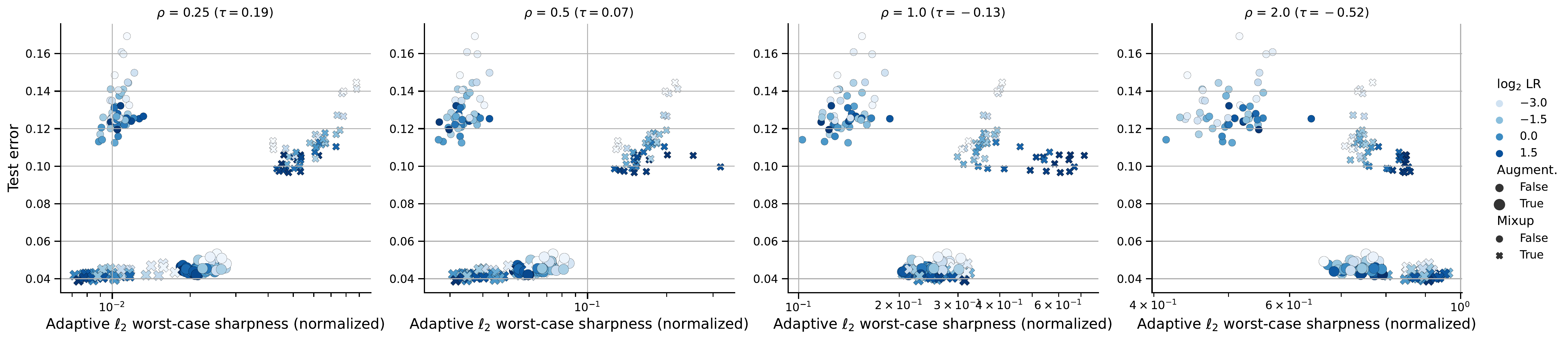}
    \includegraphics[width=1.0\textwidth]{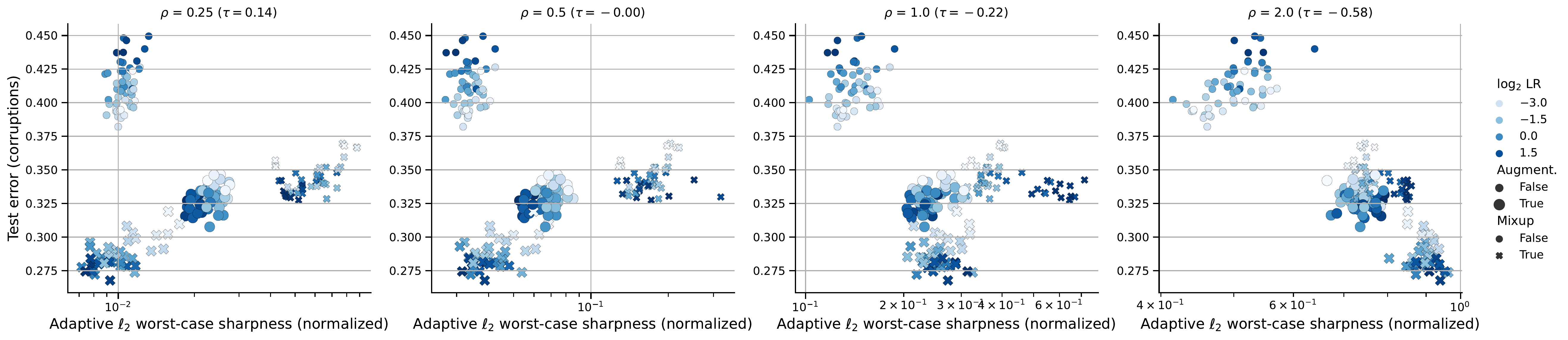}
    \caption{Average and worst-case $\ell_2$ adaptive sharpness definitions (normalized) vs. test error and OOD test error (common corruptions) on CIFAR-10 for ResNets-18 for different radii $\rho$.}
    \label{fig:cifar10_resnet_adaptive_l2_sharpness_normalized}
\end{figure*}

\begin{figure*}[t!]
    \centering
    \textbf{Standard average-case $\ell_\infty$ (uniform perturbations) sharpness (unnormalized) for ResNets-18}
    \includegraphics[width=1.0\textwidth]{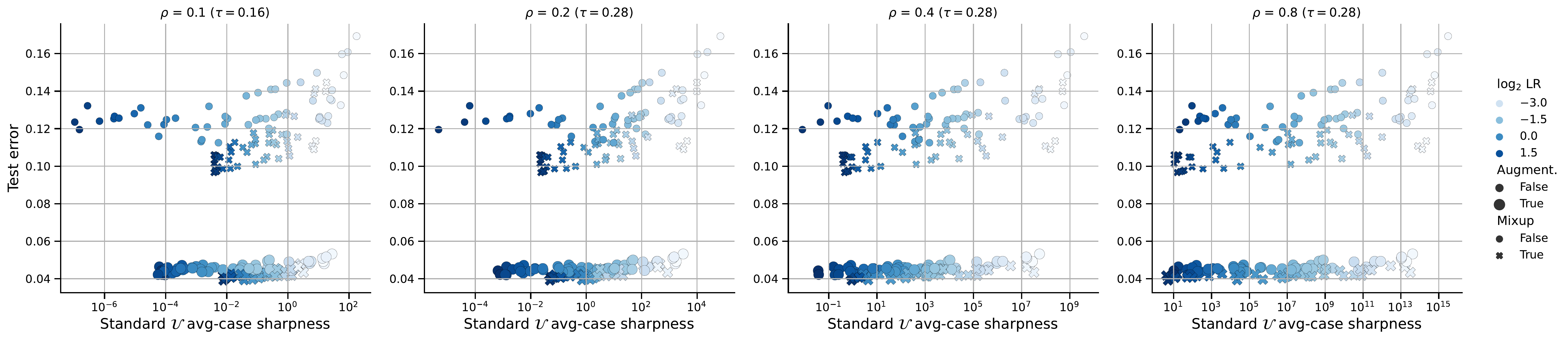}
    \includegraphics[width=1.0\textwidth]{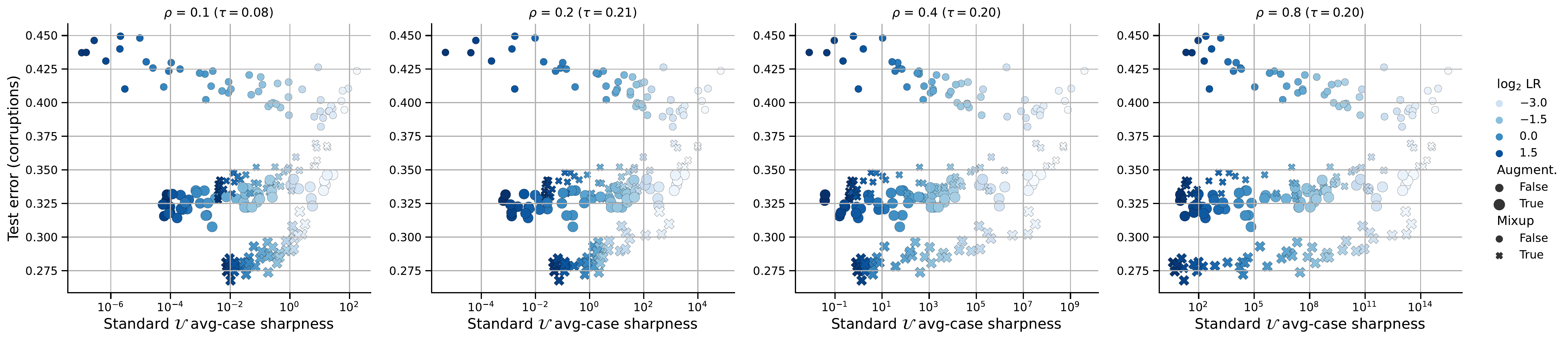}
    \textbf{Standard worst-case $\ell_\infty$ sharpness (unnormalized) for ResNets-18}
    \includegraphics[width=1.0\textwidth]{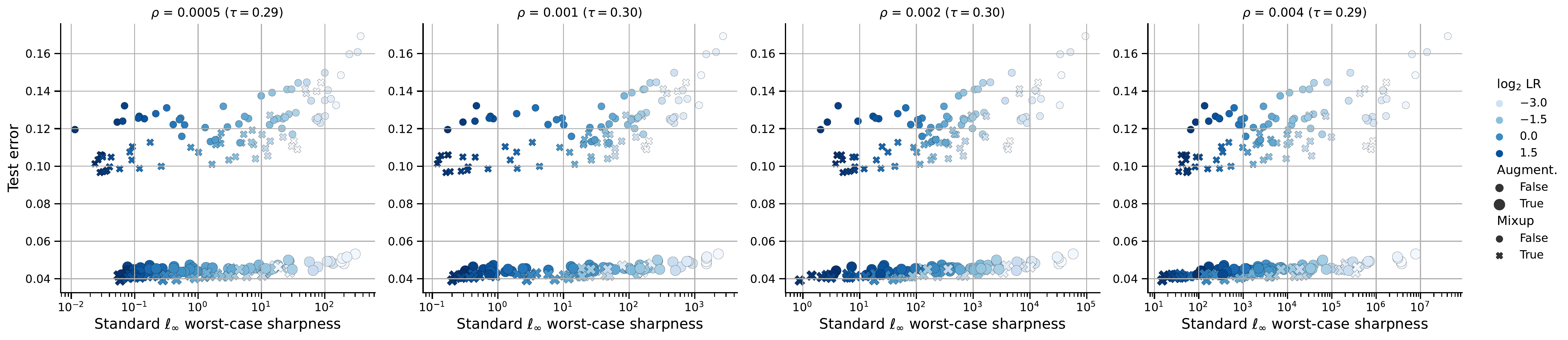}
    \includegraphics[width=1.0\textwidth]{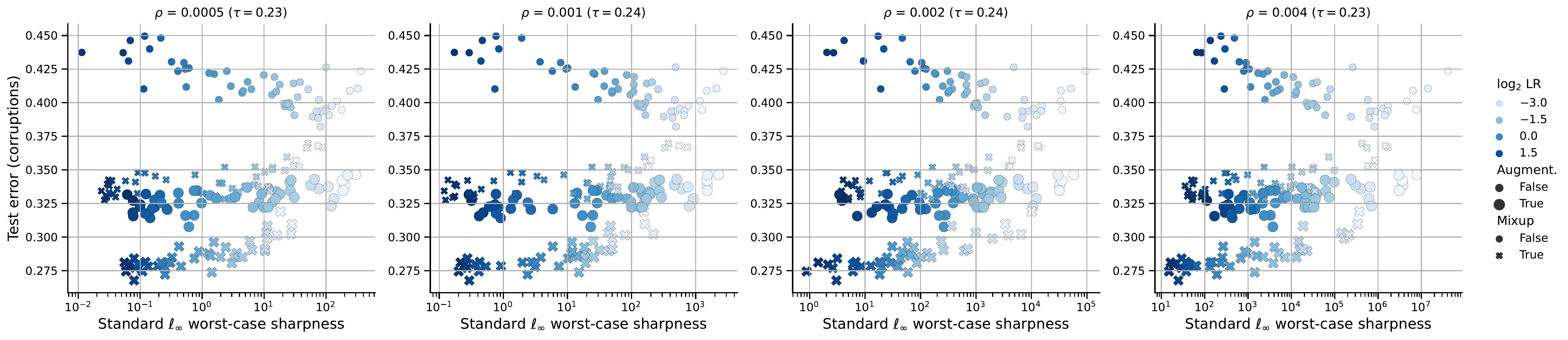}
    \caption{Average and worst-case $\ell_\infty$ standard sharpness definitions (unnormalized) vs. test error and OOD test error (common corruptions) on CIFAR-10 for ResNets-18 for different radii $\rho$.}
    \label{fig:cifar10_resnet_standard_linf_sharpness}
\end{figure*}
\begin{figure*}[t!]
    \centering
    \textbf{Adaptive average-case $\ell_\infty$ (uniform perturbations) sharpness (unnormalized) for ResNets-18}
    \includegraphics[width=1.0\textwidth]{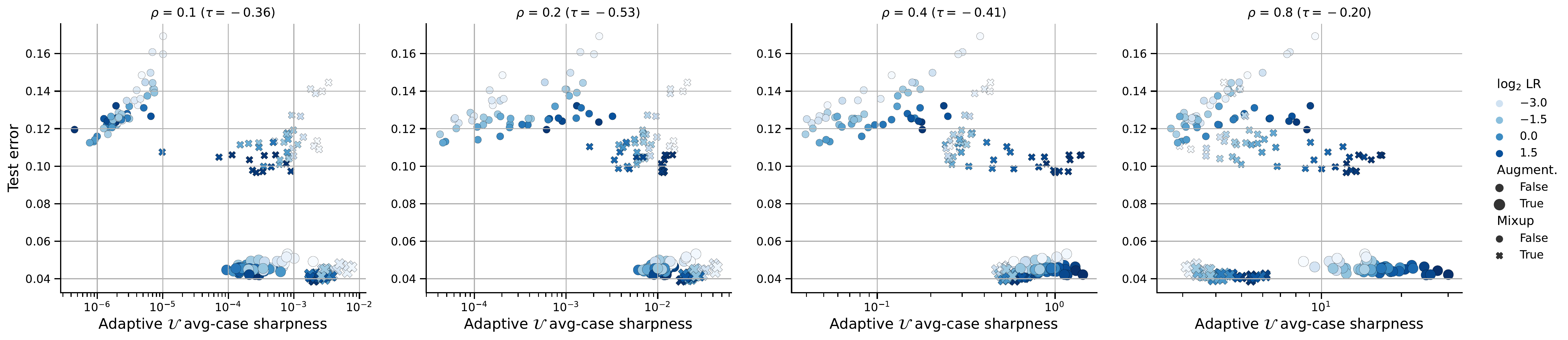}
    \includegraphics[width=1.0\textwidth]{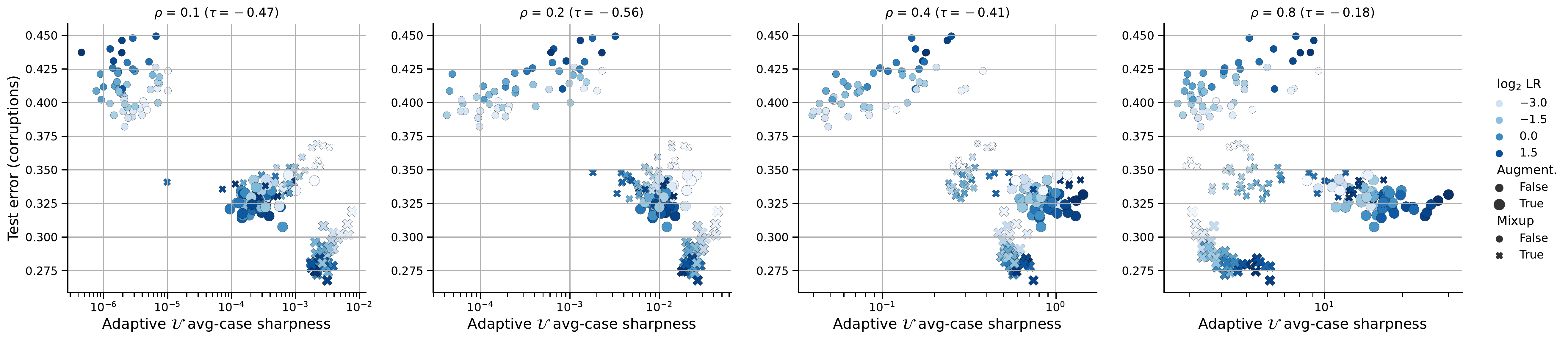}
    \textbf{Adaptive worst-case $\ell_\infty$ sharpness (unnormalized) for ResNets-18}
    \includegraphics[width=1.0\textwidth]{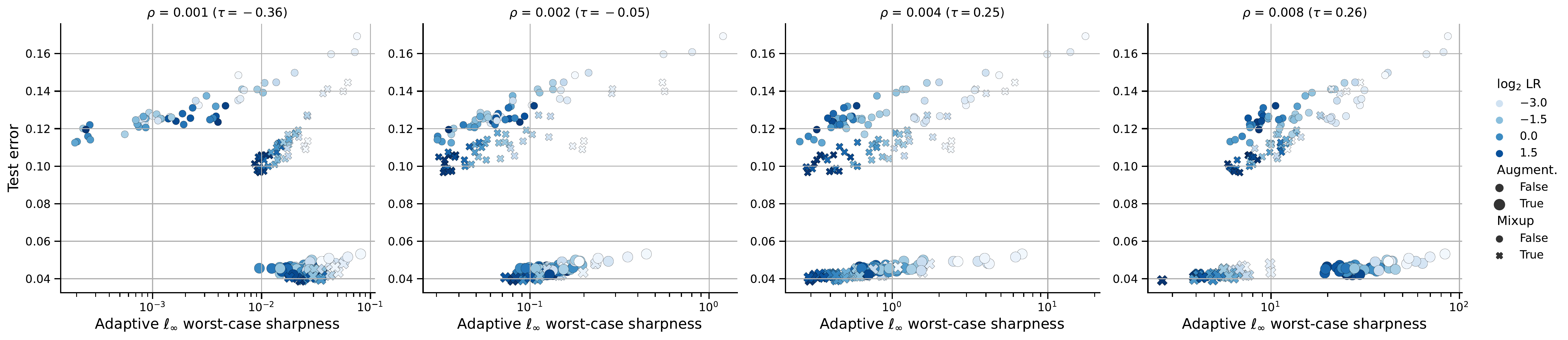}
    \includegraphics[width=1.0\textwidth]{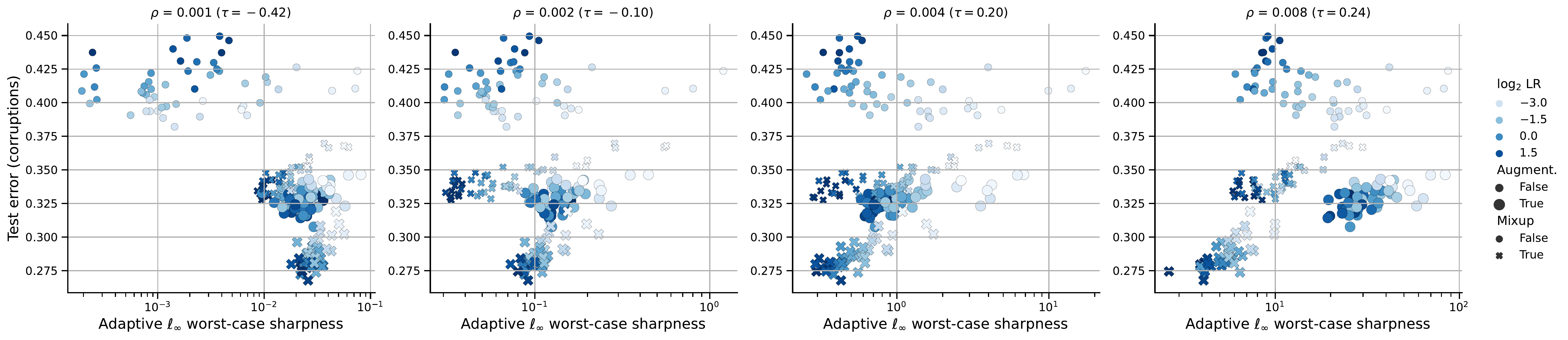}
    \caption{Average and worst-case $\ell_\infty$ adaptive sharpness definitions (unnormalized) vs. test error and OOD test error (common corruptions) on CIFAR-10 for ResNets-18 for different radii $\rho$.}
    \label{fig:cifar10_resnet_adaptive_linf_sharpness}
\end{figure*}
\begin{figure*}[t!]
    \centering
    \textbf{Adaptive average-case $\ell_\infty$ (uniform perturbations) sharpness (normalized) for ResNets-18}
    \includegraphics[width=1.0\textwidth]{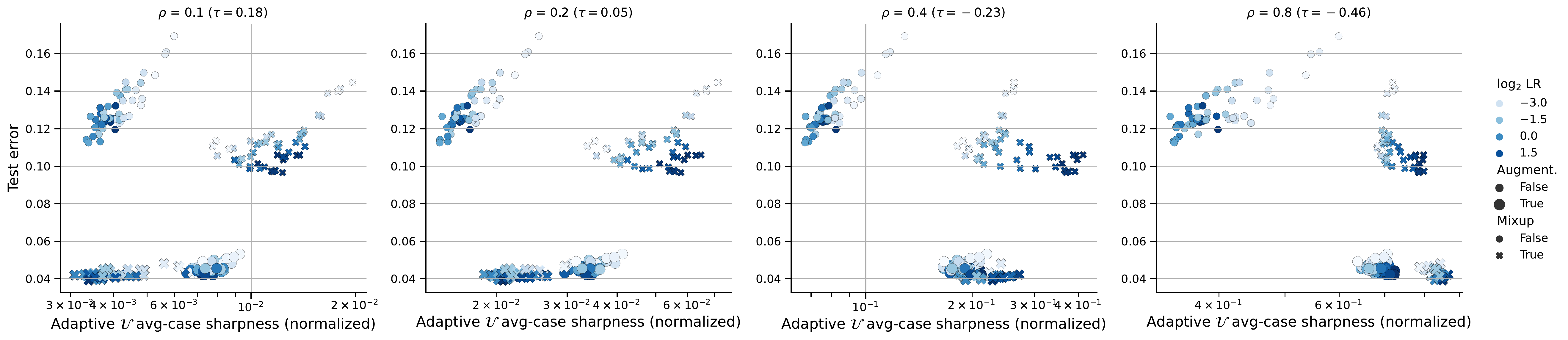}
    \includegraphics[width=1.0\textwidth]{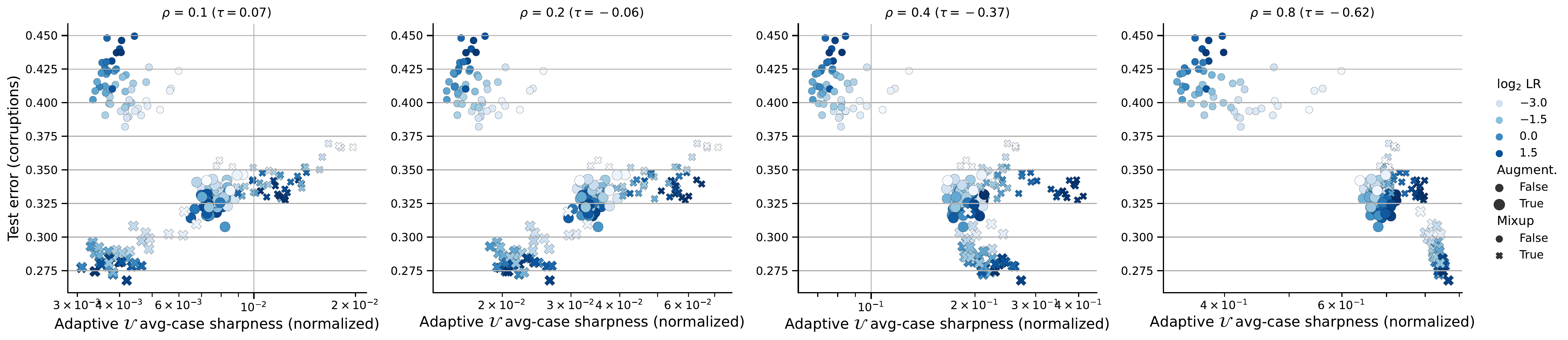}
    \textbf{Adaptive worst-case $\ell_\infty$ sharpness (normalized) for ResNets-18}
    \includegraphics[width=1.0\textwidth]{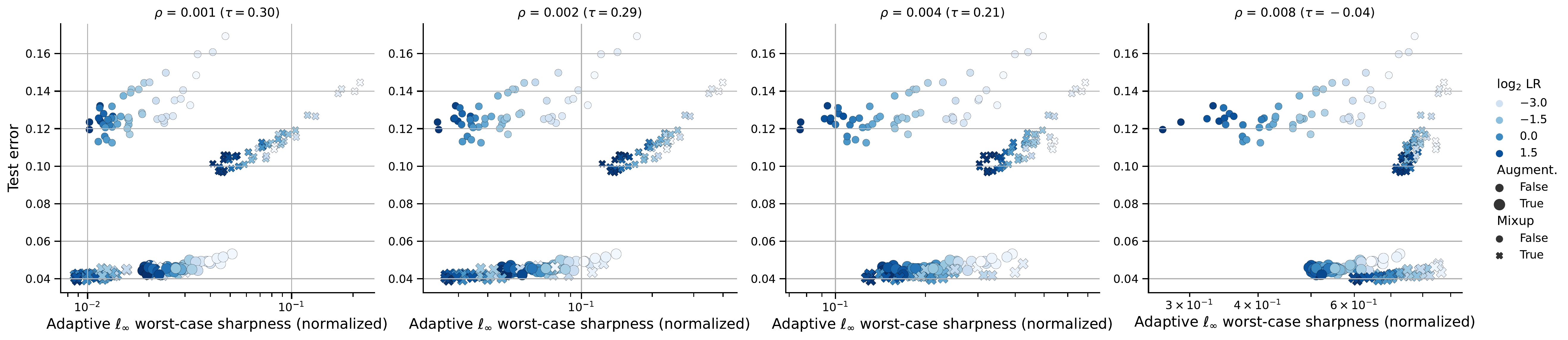}
    \includegraphics[width=1.0\textwidth]{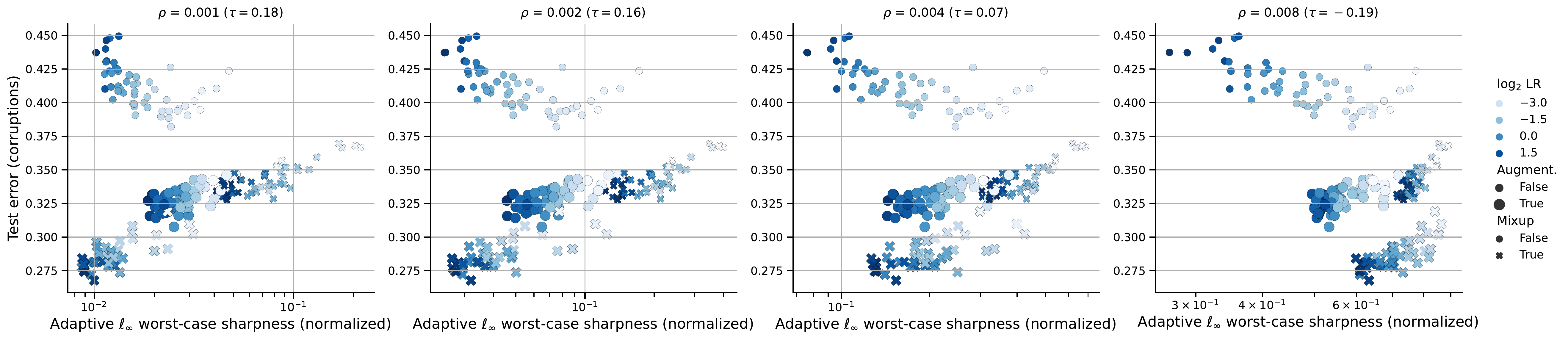}
    \caption{Average and worst-case $\ell_\infty$ adaptive sharpness definitions (normalized) vs. test error and OOD test error (common corruptions) on CIFAR-10 for ResNets-18 for different radii $\rho$.}
    \label{fig:cifar10_resnet_adaptive_linf_sharpness_normalized}
\end{figure*}

\begin{figure*}[t!]
    \centering
    \textbf{Standard average-case $\ell_2$ (Gaussian perturbations) sharpness (unnormalized) for ViTs}
    \includegraphics[width=1.0\textwidth]{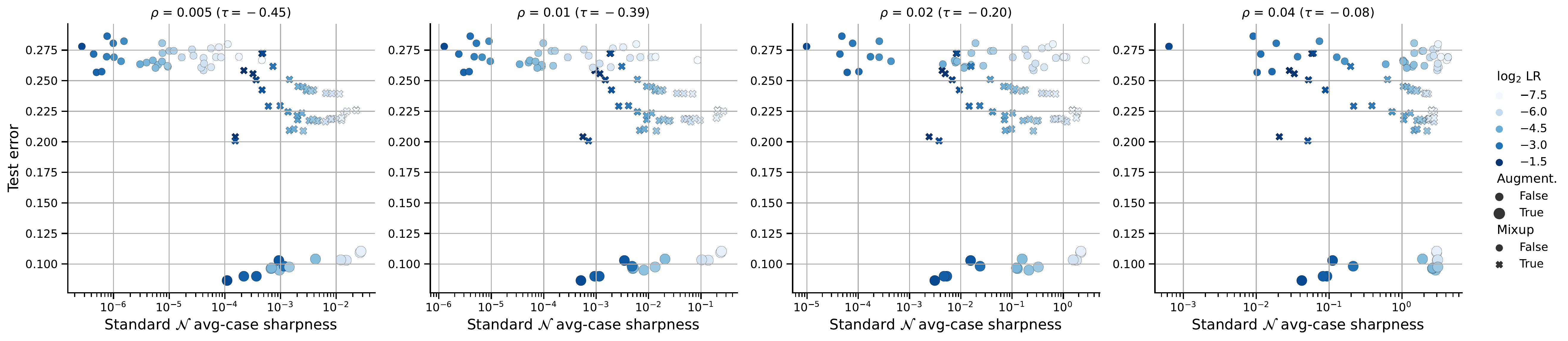}
    \includegraphics[width=1.0\textwidth]{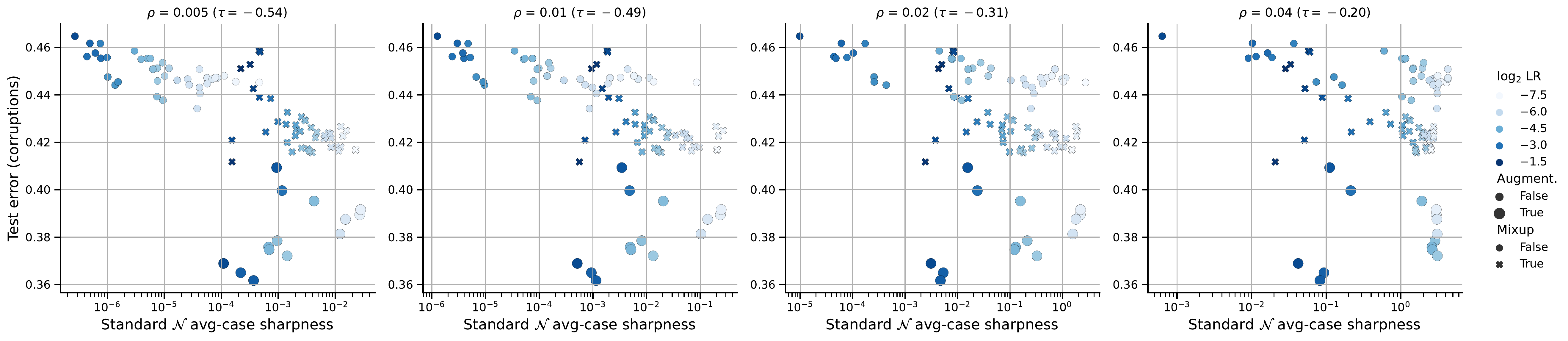}
    \textbf{Standard worst-case $\ell_2$ sharpness (unnormalized) for ViTs}
    \includegraphics[width=1.0\textwidth]{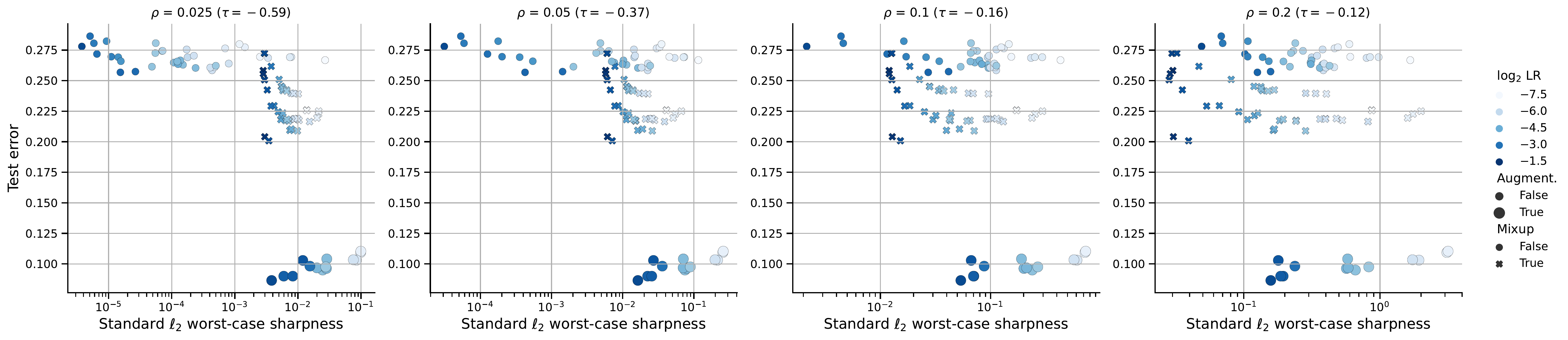}
    \includegraphics[width=1.0\textwidth]{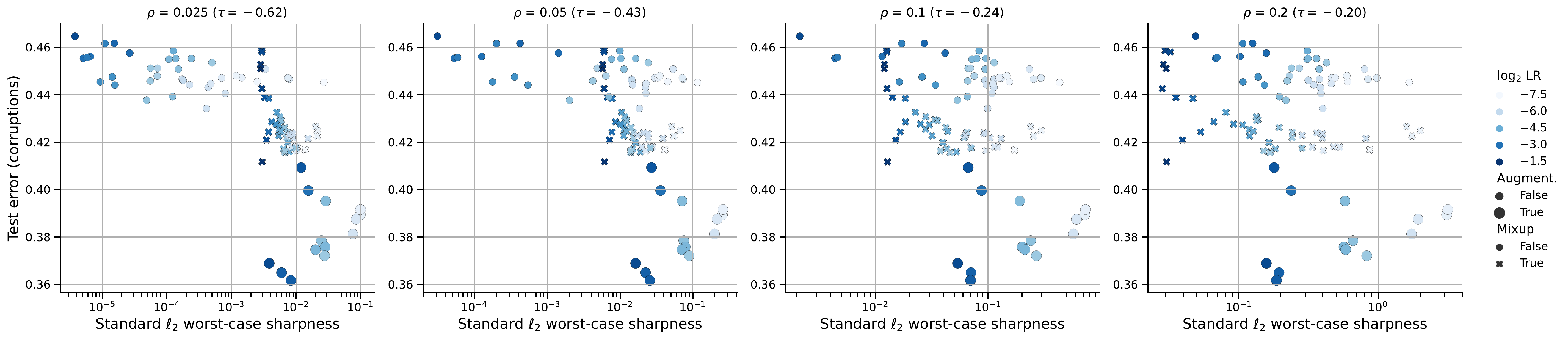}
    \caption{Average and worst-case $\ell_2$ standard sharpness definitions (unnormalized)  vs. test error and OOD test error (common corruptions) on CIFAR-10 for ViTs for different radii $\rho$.}
    \label{fig:cifar10_vit_standard_l2_sharpness}
\end{figure*}
\begin{figure*}[t!]
    \centering
    \textbf{Adaptive average-case $\ell_2$ (Gaussian perturbations) sharpness (unnormalized) for ViTs}
    \includegraphics[width=1.0\textwidth]{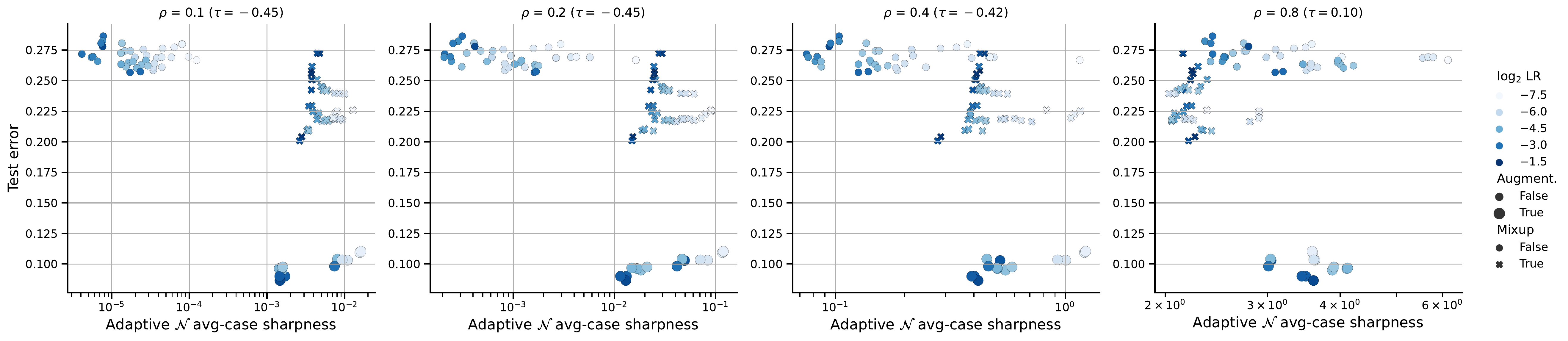}
    \includegraphics[width=1.0\textwidth]{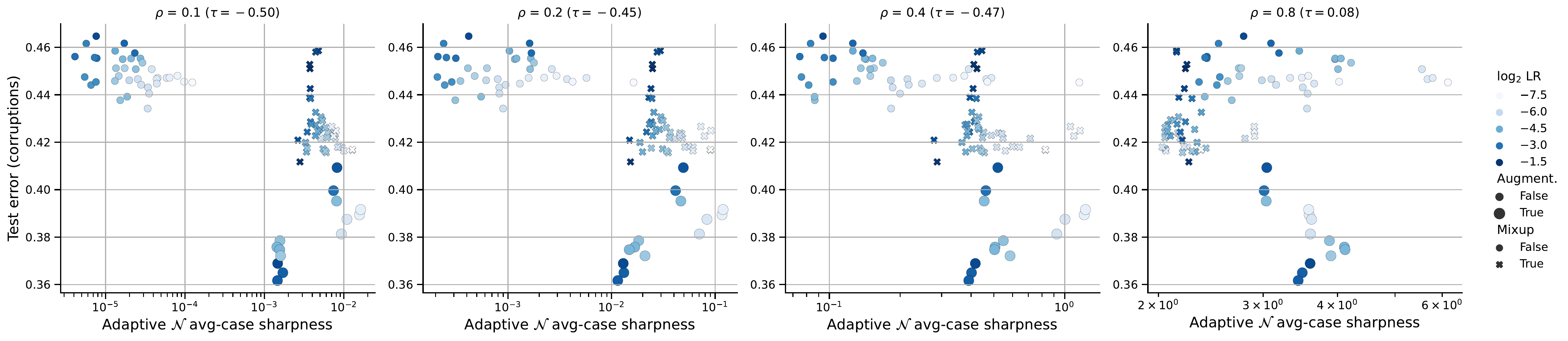}
    \textbf{Adaptive worst-case $\ell_2$ sharpness (unnormalized) for ViTs}
    \includegraphics[width=1.0\textwidth]{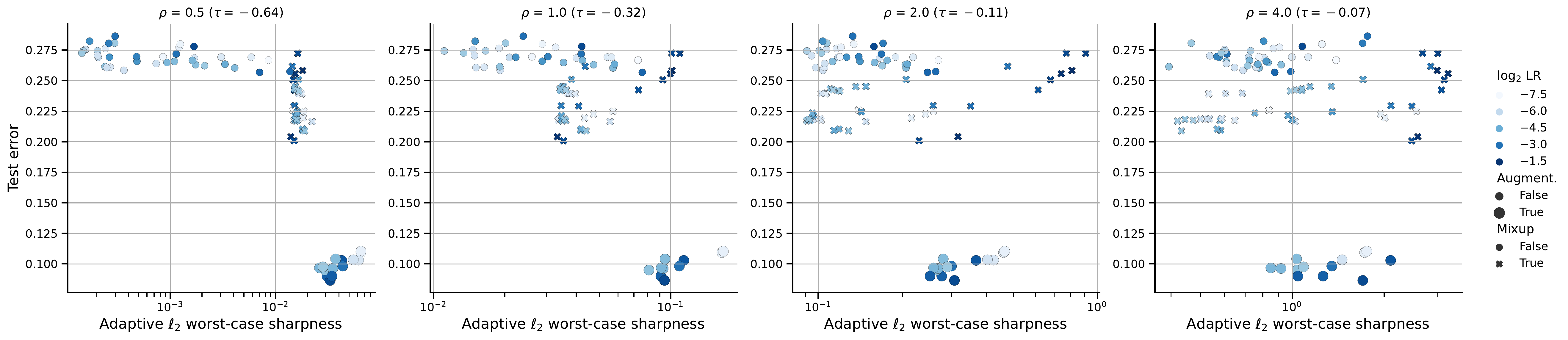}
    \includegraphics[width=1.0\textwidth]{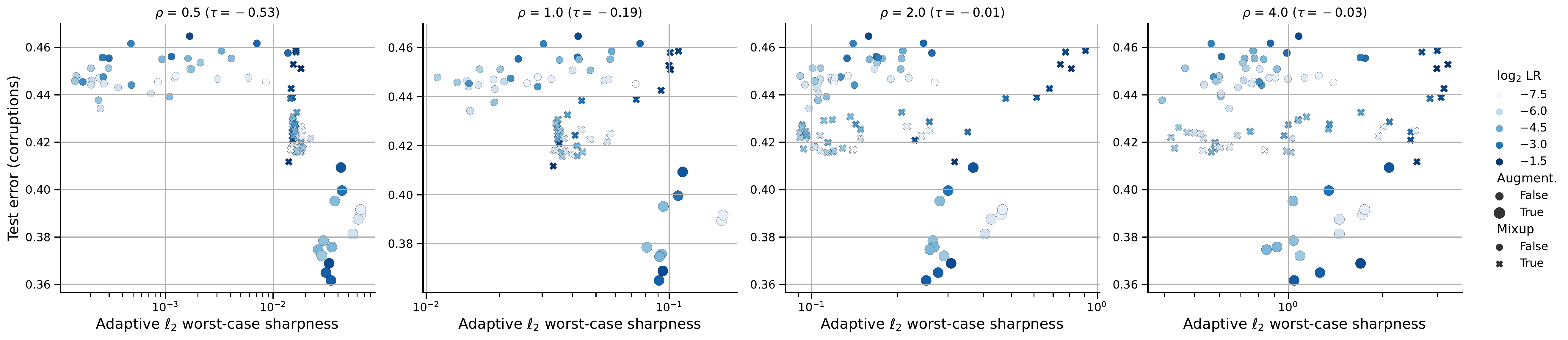}
    \caption{Average and worst-case $\ell_2$ adaptive sharpness definitions (unnormalized)  vs. test error and OOD test error (common corruptions) on CIFAR-10 for ViTs for different radii $\rho$.}
    \label{fig:cifar10_vit_adaptive_l2_sharpness}
\end{figure*}
\begin{figure*}[t!]
    \centering
    \textbf{Adaptive average-case $\ell_2$ (Gaussian perturbations) sharpness (normalized) for ViTs}
    \includegraphics[width=1.0\textwidth]{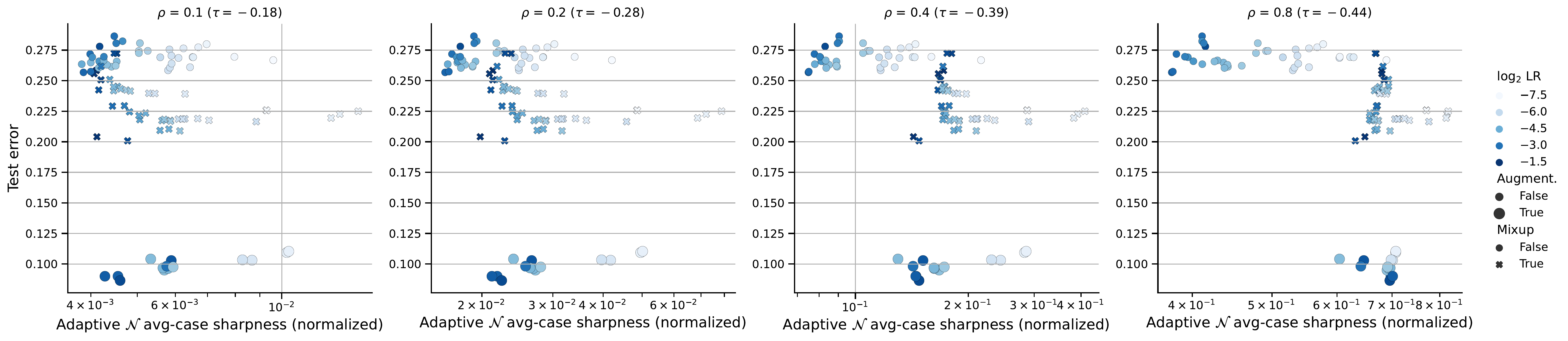}
    \includegraphics[width=1.0\textwidth]{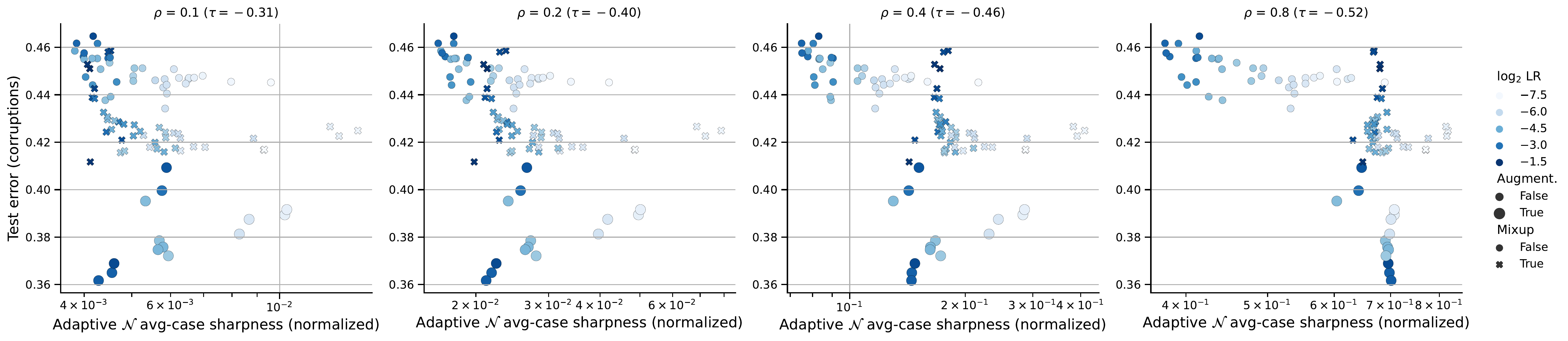}
    \textbf{Adaptive worst-case $\ell_2$ sharpness (normalized) for ViTs}
    \includegraphics[width=1.0\textwidth]{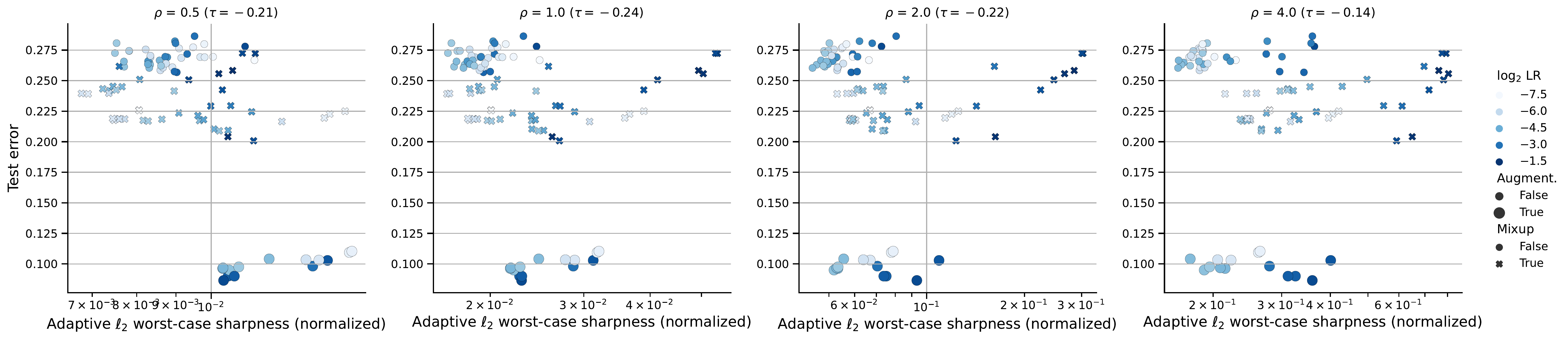}
    \includegraphics[width=1.0\textwidth]{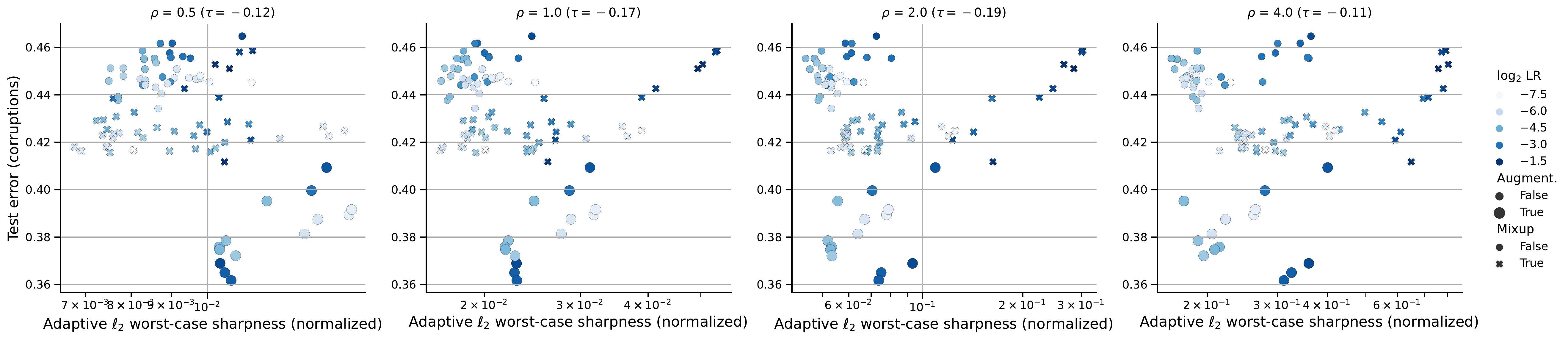}
    \caption{Average and worst-case $\ell_2$ adaptive sharpness definitions (normalized) vs. test error and OOD test error (common corruptions) on CIFAR-10 for ViTs for different radii $\rho$.}
    \label{fig:cifar10_vit_adaptive_l2_sharpness_normalized}
\end{figure*}

\begin{figure*}[t!]
    \centering
    \textbf{Standard average-case $\ell_\infty$ (uniform perturbations) sharpness (unnormalized) for ViTs}
    \includegraphics[width=1.0\textwidth]{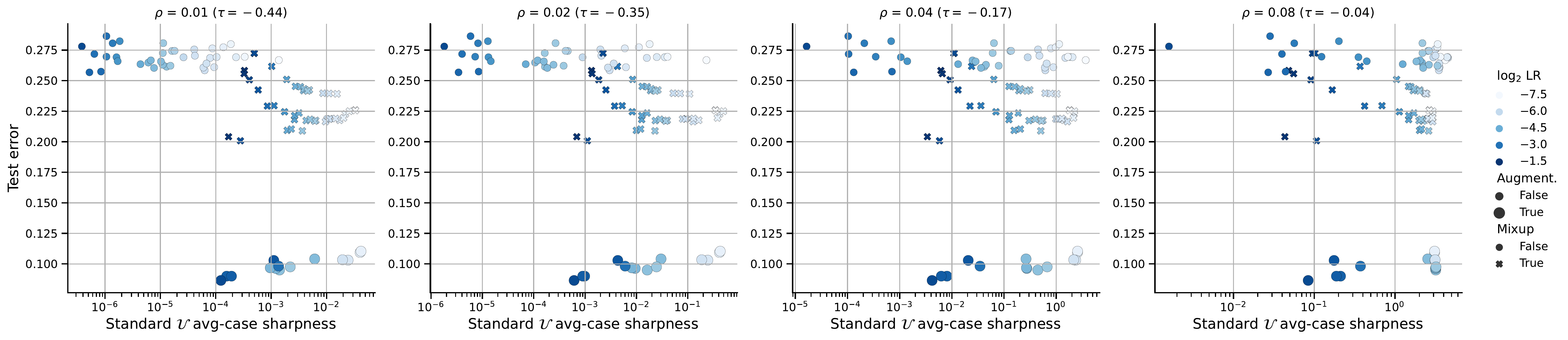}
    \includegraphics[width=1.0\textwidth]{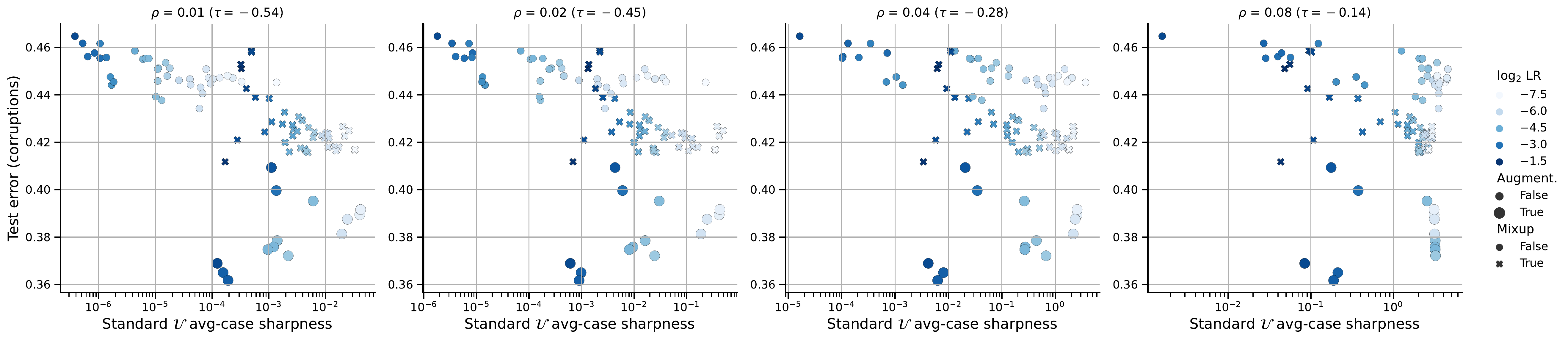}
    \textbf{Standard worst-case $\ell_\infty$ sharpness (unnormalized) for ViTs}
    \includegraphics[width=1.0\textwidth]{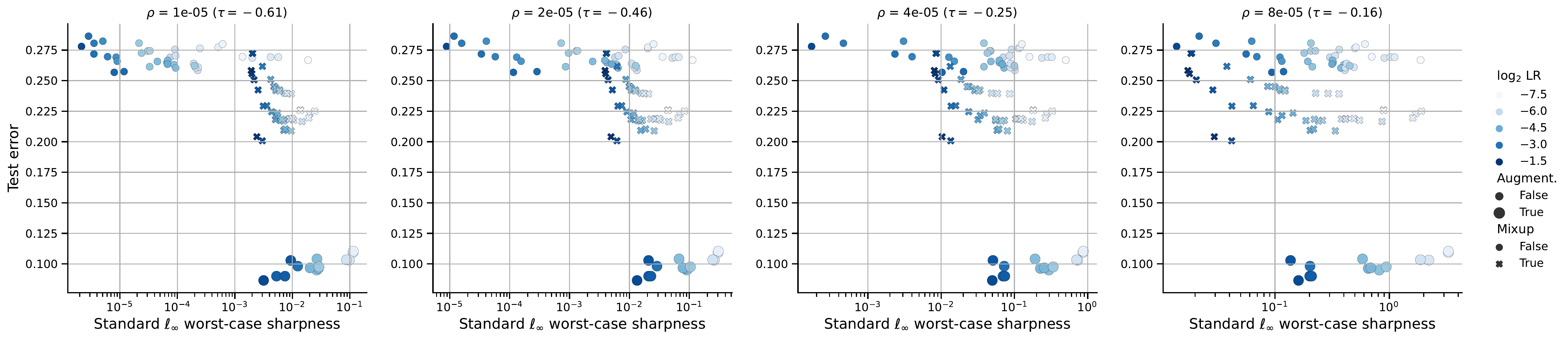}
    \includegraphics[width=1.0\textwidth]{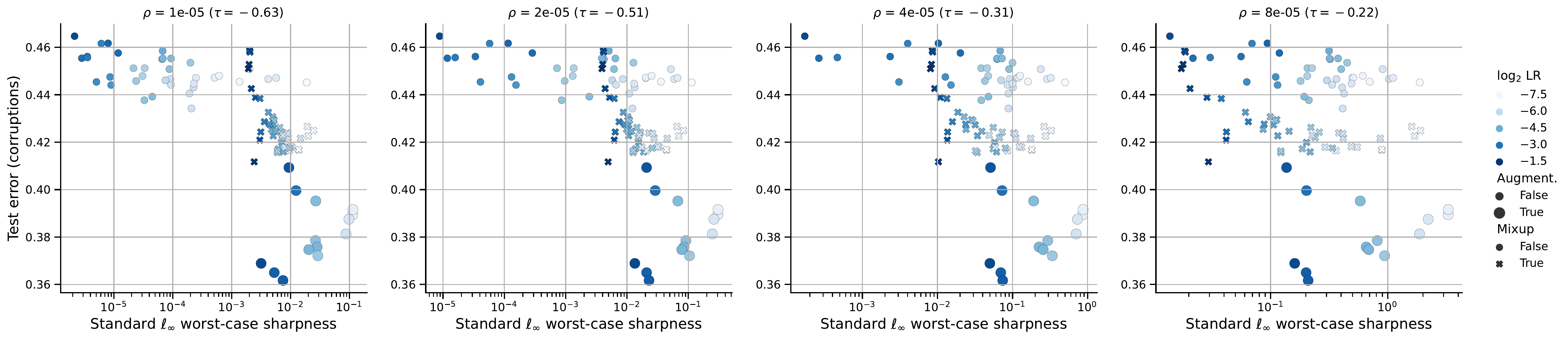}
    \caption{Average and worst-case $\ell_\infty$ standard sharpness definitions (unnormalized)  vs. test error and OOD test error (common corruptions) on CIFAR-10 for ViTs for different radii $\rho$.}
    \label{fig:cifar10_vit_standard_linf_sharpness}
\end{figure*}
\begin{figure*}[t!]
    \centering
    \textbf{Adaptive average-case $\ell_\infty$ (uniform perturbations) sharpness (unnormalized) for ViTs}
    \includegraphics[width=1.0\textwidth]{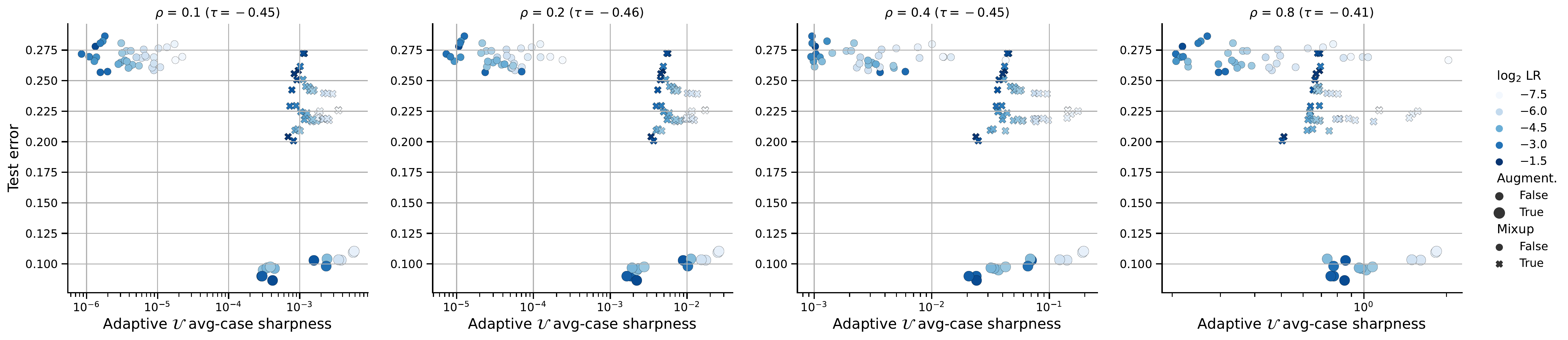}
    \includegraphics[width=1.0\textwidth]{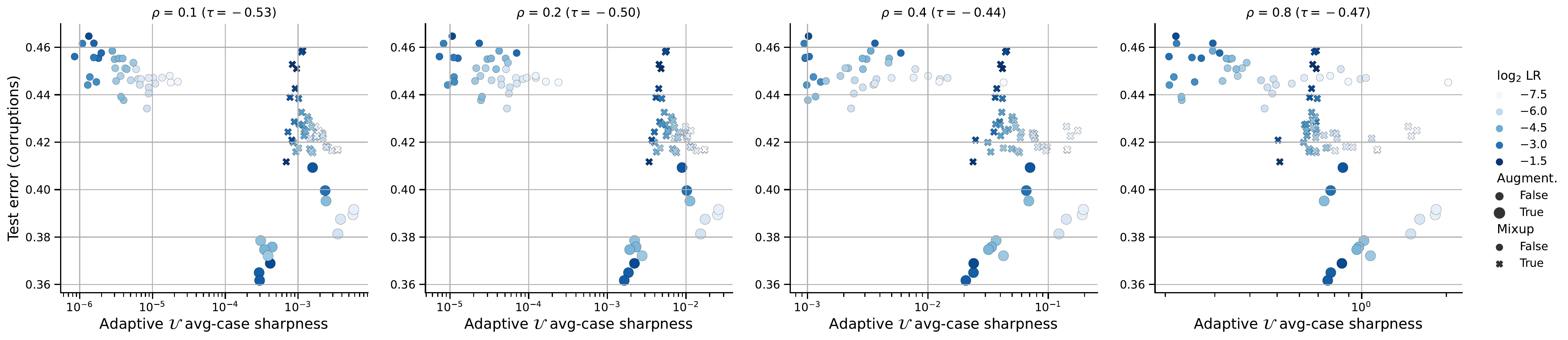}
    \textbf{Adaptive worst-case $\ell_\infty$ sharpness (unnormalized) for ViTs}
    \includegraphics[width=1.0\textwidth]{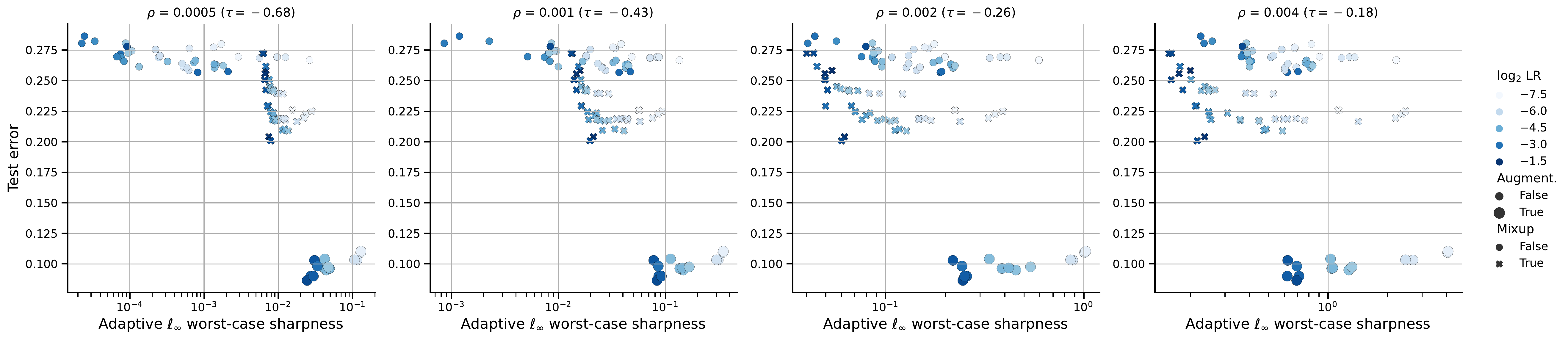}
    \includegraphics[width=1.0\textwidth]{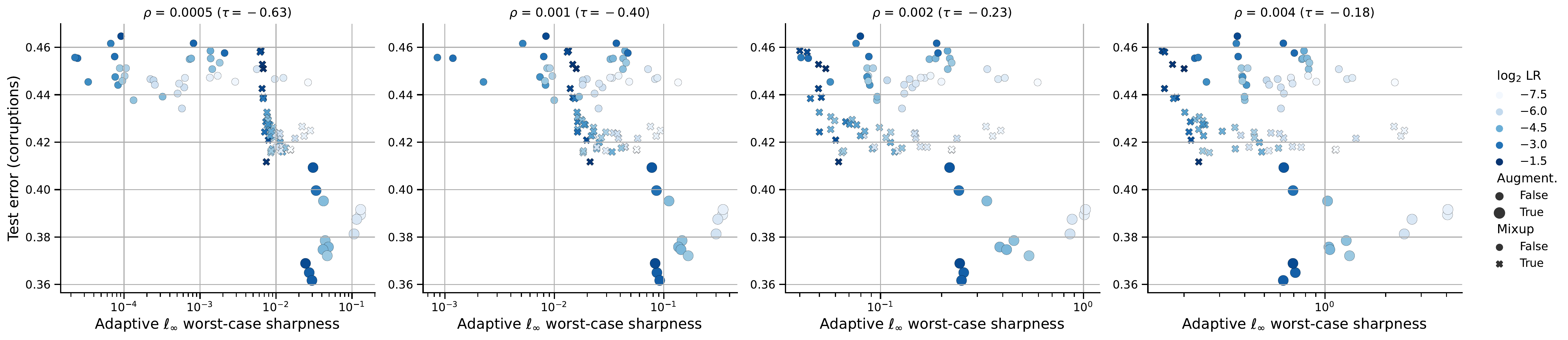}
    \caption{Average and worst-case $\ell_\infty$ adaptive sharpness definitions (unnormalized) and OOD test error (common corruptions) on CIFAR-10 for ViTs for different radii $\rho$.}
    \label{fig:cifar10_vit_adaptive_linf_sharpness}
\end{figure*}
\begin{figure*}[t!]
    \centering
    \textbf{Adaptive average-case $\ell_\infty$ (uniform perturbations) sharpness (normalized) for ViTs}
    \includegraphics[width=1.0\textwidth]{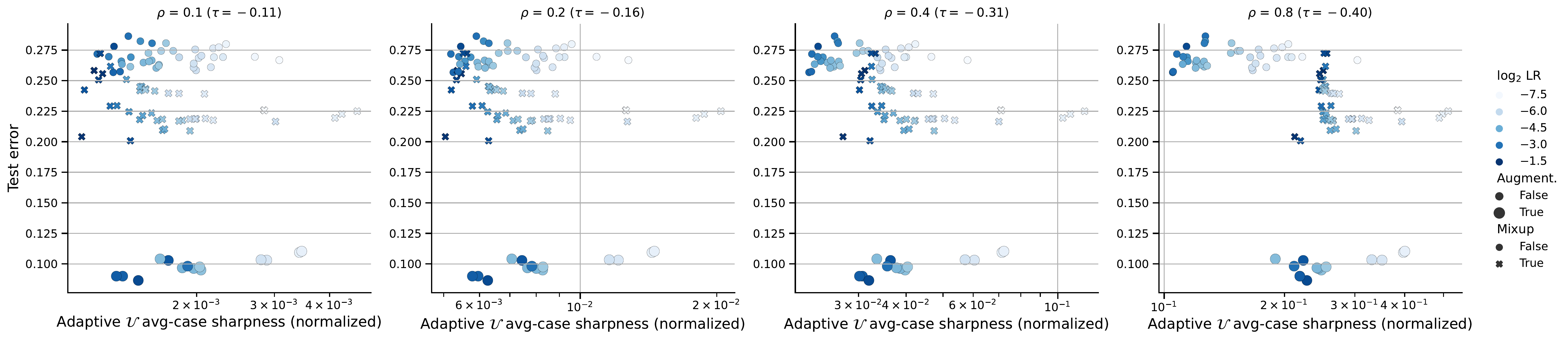}
    \includegraphics[width=1.0\textwidth]{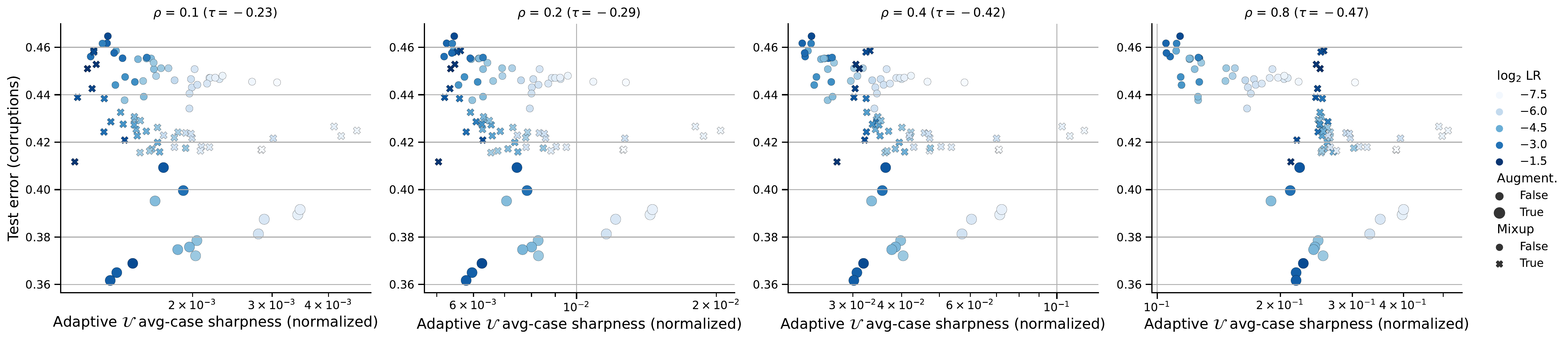}
    \textbf{Adaptive worst-case $\ell_\infty$ sharpness (normalized) for ViTs}
    \includegraphics[width=1.0\textwidth]{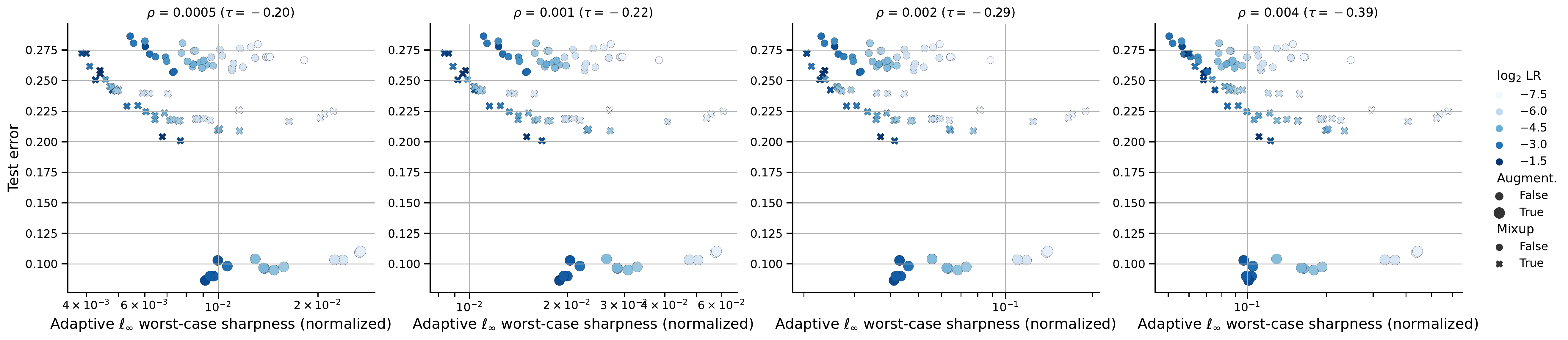}
    \includegraphics[width=1.0\textwidth]{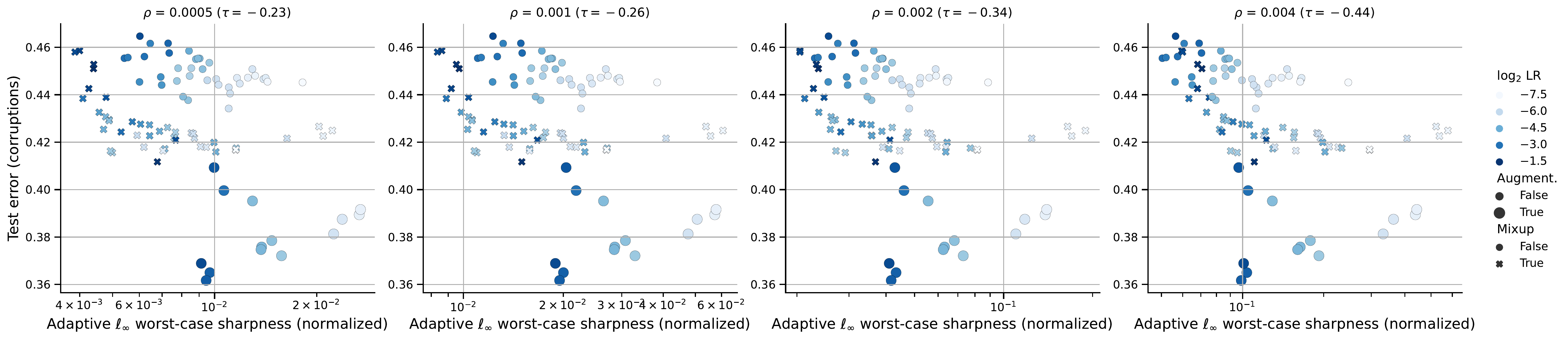}
    \caption{Average and worst-case $\ell_\infty$ adaptive sharpness definitions (normalized) vs. test error and OOD test error (common corruptions) on CIFAR-10 for ViTs for different radii $\rho$.}
    \label{fig:cifar10_vit_adaptive_linf_sharpness_normalized}
\end{figure*}

\end{document}